\documentclass{tlp}
\usepackage{amsmath}
\usepackage{amssymb}
\usepackage{graphicx}
\usepackage{upquote}
\usepackage{float}
\usepackage{enumerate}
\usepackage{textcomp}
\usepackage{scrextend}
\floatstyle{boxed}
\restylefloat{figure}

\newtheorem{theorem}{Theorem}
\newtheorem{observation}{Observation}
\newtheorem{lemma}[observation]{Lemma}

\newtheorem{definition}{Definition}
\newtheorem{notation}{Notation}
\newtheorem{example}{Example}

\title[Justifying Answer Sets using Argumentation]{Justifying Answer Sets using Argumentation}
\author[Claudia Schulz and Francesca Toni]{CLAUDIA SCHULZ and FRANCESCA TONI\\ Department of Computing, Imperial College London\\ London SW7 2AZ, UK\\ \email{\{claudia.schulz, f.toni\}@imperial.ac.uk}}
\pagerange{\pageref{firstpage}--\pageref{lastpage}}

 \submitted{13 April 2014}
 \revised{2 October 2014}
 \accepted{30 Otober 2014}

\begin{document}
\maketitle

\noindent
{\bf Note:} This article has been accepted for publication in
\emph{Theory and Practice of Logic Programming}, \copyright Cambridge
University Press.\\

\label{firstpage}

%%%%%%%%%%%%%%%%%%%%%%%%%%%%%%%%%%%%%%%%%%%%%%%%%%%%%%%%%%%%%%%%%%%%%%%%%%%%%%%%%%%%%%%%%%%%%%%%%%%%%%%%%%%%%%%%%%%%%%%%%%%%%%%%%%%%%%%%%%%%%%%%%%%%%%%%%%%%%%%
\begin{abstract}
An answer set is a plain set of literals which has no further structure that would explain why certain literals are part of it and why others are not.
We show how argumentation theory can help to explain why a literal is or is not contained in a given answer set by defining
two justification methods,
both of which make use of the correspondence between answer sets of a logic program and
stable extensions of the Assumption-Based Argumentation (ABA) framework constructed from the same logic program.
\emph{Attack Trees} justify a literal in argumentation-theoretic terms, i.e. using arguments and attacks between them,
whereas \emph{ABA-Based Answer Set Justifications} express the same justification structure in logic programming terms, that is using literals and their relationships.
Interestingly, an ABA-Based Answer Set Justification corresponds to an admissible fragment of the answer set in question, and
an Attack Tree corresponds to an admissible fragment of the stable extension corresponding to this answer set.
\end{abstract}


%%%%%%%%%%%%%%%%%%%%%%%%%%%%%%%%%%%%%%%%%%%%%%%%%%%%%%%%%%%%%%%%%%%%%%%%%%%%%%%%%%%%%%%%%%%%%%%%%%%%%%%%%%%%%%%%%%%%%%%%%%%%%%%%%%%%%%%%%%%%%%%%%%%%%%%%%%%%%%%%%%
%%%%%%%%%%%%%%%%%%%%%%%%%%%%%%%%%%%%%%%%%%%%%%%%%%%%%%%%%%%%%%%%%%%%%%%%%%%%%%%%%%%%%%%%%%%%%%%%%%%%%%%%%%%%%%%%%%%%%%%%%%%%%%%%%%%%%%%%%%%%%%%%%%%%%%%%%%%%%%%%%%
\section{Introduction}
\label{sec:intro}

Answer Set Programming (ASP) is one of the most widely used non-monotonic reasoning paradigms, allowing to efficiently compute solutions to problems involving defaults and exceptions \cite{answer_sets}.
A problem is represented in terms of a logic program, that is if-then clauses containing negation-as-failure (NAF) literals which express exception conditions for the applicability of clauses.
The solutions to the problem are then given by the declarative answer set semantics \cite{AS_gelfond} for the logic program.
ASP is applied in a variety of different areas, ranging from bioinformatics \cite{ASP_bioinformatics} over music composition \cite{ASP_music} to multi-agent systems \cite{ASP_multiagent}.
Answer set solvers like clingo \cite{clingo}, smodels \cite{smodels}, and DLV \cite{dlv} provide efficient tools for the computation of answer sets.

Especially with respect to the application of ASP in real-world scenarios involving non-experts,
it is useful to have an explanation as to why something does or does not belong to a solution.
As an example, consider a medical decision support system which operates on a logic program comprising general treatment decision rules along with facts about a patient's medical conditions.
The answer sets of such a logic program contain treatment suggestions or exclusions for the given patient.
For a doctor using this medical decision support system, it is important to know why the system suggests a certain treatment
as well as why a treatment is not part of a solution.
In ASP terms, the doctor needs a justification as to why a literal is or is not contained in an answer set.
This is particularly important if the doctor's intended treatment decision disagrees with the system's suggestion.
However, no matter whether an answer set is computed by an answer set solver or by hand using trial and error, it is a plain set of literals.
That is to say that an answer set does not provide any justification as to why certain literals are part of it whereas others are not.

In this paper we present two methods for justifying literals with respect to an answer set of a consistent logic program
by applying argumentation theory, another widely used technique in the field of non-monotonic reasoning.
Here, we use \emph{Assumption-Based Argumentation} (ABA) \cite{aba_lp,assumption_based}, a structured argumentation framework which constructs arguments from rules and assumptions,
and attacks from the notion of contrary of assumptions.
ABA is particularly suitable for our purpose as it was inspired by logic programming, default logic and other non-monotonic reasoning approaches \cite{aba_lp} which are closely related to ASP.
Due to this connection, it is straight forward to construct the \emph{translated ABA framework} of a logic program, i.e. the ABA framework expressing the same problem as the logic program.
One of the semantics for ABA frameworks is the stable extension semantics \cite{aba_lp,Dung},
which has its roots in the stable model semantics for logic programs.
Since the answer set semantics is based on the stable model semantics as well,
every answer set of a logic program corresponds to a stable extension of the translated ABA framework, and vice versa.
We make use of this connection to justify literals with respect to a given answer set of a consistent logic program by means of
arguments in the context of the corresponding stable extension of the translated ABA framework.

The first justification approach, an \emph{Attack Tree}, expresses how to construct an argument for the literal in question (the supporting argument)
as well as which arguments attack the argument for the literal in question (the attacking arguments);
the same information is provided for all arguments attacking the attacking arguments, and so on.
The second justification approach, an \emph{ABA-Based Answer Set (ABAS) Justification} of a literal, represents the same information as an Attack Tree,
but expressed in terms of literals rather than arguments.
An ABAS Justification comprises facts and NAF literals necessary to derive the literal in question (the ``supporting literals'')
as well as information about literals which are in conflict with the literal in question (the ``attacking literals'').
The same information is provided for all supporting and attacking literals of the literal in question, for all their supporting and attacking literals, and so on.

An Attack Tree is a (possibly infinite) tree with nodes holding arguments, where the argument held by a child node attacks the argument held by the parent node.
Since arguments are trees themselves, indicating which components (rules, assumptions) are necessary to construct the argument,
an Attack Tree has a two-layered structure: It is a tree consisting of trees.
An ABAS Justification is the flattened version of an Attack Tree,
containing literal-pairs which express the different parent-child relations expressed in an Attack Tree.
The relation between arguments in the Attack Tree is represented in terms of literal-pairs which are in an attack relation;
the relation between components of an argument is represented in terms of literal-pairs which are in a support relation.
An ABAS Justification can also be interpreted as a graph, where every literal occurring in a pair forms a node in the graph.
The graph has a support edge between two literal-nodes if these two literals occur as a literal-pair in a support relation in the ABAS Justification.
Analogously, the graph has an attack edge between two literal-nodes if these two literals occur as a literal-pair in an attack relation in the ABAS Justification.

Our justification approaches have two purposes.
On the one hand, they contribute to the field of answer set justification research, 
which has been identified as an important but not yet sufficiently studied research area \cite{explanation_review,asp_expl_position}.
The reason to use ABA for explanations instead of constructing justifications from the logic program straight away in terms of simple derivations or proof trees \cite{explanation_derivation,explanation_prooftree}
is that ABA is conceptually close to logic programs but provides additional concepts and constructs which have been identified as useful for explanation purposes,
such as the notion of arguments and attacks \cite{explanation_argumentsLP,explanation_argumentation}.
On the other hand, our justification approaches also provide a theoretical impact with respect to the relation between non-monotonic reasoning systems.
Even though ASP has been applied to argumentation theory in the sense that an argumentation framework can be equivalently expressed in ASP \cite{argASP,delp_asp},
the converse has not been discussed in the literature.
To the best of our knowledge, Attack Trees and ABAS Justifications are the first approaches applying argumentation theory for ASP, with the 
exception of
\begin{itemize}
 \item early work on manually constructing arguments and attacks from a logic program
according to Toulmin's argument scheme, which then serves as an explanation of the logic program \cite{explanation_argumentsLP}; and
 \item Argumentation-Based Answer Set Justification \cite{arg_based_just} which can be considered as a predecessor of ABAS Justifications.
 Similarly to ABAS Justifications, Argumentation-Based Answer Set Justifications are constructed from arguments and attacks between them,
but using the ASPIC+ argumentation framework \cite{Prakken} instead of ABA.
\end{itemize}

The paper is organized as follows: In Section~\ref{sec:background} we recall some key concepts of ASP and ABA and
give some preliminary definitions and results building upon this background.
Furthermore, we give a motivating (medical) example for ABAS Justifications.
In Section~\ref{sec:transatlion} we show how to translate a logic program into an ABA framework and prove their correspondence with respect to the stable model semantics.
In Section~\ref{sec:attackTrees} we introduce Attack Trees drawn from a translated ABA framework as a first justification method,
show their relationship with abstract dispute trees for ABA \cite{abaDialectical}, and characterize the explanation they provide as an
admissible fragment of the answer set in question.
Based on Attack Trees, we define two forms of ABAS Justifications:
Basic ABA-Based Answer Set Justifications (Section~\ref{sec:ASjust}) demonstrate the main idea of flattening Attack Trees,
yielding a justification in terms of literals and their relations.
Labelled ABA-Based Answer Set Justifications (Section~\ref{sec:ASjust_labelled}) are a more elaborate version of
Basic ABA-Based Answer Set Justifications, following the same flattening strategy, but additionally using labels to solve some deficiencies of the basic variant.
In Section~\ref{sec:relWork} we compare ABAS Justifications to related work and in
Section~\ref{sec:conclusion} we conclude.

%%%%%%%%%%%%%%%%%%%%%%%%%%%%%%%%%%%%%%%%%%%%%%%%%%%%%%%%%%%%%%%%%%%%%%%%%%%%%%%%%%%%%%%%%%%%%%%%%%%%%%%%%%%%%%%%%%%%%%%%%%%%%%%%%%%%%%%%%%%%%%%%%%%%%%%%%%%%%%%%%%
%%%%%%%%%%%%%%%%%%%%%%%%%%%%%%%%%%%%%%%%%%%%%%%%%%%%%%%%%%%%%%%%%%%%%%%%%%%%%%%%%%%%%%%%%%%%%%%%%%%%%%%%%%%%%%%%%%%%%%%%%%%%%%%%%%%%%%%%%%%%%%%%%%%%%%%%%%%%%%%%%%
\section{Background and Preliminaries}
\label{sec:background}

This section describes all necessary background about ASP and ABA to understand the definitions of ABAS Justifications.
In addition, we prove some core results about concepts in ASP and in ABA which have not or have only partially been considered in the literature before.
We then use these to prove our main results in the remainder of the paper.

\subsection{Answer Set Programming}
A \emph{logic program} $\mathcal{P}$ is a (finite) set of clauses of the form $l_0 \leftarrow l_1, \ldots, l_m,$ $not ~ l_{m+1},\\ \ldots, not ~ l_{m+n}$ with $m,n \geq 0$.
All $l_i$ are classical ground\footnote{As conventional in the logic programming literature,
clauses containing variables are shorthand for all their ground instances.} literals,
i.e. atoms $a$ or negated atoms $\neg a$, and $not~ l_{m+1},\ldots, not~ l_{m+n}$ are \emph{negation-as-failure} (NAF) literals.
The classical literal $l_0$ on the left-hand side of the arrow is referred to as the clause's \emph{head}, all literals on the right of the arrow form the \emph{body} of the clause.
If the body of a clause is empty, the head is called a \emph{fact}.

\begin{notation}
The letter $k$ is used for a literal in general, i.e. a classical literal $l$ or a NAF literal $not~ l$.
$\mathcal{HB}_{\mathcal{P}}$ denotes the Herbrand Base of $\mathcal{P}$, that is the set of all ground atoms of $\mathcal{P}$.
$Lit_{\mathcal{P}} = \mathcal{HB}_{\mathcal{P}} \; \cup \; \{\neg a \; | \; a \in \mathcal{HB}_{\mathcal{P}}\}$ is the set of all classical literals of $\mathcal{P}$, and
$\textit{NAF}_{\mathcal{P}} = \{ not~ l \; |\; l \in Lit_{\mathcal{P}} \}$
consists of all NAF literals of $\mathcal{P}$.
We say that $l$ is the \emph{corresponding classical literal} of a NAF literal $not~ l$.
\end{notation}

In the following, we recall the concept of answer sets as introduced in \cite{AS_gelfond}.
Let $\mathcal{P}$ be a logic program not containing NAF literals.
The \emph{answer set} of $\mathcal{P}$, denoted $\mathcal{AS}(\mathcal{P})$, is the smallest set $S \subseteq Lit_{\mathcal{P}}$ such that:
\begin{enumerate}
	\item for any clause $l_0 \leftarrow l_1,\ldots,l_m$ in $\mathcal{P}$: if $l_1,\ldots,l_m \in S$ then $l_0 \in S$; and
	\item $S = Lit_{\mathcal{P}}$ if $S$ contains complementary classical literals $a$ and $\neg a$.
\end{enumerate}

For a logic program $\mathcal{P}$, possibly containing NAF literals,
and any set $S \subseteq Lit_{\mathcal{P}}$, the \emph{reduct} $\mathcal{P}^S$ is obtained from $\mathcal{P}$ by deleting:
\begin{enumerate}
	\item all clauses with $not~ l$ in their bodies where $l \in S$, and
	\item all NAF literals in the remaining clauses.
\end{enumerate}
Then, $S$ is an answer set of $\mathcal{P}$ if it is the answer set of the reduct $\mathcal{P}^S$, i.e. if $S = \mathcal{AS}(\mathcal{P}^S)$.
A logic program is \emph{inconsistent} if it has no answer set or if its only answer set is $Lit_{\mathcal{P}}$; otherwise it is \emph{consistent}.
In the remainder of the paper, and if not stated otherwise, we assume that logic programs are consistent.

Note that answer sets only contain classical literals.
However, if $l \notin S$ for an answer set $S$ of $\mathcal{P}$ and some classical literal $l \in Lit_{\mathcal{P}}$, 
then $not~ l$ is considered satisfied with respect to $S$.
Thus, we introduce the following new definition.

\begin{definition}[Answer Set with NAF literals]
\label{def:AS_NAF}
Let $\mathcal{P}$ be a logic program and let $S \subseteq Lit_{\mathcal{P}}$ be a set of classical literals.
$\Delta_S = \{ not ~ l \in \textit{NAF}_{\mathcal{P}}\; | \; l \notin S \}$ consists of all NAF literals $not~ l$ whose corresponding classical literal $l$ is not contained in $S$.
If $S$ is an answer set of $\mathcal{P}$, then $S_{\textit{NAF}} = S\; \cup\; \Delta_S$ is an \emph{answer set with NAF literals} of $\mathcal{P}$.
\end{definition}

Intuitively, $S_{\textit{NAF}}$ consists of all literals in an answer set $S$ plus all NAF literals which are satisfied with respect to $S$.
For the purpose of proving correspondence between answer sets of a logic program and stable extensions of an argumentation framework in Section~\ref{sec:transatlion},
we introduce a new reformulation of answer sets in terms of modus ponens and prove correspondence with the original definition:

\begin{notation}
\label{not:mp}
 $\vdash_{MP}$ denotes derivability using \emph{modus ponens} on $\leftarrow$ as the only inference rule.
 $\mathcal{P} \; \cup \; \Delta_S$, for $\mathcal{P}$ a logic program and $\Delta_S \subseteq \textit{NAF}_{\mathcal{P}}$,
 denotes the logic program $\mathcal{P} \;\cup\; \{not ~ l \leftarrow\; | \; not ~ l \in \Delta_S\}$.
When used on such $\mathcal{P} \; \cup \; \Delta_S$, $\vdash_{MP}$ treats NAF literals purely syntactically as in \cite{modusponens_lp}
and treats facts $l \leftarrow\:$ as $l \leftarrow true$
 where $\mathcal{P} \; \cup \; \Delta_S \vdash_{MP} true$ for any logic program $\mathcal{P}$ and any set of NAF literals $\Delta_S$.
\end{notation}

\begin{lemma}
 \label{lem:AS_MP}
 Let $\mathcal{P}$ be a consistent logic program and let $S \subseteq Lit_{\mathcal{P}}$ be a set of classical literals.
 \begin{itemize}
  \item $S$ is an answer set of $\mathcal{P}$ if and only if $S = \{ l \in Lit_{\mathcal{P}}\; |\; \mathcal{P} \; \cup \; \Delta_S \vdash_{MP} l\}$.
  \item $S_{\textit{NAF}} = S\; \cup \; \Delta_S$ is an answer set with NAF literals of $\mathcal{P}$ if and only if $S_{\textit{NAF}} = \{ k \; | \; \mathcal{P} \; \cup \; \Delta_S \vdash_{MP} k\}$.
 \end{itemize}
\end{lemma}

\begin{proof}
 We prove both items:
 \begin{itemize}
  \item If $S$ is an answer set of $\mathcal{P}$ then $S = \mathcal{AS}(\mathcal{P}^S)$. 
  This means that $\forall l \in S$ there exists a clause $l \leftarrow l_1, \ldots, l_m \: \in \mathcal{P}^S$ such that $l_1, \ldots, l_m \in S$.
  It follows that there exists a clause $l \leftarrow l_1, \ldots, l_m, not ~ l_{m+1}, \ldots, not ~ l_{m+n} \in \mathcal{P}$ such that $l_1, \ldots, l_m \in S$ and 
  $l_{m+1}, \ldots, l_{m+n} \notin S$. Then, by Definition~\ref{def:AS_NAF}, $ not ~ l_{m+1}, \ldots, not ~ l_{m+n} \in \Delta_S$.
  Thus, $\mathcal{P} \; \cup \; \Delta_S \vdash_{MP} l$.\\
  For the other direction, if $\mathcal{P} \; \cup \; \Delta_S \vdash_{MP} l$ then (1) $l \in \Delta_S$ or (2) there exists a clause
  $l \leftarrow l_1, \ldots, l_m, not ~ l_{m+1}, \ldots, not ~ l_{m+n} \in \mathcal{P}$ such that $\forall l_i (1 \leq i \leq m):
  \mathcal{P} \; \cup \; \Delta_S \vdash_{MP} l_i$ and $\forall not~ l_j (m+1 \leq j \leq m+n): \mathcal{P} \; \cup \; \Delta_S \vdash_{MP} not ~ l_j$.
  In the first case, $l$ is a NAF literal which should not be part of $S$.
  This is satisfied since $l \notin Lit_{\mathcal{P}}$ and therefore $l \notin S = \{ l \in Lit_{\mathcal{P}}\; |\; \mathcal{P} \; \cup \; \Delta_S \vdash_{MP} l\}$.
  In the second case, since $\mathcal{P}$ contains no clause with a NAF literal in its head it follows that $not~ l_j \in \Delta_S$, i.e.
  $\forall l_j: l_j \notin S$. Then, by definition of reduct, $l \leftarrow l_1, \ldots, l_m \in \mathcal{P}^S$.
  Since $\mathcal{P} \; \cup \; \Delta_S \vdash_{MP} l_i$, $l_i \in S$, thereby satisfying the condition of an answer set for $l$ to be in $S$.
  \item If $S_{NAF}$ is an answer set with NAF literals then by Definition~\ref{def:AS_NAF}, $S_{NAF} = S\; \cup\; \Delta_S$.
  Then, by the first item $S_{NAF} = \{ l \in Lit_{\mathcal{P}}\; |\; \mathcal{P} \; \cup \; \Delta_S \vdash_{MP} l\} \; \cup \; \Delta_S$.
  By Notation~\ref{not:mp}, $\forall not~ l_i \in \Delta_S:  \Delta_S \vdash_{MP} not~l_i$ and therefore 
  $\mathcal{P} \; \cup \; \Delta_S \vdash_{MP} not~ l_i$ for any logic program $\mathcal{P}$. Thus, not restricting the conclusions of modus ponens to $Lit_{\mathcal{P}}$
  yields $S_{NAF} = \{ k \; | \; \mathcal{P} \; \cup \; \Delta_S \vdash_{MP} k\}$.\\
  For the other direction, if $\mathcal{P} \; \cup \; \Delta_S \vdash_{MP} k$ then by the proof of the first item $k \in \Delta_S$ or 
  $k \in S$ where $S$ is an answer set. Thus, $S_{NAF}$ is equivalent to $S\; \cup \; \Delta_S$, satisfying Definition~\ref{def:AS_NAF}.
 \end{itemize}
 \hfill
\end{proof}

%%%%%%%%%%%%%%%%%%%%%%%%%%%%%%%%%%%%%%%%%%%%%%%%%%%%%%%%%%%%%%%%%%%%%%%%%%%%%%%%%%%%%%%%%%%%%%%%%%%%%%%%%%%%%%%%%%%%%%%%%%%%%%%%%%%%%%%%%%%%%%%%%%%%%%%%%%%%%%%%%%
\subsection{An intuitive example of ASP}
\label{sec:background_example}

 Let Dr. Smith be an ophtalmologist (an eye doctor) and let one of his patients be Peter, who is diagnosed by Dr. Smith as being shortsighted.
 Based on this diagnosis, Dr. Smith has to decide on the most suitable treatment for Peter, taking into account
 the additional information he has about his patient, namely that Peter is afraid to touch his own eyes, that he is a student, and that he likes to do sports.
 Based on this information and his specialist knowledge, Dr. Smith decides that the most appropriate treatment for Peter's shortsightedness is laser surgery.
 Dr. Smith now checks whether this decision is in line with the recommendation of his decision support system, which is implemented in ASP.
 
 \begin{example}
 \label{ex:doctor}
 The following logic program $\mathcal{P}_{doctor}$ represents the decision support system used by Dr. Smith.
 It encodes some general world knowledge as well as 
 an ophtalmologist's specialist knowledge about the possible treatments of shortsightedness.
 $\mathcal{P}_{doctor}$ also captures the additional information that Dr. Smith has about his shortsighted patient Peter.
 \begin{align*}
  tightOnMoney &\leftarrow student, not~ richParents\\
  caresAboutPracticality &\leftarrow likesSports\\
  correctiveLens & \leftarrow shortSighted, not~ laserSurgery\\
  laserSurgery &\leftarrow shortSighted,  not~ tightOnMoney, not~ correctiveLens\\
  glasses &\leftarrow correctiveLens, not~caresAboutPracticality,\\ &\phantom{\leftarrow \; \; } not~contactLens\\
  contactLens &\leftarrow correctiveLens, not~\textit{afraidToTouchEyes},\\ &\phantom{\leftarrow \; \; } not~longSighted, not~glasses\\
  intraocularLens &\leftarrow correctiveLens, not~glasses, not~contactLens\\
  shortSighted &\leftarrow\\
  \textit{afraidToTouchEyes} &\leftarrow\\
  student &\leftarrow\\
  likesSports &\leftarrow
 \end{align*}
 $\mathcal{P}_{doctor}$ has only one answer set $S_{doctor} = \{shortSighted, \textit{afraidToTouchEyes},\\ student, likesSports,tightOnMoney, correctiveLens, caresAboutPracticality,\\ intraocularLens\}$.
 \end{example}
 
To Dr. Smith's surprise, the answer set computed by the decision support system contains the literal $intraocularLens$ but not $laserSurgery$,
suggesting that Peter should get intraocular lenses instead of having laser surgery.
 Dr. Smith now finds himself in the difficult situation to determine whether to trust his own treatment decision 
 or whether to take up the system's suggestion even without understanding it.
 Providing Dr. Smith with an explanation of the system's treatment suggestion or with an explanation as to why his own intended decision
 might be wrong would make it considerably easier for Dr. Smith to decide whether to trust himself or the decision support system.

We will use this example of Dr. Smith and his patient Peter to demonstrate our explanation approaches and to show how they can be applied to explain the solutions of
 a decision support system which is based on ASP.

%%%%%%%%%%%%%%%%%%%%%%%%%%%%%%%%%%%%%%%%%%%%%%%%%%%%%%%%%%%%%%%%%%%%%%%%%%%%%%%%%%%%%%%%%%%%%%%%%%%%%%%%%%%%%%%%%%%%%%%%%%%%%%%%%%%%%%%%%%%%%%%%%%%%%%%%%%%%%%%%%%
\subsection{ABA frameworks}
Much of the literature on argumentation in Artificial Intelligence focuses on two kinds of argumentation frameworks.
Abstract Argumentation \cite{Dung} assumes that a set of abstract entities (the arguments) are given along with an attack relation between them.
In contrast, structured argumentation frameworks such as \cite{Prakken,Garcia,Governatori} provide mechanisms for the construction of arguments from given knowledge,
mostly in the form of rules,
and for identifying the attack relation between arguments based on the structure of arguments.
We will here focus on the structured argumentation framework of \cite{aba_lp,assumption_based} called Assumption-Based Argumentation (ABA).

An \emph{ABA framework}  \cite{assumption_based} is a tuple $\langle \mathcal{L}, \mathcal{R}, \mathcal{A},\, \bar{\;}\,\rangle$,
where
\begin{itemize}
 \item $(\mathcal{L}, \mathcal{R})$ is a deductive system with\\
 $\mathcal{L}$ a formal language and\\
 $\mathcal{R}$ a set of inference rules of the form $\alpha_0 \leftarrow \alpha_1,\ldots, \alpha_m$ such that $m \geq 0$ and
 all $\alpha_i$ are sentences in $\mathcal{L}$;
 \item $\mathcal{A} \subseteq \mathcal{L}$ is a non-empty set of \emph{assumptions};
 \item $\bar{\;}$ is a total mapping from $\mathcal{A}$ into $\mathcal{L}$
defining the \emph{contrary} of each assumption, where $\overline{\alpha}$ denotes the contrary of $\alpha \in \mathcal{A}$.
\end{itemize}

Note that in this paper we use the same notation $\leftarrow$ for inference rules in ABA and for clauses in a logic program.
This will facilitate the presentation of our methods later.
We also adopt the logic programming terminology of ``head'', ``body'', ``fact'', and $\vdash_{MP}$ (see Notation~\ref{not:mp}) for ABA frameworks.
The following definitions are restricted to \emph{flat} ABA frameworks, where assumptions do not occur as the head of inference rules,
as we only need this kind of framework for our purposes.

In this paper we use a notion of ABA argument which is slightly different from the definitions in the ABA literature, in that
an ABA argument as defined here comprises not only the set of assumptions supporting this argument as in standard ABA,
but also the set of facts used in the construction
of this argument. 

\begin{definition}[ABA Argument]
\label{def:argument}
 Let $\langle \mathcal{L}, \mathcal{R}, \mathcal{A},\, \bar{\;}\,\rangle$ be an ABA framework.
 An argument for (the \emph{conclusion}) $\alpha \in \mathcal{L}$ supported by a set of \emph{assumption-premises} $AP \subseteq \mathcal{A}$  and a set of
 \emph{fact-premises} $FP \subseteq \{\beta \; | \;\beta \leftarrow \; \in \mathcal{R}\}$
 is a finite tree, where every node holds a sentence in $\mathcal{L}$, such that
 \begin{itemize}
  \item the root node holds $\alpha$;
  \item for every node $N$
  \begin{itemize}
   \item if $N$ is a leaf then $N$ holds either an assumption or a fact;
   \item if $N$ is not a leaf and $N$ holds the sentence $\gamma_0$, then there is an inference rule $\gamma_0 \leftarrow \gamma_1,\ldots,\gamma_m$ ($m > 0$) and
   $N$ has $m$ children, holding $\gamma_1,\ldots,\gamma_m$ respectively;
  \end{itemize}
  \item $AP$ is the set of all assumptions held by leaves;
  \item $FP$ is the set of all facts held by leaves.
 \end{itemize}
\end{definition}

We now define some further terminology for special kinds of arguments and for naming arguments in general.
\begin{notation}
  An argument for $\alpha$ supported by $AP$ and $FP$ is denoted $(AP, FP) \vdash \alpha$.
  We often use a unique name to denote an argument, e.g. $A: (AP,FP) \vdash \alpha$ is an argument with name $A$.
  With an abuse of notation, the name of an argument sometimes stands for the whole argument, for example $A$ denotes the argument $A: (AP,FP) \vdash \alpha$.
  An argument of the form $(\{\alpha\},\emptyset) \vdash \alpha$ is called \emph{assumption-argument},
  and similarly an argument of the form $(\emptyset,\{\alpha\}) \vdash \alpha$ is called \emph{fact-argument}.
  Given some argument $A:(AP, FP) \vdash \alpha$ with $\beta \in AP$ and $\gamma \in FP$, we say that $(\{\beta\},\emptyset) \vdash \beta$
  is the \emph{assumption-argument of the assumption-premise} $\beta$ of argument $A$
  and that $(\emptyset,\{\gamma\}) \vdash \gamma$ is the \emph{fact-argument of the fact-premise} $\gamma$ of $A$.
 \end{notation}
 
 Definition~\ref{def:argument} generates the notion of argument in \cite{assumption_based}:
 If $(AP, FP) \vdash \alpha$ is an argument according to Definition~\ref{def:argument}, then $AP \vdash \alpha$ is an argument in \cite{assumption_based}.
 Conversely, if $AP \vdash \alpha$ is an argument in \cite{assumption_based}, then there exists some
 $FP \subseteq \{\beta \; | \;\beta \leftarrow \; \in \mathcal{R}\}$ such that 
 $(AP, FP) \vdash \alpha$ is an argument according to Definition~\ref{def:argument}.

 ABA arguments can be naturally formulated in terms of $\vdash_{MP}$, as follows:
 
 \begin{lemma}
  \label{lem:arg_mp}
Let $\langle \mathcal{L}, \mathcal{R}, \mathcal{A},\, \bar{\;}\,\rangle$ be an ABA framework.
$(AP, FP) \vdash \alpha$ is an argument in $\langle \mathcal{L}, \mathcal{R}, \mathcal{A},\, \bar{\;}\,\rangle$ if and only if
$\mathcal{R} \; \cup \; AP \vdash_{MP} \alpha$ and $AP \subseteq \mathcal{A}$.
 \end{lemma}

 \begin{proof}
  This follows directly from the definition of arguments.\hfill
 \end{proof}

 The attack relation between arguments defined here is a slight variation of the notion in \cite{assumption_based}, as it considers arguments with both assumption- and fact-premises.

\begin{definition}[Attacks]
\label{def:attacks}
 An argument $(AP_1,FP_1) \vdash \alpha_1$ \emph{attacks} an argument $(AP_2,FP_2) \vdash \alpha_2$ \emph{on the assumption-premise} $\alpha_3$
 if and only if $\alpha_3 \in AP_2$ and $\overline{\alpha_3} = \alpha_1$.
 Equivalently, we say that $(AP_2,FP_2) \vdash \alpha_2$ \emph{is attacked by} $(AP_1,FP_1) \vdash \alpha_1$ or that $(AP_1,FP_1) \vdash \alpha_1$ is an \emph{attacker of} $(AP_2,FP_2) \vdash \alpha_2$.\\
 A set of arguments $X$ \emph{attacks} an argument $B$ if and only if there is an argument $A \in X$ which attacks $B$.
 A set of arguments $X_1$ attacks a set of arguments $X_2$ if and only if $X_1$ attacks some argument $B \in X_2$.
\end{definition}

This definition of attack is purely based on the notion of contrary of assumptions, i.e. 
fact-premises only occur as part of the argument but do not directly influence the attack relation.
Since arguments as introduced here and in \cite{assumption_based} correspond,
the attack relation in Definition~\ref{def:attacks} directly correspond to attacks in \cite{assumption_based}:
If an argument $(AP_1,FP_1) \vdash \alpha_1$ attacks an argument $(AP_2,FP_2) \vdash \alpha_2$ according to Definition~\ref{def:attacks}, then 
$AP_1 \vdash \alpha_1$ attacks $AP_2 \vdash \alpha_2$ as defined in \cite{assumption_based}.
Conversely, if $AP_1 \vdash \alpha_1$ attacks $AP_2 \vdash \alpha_2$ as defined in \cite{assumption_based},
then there exist $FP_1, FP_2 \subseteq \{\beta \; | \;\beta \leftarrow \; \in \mathcal{R}\}$
such that $(AP_1,FP_1) \vdash \alpha_1$ attacks $(AP_2,FP_2) \vdash \alpha_2$ according to Definition~\ref{def:attacks}.

%%%%%%%%%%%%%%%%%%%%%%%%%%%%%%%%%%%%%%%%%%%%%%%%%%%%%%%%%%%%%%%%%%%%%%%%%%%%%%%%%%%%%%%%%%%%%%%%%%%%%%%%%%%%%%%%%%%%%%%%%%%%%%%%%%%%%%%%%%%%%%%%%%%%%%%%%%%%%%%%%%
\subsection{ABA semantics}

The semantics of argumentation frameworks are given in terms of extensions, i.e. sets of arguments deemed to be ``winning''.
For our purposes we focus on the admissible and on the stable extension semantics introduced in \cite{Dung} for Abstract Argumentation and in \cite{aba_lp} for ABA.
Let $\langle \mathcal{L}, \mathcal{R}, \mathcal{A},\, \bar{\;}\,\rangle$ be an ABA framework and let $X$ be a set of arguments in $\langle \mathcal{L}, \mathcal{R}, \mathcal{A},\, \bar{\;}\,\rangle$.
\begin{itemize}
\item $X$ \emph{defends} an argument $A$ if and only if $X$ attacks all attackers of $A$.
\item $X$ is an \emph{admissible extension} of $\langle \mathcal{L}, \mathcal{R}, \mathcal{A},\, \bar{\;}\,\rangle$
if and only if $X$ does not attack itself and $X$ defends all arguments in $X$.
\item $X$ is a \emph{stable extension} of $\langle \mathcal{L}, \mathcal{R}, \mathcal{A},\, \bar{\;}\,\rangle$ if and only if $X$ does not attack itself
and $X$ attacks each argument in $\langle \mathcal{L}, \mathcal{R}, \mathcal{A},\, \bar{\;}\,\rangle$ which does not belong to $X$,
or, equivalently, if and only if $X = \{A \textit{ in } \langle \mathcal{L}, \mathcal{R}, \mathcal{A},\, \bar{\;}\,\rangle\; | \; X \textit{ does not attack } A \}$.
\end{itemize}

Admissible extensions can also be defined using trees of attacking arguments.

An \emph{abstract dispute tree} \cite{abaDialectical} for an ABA argument $A$ is a (possibly infinite) tree such that:
\begin{enumerate}
 \item Every node in the tree is labelled by an argument and is assigned the status of \emph{proponent} or \emph{opponent} node, but not both.
 \item The root is a proponent node labelled by $A$.
 \item For every proponent node $N$ labelled by an argument $B$ and for every argument $C$ attacking $B$, there exists a child of $N$
 which is an opponent node labelled by $C$.
 \item For every opponent node $N$ labelled by an argument $B$, there exists exactly one child of $N$
 which is a proponent node labelled by an argument which attacks $B$.
 \item There are no other nodes in the tree except those given by 1-4 above.
\end{enumerate}

An abstract dispute tree is \emph{admissible} \cite{assumption_based} if and only if no argument labels both a proponent and an opponent node.
It has been shown that the set of all arguments labelling proponent nodes in an admissible dispute tree is an admissible extension \cite{aba_abstract}.
We will use this result to characterize our justification approaches.

We now look at some properties of the stable extension semantics which will be used throughout the paper.
Lemma~\ref{lem:stable_ext} characterizes a stable extension in terms of
the assumption-premises of arguments contained in this stable extension as all arguments not attacked by this stable extension.
\begin{lemma}
 \label{lem:stable_ext}
 Let $\langle \mathcal{L}, \mathcal{R}, \mathcal{A},\, \bar{\;}\,\rangle$ be an ABA framework and let $X$ be a set of arguments in $\langle \mathcal{L}, \mathcal{R}, \mathcal{A},\, \bar{\;}\,\rangle$.
 $X$ is a stable extension of $\langle \mathcal{L}, \mathcal{R}, \mathcal{A},\, \bar{\;}\,\rangle$ if and only if $X = \{(AP,FP) \vdash \alpha\; | \; AP \subseteq \Lambda_X\}$ where
 $\Lambda_X = \{ \beta \in \mathcal{A} \; | \; \nexists (AP,FP) \vdash \overline{\beta} \in X\}$.
\end{lemma}

\begin{proof}
Similar to the proof of Theorem 3.10 in \cite{aba_lp}:
By the definition of stable extension, $X$ is a stable extension if and only if
$X  = \{A \textit{ in } \langle \mathcal{L}, \mathcal{R}, \mathcal{A},\, \bar{\;}\,\rangle\; |\\ \; X \textit{ does not attack } A \}$.
Then, $X = \{ (AP_1,FP_1) \vdash \alpha_1 \; | \; \nexists (AP_2,FP_2) \vdash \alpha_2 \in X\\ \textit{attacking } (AP_1,FP_1) \vdash \alpha_1\}$
by Definitions \ref{def:argument} and \ref{def:attacks}, and
$X = \{ (AP_1,FP_1) \vdash \alpha_1 \; | \; \nexists (AP_2,FP_2) \vdash \alpha_2 \in X \textit{ s.t. }  \beta \in AP_1,\overline{\beta} = \alpha_2\}$
by Definition~\ref{def:attacks}.
This can be split into $X = \{ (AP_1,FP_1) \vdash \alpha_1 \; | \; \forall \beta \in AP_1: \beta \in \Lambda_X\}$
where $\Lambda_X = \{ \beta \in \mathcal{A} \; | \; \nexists (AP_2,FP_2) \vdash  \alpha_2 \in X \textit{ s.t. } \alpha_2 = \overline{\beta}\}$.
\hfill
\end{proof}

After defining a stable extension in terms of the properties of its arguments,
we now take a closer look at conditions for an argument to be or not to be contained in a stable extension.
The following lemma characterizes the arguments contained in a stable extension:
An argument is part of a stable extension if and only if the assumption-arguments of all its assumption-premises
and the fact-arguments of all its fact-premises are in this stable extension.

\begin{lemma}
 \label{lem:stable_ext_support}
 Let $\langle \mathcal{L}, \mathcal{R}, \mathcal{A},\, \bar{\;}\,\rangle$ be an ABA framework and let $X$ be a stable extension of $\langle \mathcal{L}, \mathcal{R}, \mathcal{A},\, \bar{\;}\,\rangle$.
 $(AP, FP) \vdash \alpha \in X$ if and only if $\forall \beta \in AP$ it holds that $(\{\beta\}, \emptyset) \vdash \beta \in X$ and
 $\forall \gamma \in FP$ it holds that $(\emptyset, \{\gamma\}) \vdash \gamma \in X$.
\end{lemma}

\begin{proof}
 Note that fact-arguments are always part of a stable extension as they cannot be attacked, so we only focus on assumption-arguments.
 \begin{itemize}
  \item From left to right: If $(AP, FP) \vdash \alpha \in X$ then by Lemma~\ref{lem:stable_ext} $\forall \beta \in AP$, $(AP_1,FP_1) \vdash \overline{\beta} \notin X$.
  Consequently, $(\{\beta\}, \emptyset) \vdash \beta$ is not attacked by $X$,
  so by definition of stable extension $(\{\beta\}, \emptyset) \vdash \beta \in X$.
  \item From right to left: If $\forall \beta \in AP$ it holds that $(\{\beta\}, \emptyset) \vdash \beta \in X$
  then by definition of stable extension no $(\{\beta\}, \emptyset) \vdash \beta$ is attacked by $X$,
  so for none of the $\beta \in AP$ there exists an $(AP_1, FP_1) \vdash \overline{\beta} \in X$.
  Thus, $(AP, FP) \vdash \alpha$ is not attacked by $X$, so $(AP, FP) \vdash \alpha \in X$.  
 \end{itemize}\hfill
\end{proof}

The following lemma characterizes conditions for an argument not to be in a given stable extension:
An argument is not part of a stable extension if and only if the assumption-argument of one of its assumption-premises is not in this stable extension:

\begin{lemma}
\label{lem:stable_ext_suppport_notin}
 Let $\langle \mathcal{L}, \mathcal{R}, \mathcal{A},\, \bar{\;}\,\rangle$ be an ABA framework and let $X$ be a stable extension of $\langle \mathcal{L}, \mathcal{R}, \mathcal{A},\, \bar{\;}\,\rangle$.
 $(AP, FP) \vdash \alpha \notin X$ if and only if $\exists \beta \in AP$ such that $(\{\beta\}, \emptyset) \vdash \beta \notin X$.
\end{lemma}

\begin{proof}
 \begin{itemize}
  \item From left to right: If $(AP, FP) \vdash \alpha \notin X$ then $(AP, FP) \vdash \alpha$ is attacked by $X$ on some $\beta \in AP$.
   Consequently, $(\{\beta\}, \emptyset) \vdash \beta$ is attacked by $X$, so  $(\{\beta\}, \emptyset) \vdash \beta \notin X$.
  \item From right to left: If $\exists \beta \in AP$ such that $(\{\beta\}, \emptyset) \vdash \beta \notin X$ then
  $(\{\beta\}, \emptyset) \vdash \beta$ is attacked by $X$, meaning that there is some $(AP_1, FP_1) \vdash \overline{\beta} \in X$. 
  Thus, $(AP, FP) \vdash \alpha$ is attacked by $X$ on $\beta$, so $(AP, FP) \vdash \alpha \notin X$. 
 \end{itemize}\hfill
\end{proof}

%%%%%%%%%%%%%%%%%%%%%%%%%%%%%%%%%%%%%%%%%%%%%%%%%%%%%%%%%%%%%%%%%%%%%%%%%%%%%%%%%%%%%%%%%%%%%%%%%%%%%%%%%%%%%%%%%%%%%%%%%%%%%%%%%%%%%%%%%%%%%%%%%%%%%%%%%%%%%%%%%%
%%%%%%%%%%%%%%%%%%%%%%%%%%%%%%%%%%%%%%%%%%%%%%%%%%%%%%%%%%%%%%%%%%%%%%%%%%%%%%%%%%%%%%%%%%%%%%%%%%%%%%%%%%%%%%%%%%%%%%%%%%%%%%%%%%%%%%%%%%%%%%%%%%%%%%%%%%%%%%%%%%
\section{Translating a logic program into an ABA framework}
\label{sec:transatlion}

In order to use ABA for the justification of literals with respect to an answer set of a \emph{consistent} logic program,
the logic program has to be expressed as an ABA framework first.

%%%%%%%%%%%%%%%%%%%%%%%%%%%%%%%%%%%%%%%%%%%%%%%%%%%%%%%%%%%%%%%%%%%%%%%%%%%%%%%%%%%%%%%%%%%%%%%%%%%%%%%%%%%%%%%%%%%%%%%%%%%%%%%%%%%%%%%%%%%%%%%%%%%%%%%%%%%%%%%%%%
\subsection{The translation}
\label{sec:translation_theory}

We use the approach of \cite{aba_lp} for translating a logic program into an ABA framework,
where the clauses of a logic program form the set of ABA rules and NAF literals are used as assumptions in ABA.

\begin{definition}[Translated ABA framework]
\label{def:translated_ABA}
 Let $\mathcal{P}$ be a logic program.
 $ABA_{\mathcal{P}} = \langle \mathcal{L}_{\mathcal{P}}, \mathcal{R}_{\mathcal{P}}, \mathcal{A}_{\mathcal{P}},\: \bar{\; }\:\rangle$ is the \emph{translated ABA framework} of $\mathcal{P}$ where:
 \begin{itemize}
  \item $\mathcal{R}_{\mathcal{P}} = \mathcal{P}$
  \item $\mathcal{A}_{\mathcal{P}} = \textit{NAF}_{\mathcal{P}}$
  \item for every $not ~ l \in \mathcal{A}_{\mathcal{P}}$: $\overline{not ~ l} = l$
  \item $\mathcal{L}_{\mathcal{P}} = Lit_{\mathcal{P}} \; \cup \; \textit{NAF}_{\mathcal{P}}$
 \end{itemize}
 Note that the clauses of a logic program can be directly used as rules in the translated ABA framework as we utilize the same notation for both of them.
 Note also that translated ABA frameworks are always flat since NAF literals do not occur in the head of clauses of a logic program.
\end{definition}

\begin{example}
\label{ex:p1}
 The following logic program $\mathcal{P}_1$ will serve as a running example throughout the paper,
 where $Lit_{\mathcal{P}_1} = \{ a, \neg a, c, \neg c, d, \neg d, e, \neg e \}$:
 \begin{align*}
  a &\leftarrow not~ \neg a\\
  a &\leftarrow \neg a, not~ c, not~ e\\
  \neg a &\leftarrow not~ c, not~ d\\
  c &\leftarrow not~ e\\
  d &\leftarrow not~ \neg a\\
  e &\leftarrow
 \end{align*}
The translated ABA framework of $\mathcal{P}_1$ is $ABA_{\mathcal{P}_1} = \langle \mathcal{L}_{\mathcal{P}_1}, \mathcal{R}_{\mathcal{P}_1}, \mathcal{A}_{\mathcal{P}_1}, \: \bar{\; }\: \rangle$ with:
\begin{itemize}
 \item $\mathcal{R}_{\mathcal{P}_1} = \mathcal{P}_1$
 \item $\mathcal{A}_{\mathcal{P}_1} = NAF_{\mathcal{P}_1} = \{not~ a, not~ \neg a, not~ c, not~ \neg c, not~ d, not~ \neg d, not~ e, not~ \neg e\}$
 \item $\overline{not~ a} = a;\; \overline{not~ \neg a} = \neg a;\; \overline{not~c} = c;\; \overline{not~ \neg c} = \neg c;\;
 \overline{not~ d} = d;\; \overline{not~ \neg d} = \neg d;\\ \overline{not~e} = e;\; \overline{not~ \neg e} = \neg e$
 \item $\mathcal{L}_{\mathcal{P}_1} = Lit_{\mathcal{P}_1}\; \cup \; NAF_{\mathcal{P}_1}$
\end{itemize}
The following fourteen arguments can be constructed in $ABA_{\mathcal{P}_1}$, including eight assumption-arguments ($A_1$ - $A_8$) and one fact-argument ($A_{14}$):\\
$A_1: (\{not~a\}, \emptyset) \vdash not~ a$\\
$A_2: (\{not~ \neg a\}, \emptyset) \vdash not~ \neg a$\\
$A_3: (\{not~c\}, \emptyset) \vdash not~ c$\\
$A_4: (\{not~ \neg c\}, \emptyset) \vdash not~ \neg c$\\
$A_5: (\{not~d\}, \emptyset) \vdash not~ d$\\
$A_6: (\{not~\neg d\}, \emptyset) \vdash not~ \neg d$\\
$A_7: (\{not~e\}, \emptyset) \vdash not~ e$\\
$A_8: (\{not~\neg e\}, \emptyset) \vdash not~ \neg e$\\
$A_9: (\{not~ \neg a\}, \emptyset) \vdash a$\\
$A_{10}: (\{not~c, not~d, not~ e\}, \emptyset) \vdash a$\\
$A_{11}: (\{not~c, not~d\}, \emptyset) \vdash \neg a$\\
$A_{12}: (\{not~ e\}, \emptyset) \vdash c$\\
$A_{13}: (\{not~ \neg a\}, \emptyset) \vdash d$\\
$A_{14}: (\emptyset, \{e\}) \vdash e$\\
The attacks between these arguments are given as a graph in Figure~\ref{fig:attacks_p1}.
An arrow from a node $N_1$ to a node $N_2$ in the graph represents that the argument held by $N_1$ attacks the argument held by $N_2$.
\end{example}

\begin{figure}[ht]
 \centering
 \includegraphics[width=\textwidth]{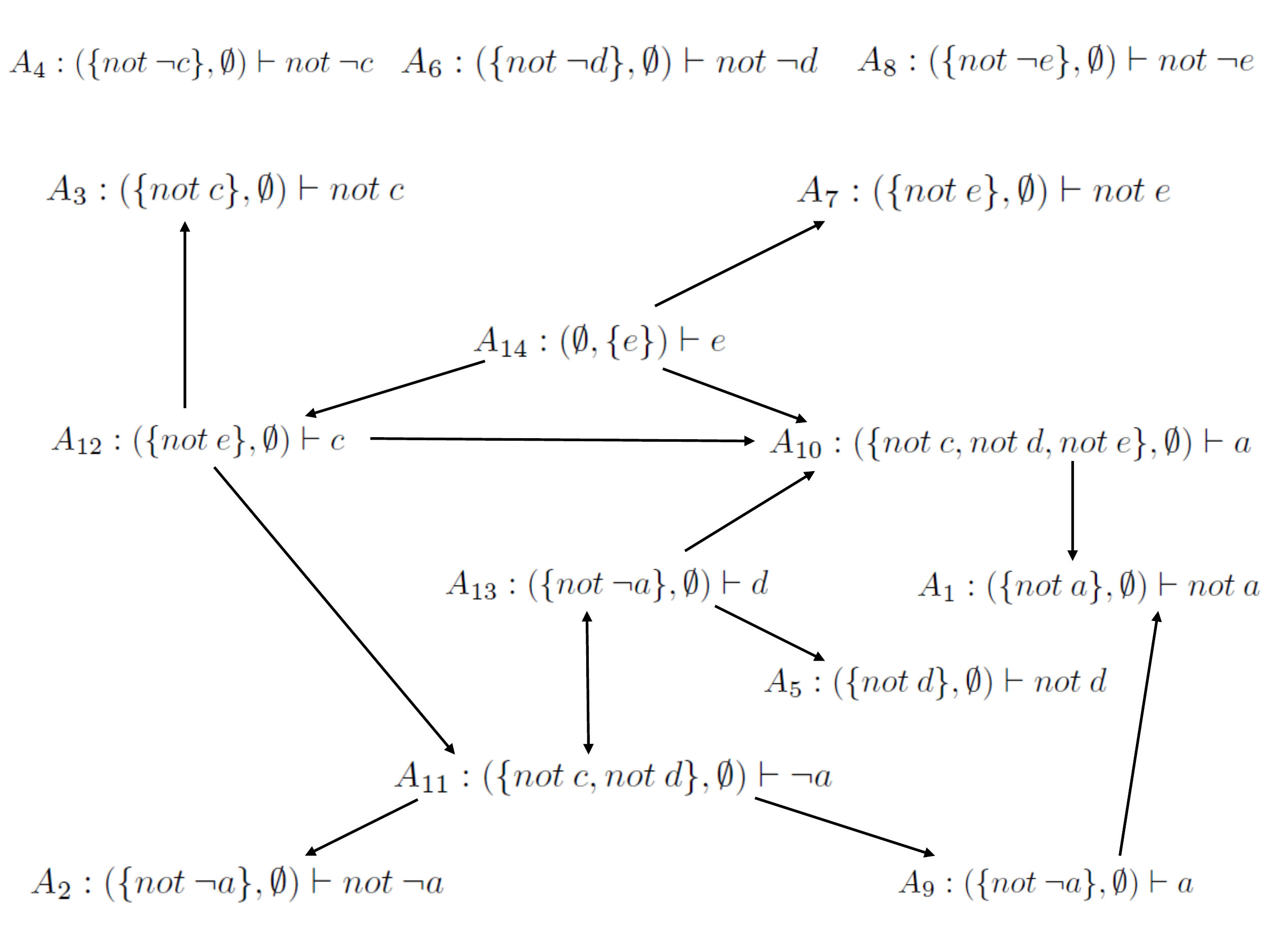}
 \caption{Attacks between arguments in the translated ABA framework of $\mathcal{P}_1$ (see Example~\ref{ex:p1}) represented by arrows between the arguments.
 The three arguments at the top ($A_4$, $A_6$, $A_8$) are neither attacked nor do they attack other arguments,
 so there are no arrows connecting them with other arguments.}
 \label{fig:attacks_p1}
\end{figure}

%%%%%%%%%%%%%%%%%%%%%%%%%%%%%%%%%%%%%%%%%%%%%%%%%%%%%%%%%%%%%%%%%%%%%%%%%%%%%%%%%%%%%%%%%%%%%%%%%%%%%%%%%%%%%%%%%%%%%%%%%%%%%%%%%%%%%%%%%%%%%%%%%%%%%%%%%%%%%%%%%%
\subsection{Correspondence between Answer Sets and Stable Extensions}
\label{sec:translation_lemmas}

In this section, we describe the relationship between answer sets of a logic program and stable extensions of the translated ABA framework.
This connection will be used for our justification approaches.
Theorem~\ref{lem:corr_ext} states that an answer set with NAF literals consists of the conclusions of all arguments in the ``corresponding'' stable extension.
Conversely, Theorem~\ref{lem:corr_stable} expresses that a stable extension consists of all arguments supported by NAF literals which are satisfied with respect to the ``corresponding'' answer set.
Note that part of this correspondence has been stated without a formal proof in \cite{aba_lp}.

\begin{theorem}
\label{lem:corr_ext}
Let $\mathcal{P}$ be a logic program and let 
$ABA_{\mathcal{P}} = \langle \mathcal{L}_{\mathcal{P}}, \mathcal{R}_{\mathcal{P}}, \mathcal{A}_{\mathcal{P}},\: \bar{\; }\: \rangle$.
Let $X$ be a set of arguments in $ABA_{\mathcal{P}}$ and let $T = \{k \;|\; \exists (AP, FP) \vdash k \in X\}$ be the set of all conclusions of arguments in $X$.\\
$X$ is a stable extension of $ABA_{\mathcal{P}}$ if and only if
$T$ is an answer set with NAF literals of $\mathcal{P}$.
\end{theorem}

\begin{proof}
\begin{itemize}
 \item $X$ is a stable extension of $ABA_{\mathcal{P}}$
 \item iff $X = \{(AP_1,FP_1) \vdash k \textit{ is an argument in }ABA_{\mathcal{P}}\; | \; AP_1 \subseteq \Lambda_{X}\}$ with\\
 $\Lambda_{X} = \{ not~ l \in \mathcal{A}_{\mathcal{P}} \; | \; \nexists (AP_2,FP_2) \vdash l \in X\}$ (by Lemma~\ref{lem:stable_ext})
 \item iff $X = \{(AP_1,FP_1) \vdash k \; | \; AP_1 \subseteq \Lambda_{X}, \mathcal{R}_{\mathcal{P}} \cup \Lambda_{X} \vdash_{MP} k\}$ with\\
 $\Lambda_{X} = \{ not~ l \in \mathcal{A}_{\mathcal{P}} \; | \; \nexists (AP_2,FP_2) \vdash l \in X\}$ (by Lemma~\ref{lem:arg_mp})
 \item iff $X = \{(AP_1,FP_1) \vdash k \; | \; AP_1 \subseteq \Lambda_{X}, \mathcal{P}\; \cup\; \Lambda_{X} \vdash_{MP} k\}$ with\\
 $\Lambda_{X} = \{ not~ l \in \textit{NAF}_{\mathcal{P}} \; | \; \nexists (AP_2,FP_2) \vdash l \in X\}$ (by Definition~\ref{def:translated_ABA})
 \item iff $X = \{(AP_1,FP_1) \vdash k \; | \; AP_1 \subseteq \Lambda_{X}, \mathcal{P}\; \cup\; \Lambda_{X} \vdash_{MP} k\}$ with\\
 $\Lambda_{X} = \{ not~ l \in \textit{NAF}_{\mathcal{P}} \; | \; l \notin T\}$ and $T = \{k \;|\; \exists (AP, FP) \vdash k \in X\}$\\ (by construction of $T$, see above)
 \item iff $X = \{(AP_1,FP_1) \vdash k \; | \; AP_1 \subseteq \Delta_T, \mathcal{P}\; \cup\; \Delta_T \vdash_{MP} k\}$ with\\
 $\Delta_T = \{ not~ l \in \textit{NAF}_{\mathcal{P}} \; | \; l \notin T\}$ and $T = \{k \;|\; \exists (AP, FP) \vdash k \in X\}$ (by Definition~\ref{def:AS_NAF})
 \item iff $\Delta_T = \{ not~ l \in \textit{NAF}_{\mathcal{P}} \; | \; l \notin T\}$ and $T = \{k \;|\; \mathcal{P}\; \cup\; \Delta_T \vdash_{MP} k\}$\\
 (substituting $X$ in $T$)
 \item iff $T$ is an answer set with NAF literals of $\mathcal{P}$ (by Lemma~\ref{lem:AS_MP}) 
\end{itemize} \hfill
\end{proof}

\begin{theorem}
\label{lem:corr_stable}
Let $\mathcal{P}$ be a logic program and let 
$ABA_{\mathcal{P}} = \langle \mathcal{L}_{\mathcal{P}}, \mathcal{R}_{\mathcal{P}}, \mathcal{A}_{\mathcal{P}},\: \bar{\; }\: \rangle$.
Let $T \subseteq Lit_{\mathcal{P}}$ be a set of classical literals and let $X = \{ (AP,FP) \vdash k \; | \; AP \subseteq \Delta_T\}$ be the set of arguments in $ABA_{\mathcal{P}}$
whose assumption-premises are in $\Delta_T$.\\
$T$ is an answer set of $\mathcal{P}$ if and only if $X$ is a stable extension of $ABA_{\mathcal{P}}$.
\end{theorem}

\begin{proof}
 \begin{itemize}
  \item $T$ is an answer set of $\mathcal{P}$
  \item iff $T = \{ l_1 \in Lit_{\mathcal{P}}\; |\; \mathcal{P} \; \cup \; \Delta_T \vdash_{MP} l_1\}$ with\\
  $\Delta_T = \{ not ~ l \in \textit{NAF}_{\mathcal{P}}\; | \; l \notin T \}$ (by Lemma~\ref{lem:AS_MP} and Definition~\ref{def:AS_NAF})
  \item  iff $T = \{ k  \in \mathcal{L}_{\mathcal{P}} \backslash \mathcal{A}_{\mathcal{P}} \; |\; \mathcal{R}_{\mathcal{P}} \; \cup \; \Delta_T \vdash_{MP} k\}$ with\\
  $\Delta_T = \{ not ~ l \in \mathcal{A}_{\mathcal{P}}\; | \; l \notin T \}$ (by Definition~\ref{def:translated_ABA})
  \item iff $T = \{ k  \in \mathcal{L}_{\mathcal{P}} \backslash \mathcal{A}_{\mathcal{P}} \; |\; \exists (AP,FP) \vdash k, AP \subseteq \Delta_T\}$ with\\
  $\Delta_T = \{ not ~ l \in \mathcal{A}_{\mathcal{P}}\; | \; l \notin T \}$ (by Lemma~\ref{lem:arg_mp})
  \item iff $T = \{ k  \in \mathcal{L}_{\mathcal{P}} \backslash \mathcal{A}_{\mathcal{P}} \; |\; \exists (AP,FP) \vdash k, AP \subseteq \Delta_T\}$ with\\
  $\Delta_T = \{ not ~ l \in \mathcal{A}_{\mathcal{P}}\; | \; \mathcal{P} \; \cup \; \Delta_T \nvdash_{MP} l \}$ (by Lemma~\ref{lem:AS_MP})
  \item iff $T = \{ k  \in \mathcal{L}_{\mathcal{P}} \backslash \mathcal{A}_{\mathcal{P}} \; |\; \exists (AP,FP) \vdash k, AP \subseteq \Delta_T\}$ with\\
  $\Delta_T = \{ not ~ l \in \mathcal{A}_{\mathcal{P}}\; | \; \nexists (AP,FP) \vdash l, AP \subseteq \Delta_T\}$ (by Lemma~\ref{lem:arg_mp})
  \item iff $T = \{ k  \in \mathcal{L}_{\mathcal{P}} \backslash \mathcal{A}_{\mathcal{P}} \; |\; \exists (AP,FP) \vdash k, AP \subseteq \Delta_T\}$ with\\
  $\Delta_T = \{ not ~ l \in \mathcal{A}_{\mathcal{P}}\; | \; \nexists (AP,FP) \vdash l \in X\}$ and $X = \{ (AP,FP) \vdash k \; | \; AP \subseteq \Delta_T\}$ (by construction of $X$, see above)
  \item iff $T = \{ k  \in \mathcal{L}_{\mathcal{P}} \backslash \mathcal{A}_{\mathcal{P}} \; |\; \exists (AP,FP) \vdash k, AP \subseteq \Lambda_{X}\}$ with\\
  $\Lambda_{X} = \{ not ~ l \in \mathcal{A}_{\mathcal{P}}\; | \; \nexists (AP,FP) \vdash l \in X\}$ and $X = \{ (AP,FP) \vdash k \; | \; AP \subseteq \Lambda_{X}\}$  (by Lemma~\ref{lem:stable_ext})
  \item iff $\Lambda_{X} = \{ not ~ l \in \mathcal{A}_{\mathcal{P}}\; | \; \nexists (AP,FP) \vdash l \in X\}$ and $X = \{ (AP,FP) \vdash k \; | \; AP \subseteq \Lambda_{X}\}$ 
  \item iff $X$ is a stable extension of $ABA_{\mathcal{P}}$ (by Lemma~\ref{lem:stable_ext}) 
 \end{itemize}\hfill
\end{proof}

\begin{example}
 \label{ex:p1_as_stable}
 The logic program $\mathcal{P}_1$ from Example~\ref{ex:p1} has two answer sets: $S_1 = \{e,d,a\}$ and $S_2 = \{e,\neg a\}$.
 The respective sets of satisfied NAF literals are $\Delta_{S_1} = \{not ~ \neg a, not~ c,\\ not~ \neg c, not~ \neg d, not~ \neg e\}$ and
 $\Delta_{S_2} = \{not~ a, not~ c, not~ \neg c, not~ d, not ~ \neg d, not \neg e\}$.
 Considering the attacks between arguments in the translated ABA framework $ABA_{\mathcal{P}_1}$ (see Figure~\ref{fig:attacks_p1}), two stable extensions can be determined for $ABA_{\mathcal{P}_1}$:
 $A_4$, $A_6$, $A_8$, and $A_{14}$ have to be part of all stable extensions as they are not attacked.
 Then, $A_7$, $A_{10}$, and $A_{12}$ cannot be in any stable extension as they are attacked by $A_{14}$.
 Consequently, $A_3$ is part of all stable extensions since it is only attacked by $A_{12}$, which is definitely not contained in any stable extension.
 As $A_{11}$ and $A_{13}$ attack each other and are not furthered attacked by other arguments, there are two stable extensions,
 one containing $A_{13}$ and the other one containing $A_{11}$.
 The first stable extension also comprises $A_2$ and $A_9$ as $A_{13}$ attacks all their attackers,
 whereas the second one additionally comprises $A_1$ and $A_5$ since $A_{11}$ attacks all their attackers.
 Thus, the two stable extensions of $ABA_{\mathcal{P}_1}$ are $\mathcal{E}_1 = \{A_2,A_3,A_4,A_6,A_8, A_9,A_{13},A_{14}\}$ and
 $\mathcal{E}_2 = \{A_1,A_3,A_4,A_5,A_6,A_8,A_{11},A_{14}\}$.
 As expected, the conclusions of arguments in the stable extensions,
 $\{not ~ \neg a, not~ c, not~ \neg c, not~ \neg d, not~ \neg e, a, d, e\}$ for  $\mathcal{E}_1$ and
 $\{not~ a, not~ c, not~ \neg c, not~d, not~ \neg d, not~ \neg e, \neg a, e\}$ for $\mathcal{E}_2$,
 coincide with $S_{1_{NAF}}$ and $S_{2_{NAF}}$, as stated in Theorem~\ref{lem:corr_ext}.
 Conversely, the two sets of arguments whose assumption-premises are subsets of $\Delta_{S_1}$ and $\Delta_{S_2}$, respectively,
 coincide with the two stable extensions $\mathcal{E}_1$ and $\mathcal{E}_2$, respectively,
 as stated in Theorem~\ref{lem:corr_stable}.
\end{example}

The following notation introduces some terminology to refer to the stable extension which corresponds to a given answer set.

\begin{notation}
\label{not:corr_arg_ext}
Given an answer set $S$ of $\mathcal{P}$ and a stable extension $\mathcal{E}$ of $ABA_{\mathcal{P}}$ such that
$S_{\textit{NAF}} = \{k \;|\;\exists (AP, FP) \vdash k \in \mathcal{E}\}$,
$\mathcal{E}$ is called the \emph{corresponding stable extension} of $S$.
Given a literal $k \in S_{\textit{NAF}}$ and the corresponding stable extension $\mathcal{E}$ of $S$, an argument $A \in \mathcal{E}$ with conclusion $k$ is called a \emph{corresponding argument} of $k$.
\end{notation}

It is easy to show that for every literal $k$ in an answer set with NAF literals there is at least one corresponding argument in the corresponding stable extension.
Conversely, if a literal $k$ is not contained in an answer set with NAF literals, then no argument with conclusion $k$ is part of the corresponding stable extension.
\begin{theorem}
\label{the:AsStable}
Let $\mathcal{P}$ be a logic program,
$S$ an answer set of $\mathcal{P}$, and $\mathcal{E}$ the corresponding stable extension of $S$ in $ABA_{\mathcal{P}}$.
Let $k \in Lit_{\mathcal{P}} \; \cup \; \textit{NAF}_{\mathcal{P}}$.
\begin{enumerate}
 \item If $k \in S_{\textit{NAF}}$, then there exists an argument $A \in \mathcal{E}$ such that $A: (AP, FP) \vdash k$ with
  $AP \subseteq \Delta_S$ and $FP \subseteq S$.
 \item If $k \notin S_{\textit{NAF}}$, then there exists no $A: (AP, FP) \vdash k$ in $ABA_{\mathcal{P}} $ such that
 $A \in \mathcal{E}$.
\end{enumerate}
\end{theorem}

\begin{proof}
\begin{enumerate}
 \item By Theorem~\ref{lem:corr_ext}, $S_{NAF} = \{k_1 \;|\; \exists (AP, FP) \vdash k_1 \in \mathcal{E}\}$,
 so if $k \in S_{NAF}$ then there exists at least one argument $A: (AP, FP) \vdash k \: \in \mathcal{E}$.
  By Theorem~\ref{lem:corr_stable}, $\mathcal{E} = \{ (AP_1,FP_1) \vdash k_1 \; | \; AP_1 \subseteq \Delta_S\}$,
  so it follows that for argument A, $AP \subseteq \Delta_S$.
  Furthermore, $FP \subseteq S$ because $FP \subseteq \{\beta\; | \; \beta \leftarrow \: \in \mathcal{P}\}$ and for consistent logic programs
  it trivially holds that $\{\beta\; | \; \beta \leftarrow \: \in \mathcal{P}\} \subseteq S$.
  \item Assume that there exists $A: (AP, FP) \vdash k$ in $ABA_{\mathcal{P}}$ such that $A \in \mathcal{E}$.
  Then according to Theorem~\ref{lem:corr_ext}, $k \in S_{\textit{NAF}}$.
  Contradiction.   
\end{enumerate}\hfill
\end{proof}

\begin{example}
 \label{p1_correspondingArg}
 As demonstrated in Example \ref{ex:p1_as_stable}, the answer sets of $\mathcal{P}_1$ correspond to the stable extensions of $ABA_{\mathcal{P}_1}$,
 where $S_1$ corresponds to $\mathcal{E}_1$ and $S_2$ corresponds to $\mathcal{E}_2$.
 When taking a closer look at $S_{1_{NAF}}$, we can verify that every literal has a corresponding argument in $\mathcal{E}_1$:
 $e$ has $A_{14}$, $d$ has $A_{13}$, $a$ has $A_9$, $not~ \neg a$ has $A_2$, $not~ c$ has $A_3$, and so on.
 Furthermore, for all literals not contained in $S_{1_{NAF}}$, there is no argument with this conclusion in the stable extension $\mathcal{E}_1$,
 e.g. $\neg a \notin S_{1_{NAF}}$ and $A_{11} \notin \mathcal{E}_1$.
 The same relationship holds between $S_2$ and $\mathcal{E}_2$.
\end{example}

 Note that the first part of Theorem~\ref{the:AsStable} only states that for a literal $k$ in the answer set with NAF literals
 there exists a corresponding argument in the corresponding stable extension.
 However, there might be further arguments $(AP, FP) \vdash k$ which are not part of the corresponding stable extension, where $AP \nsubseteq \Delta_S$.
  Note also that the second part of Theorem~\ref{the:AsStable} does not exclude the existence of arguments with conclusion $k$.
  It merely states that no such argument is contained in the corresponding stable extension.
  
  \begin{example}
   Consider the logic program $\mathcal{P}_1$ and its answer set $S_1$.
   $a \in S_1$ has the corresponding argument $A_9$ in $\mathcal{E}_1$, but there is another argument with conclusion $a$ in $ABA_{\mathcal{P}_1}$
   which is not in $\mathcal{E}_1$, namely $A_{10}$.
   Furthermore, $c \notin S_1$ but there exists an argument with conclusion $c$ in $ABA_{\mathcal{P}_1}$, namely $A_{12}$.
  As expected, this argument is not contained in the corresponding stable extension $\mathcal{E}_1$.
  \end{example}

  Theorem~\ref{the:AsStable}, part 1, provides the starting point for our justification approaches as it allows us to explain
  why a literal is in an answer set based
  on the reasons for a corresponding argument to be in the corresponding stable extension.
  Similarly, Theorem~\ref{the:AsStable}, part 2, is a starting point for justifying that a literal is not contained in an answer set
  based on arguments for that literal, all of which are not contained in the corresponding stable extension.
  In ABA it is easy to explain why an argument is or is not contained in a stable extension:
  An argument is part of a stable extension if it is not attacked by it.
  Since the stable extension attacks all arguments which are not part of it, this entails that an argument in the stable extension is defended
  by the stable extension, i.e. the stable extension attacks all attackers of this argument.
  Conversely, an argument is not part of a stable extension if it is attacked by this stable extension.
  In the following section, we will make use of these results in order to develop a justification method that provides explanations in terms of arguments and attacks between them.

%%%%%%%%%%%%%%%%%%%%%%%%%%%%%%%%%%%%%%%%%%%%%%%%%%%%%%%%%%%%%%%%%%%%%%%%%%%%%%%%%%%%%%%%%%%%%%%%%%%%%%%%%%%%%%%%%%%%%%%%%%%%%%%%%%%%%%%%%%%%%%%%%%%%%%%%%%%%%%%%%%
%%%%%%%%%%%%%%%%%%%%%%%%%%%%%%%%%%%%%%%%%%%%%%%%%%%%%%%%%%%%%%%%%%%%%%%%%%%%%%%%%%%%%%%%%%%%%%%%%%%%%%%%%%%%%%%%%%%%%%%%%%%%%%%%%%%%%%%%%%%%%%%%%%%%%%%%%%%%%%%%%%
\section{Attack Trees}
\label{sec:attackTrees}

Our first justification approach explains why arguments are or are not contained in a stable extension by constructing
an \emph{Attack Tree} of this argument with respect to the stable extension.
This tree of attacking arguments is later used to construct a justification in terms of literals:
Due to the correspondence between answer sets and stable extensions,
a justification of a literal $k$ with respect to an answer set can be obtained from an Attack Tree of an argument with conclusion $k$
constructed with respect to the corresponding stable extension.
In this section we define the notion of Attack Trees and show their relationship with abstract dispute trees for ABA,
characterizing the explanations they provide as
admissible fragments of the stable extension as well as of the answer set if an Attack Tree is constructed with respect to a corresponding stable extension.

%%%%%%%%%%%%%%%%%%%%%%%%%%%%%%%%%%%%%%%%%%%%%%%%%%%%%%%%%%%%%%%%%%%%%%%%%%%%%%%%%%%%%%%%%%%%%%%%%%%%%%%%%%%%%%%%%%%%%%%%%%%%%%%%%%%%%%%%%%%%%%%%%%%%%%%%%%%%%%%%%%
\subsection{Constructing Attack Trees}
\label{sec:attackTrees_theory}

Nodes in an Attack Tree hold arguments which are labelled either \textquotesingle$+$\textquotesingle\ or \textquotesingle$-$\textquotesingle.
An Attack Tree of an argument $A$ has $A$ itself in the root node, where either one or all attackers of $A$ form the child nodes of this root.
In the same way, every of these child nodes holding some argument $B$ have either all or one of $B$'s attackers as children, and so on.
Whether only one or all attackers of an argument are considered as child nodes depends on the argument's label in the Attack Tree,
which is determined with respect to a given set of arguments
(typically a stable extension of the translated ABA framework).
If an argument is part of given set, it is labelled \textquotesingle$+$\textquotesingle\ and has all its attackers as child nodes.
If the argument is not contained in the set, it is labelled \textquotesingle$-$\textquotesingle\ and has exactly one of its attackers as a child node.
\begin{definition}[Attack Tree]
\label{def:attackTree}
 Let $\mathcal{P}$ be a logic program,
 $X$ a set of arguments in $ABA_{\mathcal{P}}$, and
 $A$ an argument in $ABA_{\mathcal{P}}$.
 An \emph{Attack Tree} of $A$ (constructed) with respect to $X$,
 denoted $attTree_{X}(A)$, is a (possibly infinite) tree such that:
  \begin{enumerate}
  \item Every node in $attTree_{X}(A)$ holds an argument in $ABA_{\mathcal{P}}$,
  labelled \textquotesingle$+$\textquotesingle\ or \textquotesingle$-$\textquotesingle.
  \item The root node is $A^+$ if $A \in X$ or $A^-$ if $A \notin X$.
  \item For every node $A_N^+$ and for every argument $A_i$ attacking $A_N$ in $ABA_{\mathcal{P}}$,
  there exists a child node $A_i^-$ of $A_N^+$.
  \item Every node $A_N^-$
  \begin{enumerate}[(i)]
   \item has no child node if $A_N$ is not attacked in $ABA_{\mathcal{P}}$ or if for all attackers $A_i$ of $A_N$: $A_i \notin X$; or else
   \item has exactly one child node $A_i^+$ for some $A_i \in X$ attacking $A_N$.
  \end{enumerate}
  \item There are no other nodes in $attTree_{X}(A)$ except those given in 1-4.
 \end{enumerate}
 \end{definition}

 If $attTree_{X}(A)$ is an Attack Tree of $A$ with respect to $X$ we also say that $A$ \emph{has the Attack Tree} $attTree_{X}(A)$.
 Note that due to condition 4(ii), where only one of possibly many arguments $A_i$ is chosen, an argument can have more than one Attack Tree.
 Furthermore, note the difference between 3, where $A_i$ is any argument attacking $A_N$, and 4(ii), where $A_i$ has to be an attacking argument contained in $X$.
 
\begin{notation}
\label{not:attTree_posneg}
 If $A \in X$, and thus the root node of $attTree_{X}(A)$ is $A^+$, we denote the Attack Tree as $attTree_{X}^+(A)$ and call it a \emph{positive Attack Tree}.
 If $A \notin X$, and thus the root node of $attTree_{X}(A)$ is $A^-$, we denote the Attack Tree as $attTree_{X}^-(A)$ and call it a \emph{negative Attack Tree}.
\end{notation}

The next example illustrates the notion of Attack Trees with respect to a set of arguments which is a stable extension.

\begin{example}
 \label{ex:p1_attackTrees}
 We consider the logic program $\mathcal{P}_1$ and its translated ABA framework $ABA_{\mathcal{P}_1}$ (see Example~\ref{ex:p1}).
 Figure~\ref{fig:p1_attackTrees_a6} shows the two negative Attack Trees of argument $A_{10}$ with respect to the stable extension $\mathcal{E}_1 = \{A_2,A_3,A_4,A_6,A_8, A_9,A_{13}, A_{14}\}$,
 i.e. $attTree_{\mathcal{E}_1}^-(A_{10})_1$ and $attTree_{\mathcal{E}_1}^-(A_{10})_2$.
 Since $A_{10} \notin \mathcal{E}_1$, the root node of all Attack Trees of $A_{10}$ holds $A_{10}^-$,
 and consequently has exactly one or not attacker of $A_{10}$ as a child node.
 $A_{10}$ is attacked by the three arguments $A_{12}$, $A_{13}$, and $A_{14}$ (see Figure~\ref{fig:attacks_p1}),
 so these are the candidates for being a child node of $A_{10}^-$.
 However, $A_{12}^+$ cannot serve as a child node of $A_{10}^-$ as $A_{12} \notin \mathcal{E}_1$ (see condition 4(ii) in Definition~\ref{def:attackTree}).
 Since both $A_{13}$ and $A_{14}$ are contained in $\mathcal{E}_1$, either of them can be used as a child node of $A_{10}^-$,
 leading to two possible Attack Trees of $A_{10}$.
  The left of Figure~\ref{fig:p1_attackTrees_a6} depicts the negative Attack Tree $attTree_{\mathcal{E}_1}^-(A_{10})_1$
  where $A_{14}^+$ is chosen as the child node of $A_{10}^-$, whereas the right illustrates $attTree_{\mathcal{E}_1}^-(A_{10})_2$ where $A_{13}^+$ is chosen.
 $attTree_{\mathcal{E}_1}^-(A_{10})_1$ ends with $A_{14}^+$ since $A_{14}$ is not attacked in $ABA_{\mathcal{P}_1}$.
 In contrast, choosing $A_{13}^+$ as the child node of $A_{10}^-$ leads to an infinite negative Attack Tree $attTree_{\mathcal{E}_1}^-(A_{10})_2$:
 $A_{13}^+$ has a single child $A_{11}^-$ since $A_{11}$ is the only argument attacking $A_{13}$;
 $A_{11}$ is attacked by both $A_{12}$ and $A_{13}$ in $\mathcal{P}_1$, but only $A_{13}^+$ can serve as a child node of $A_{11}^-$ as $A_{12} \notin \mathcal{E}_1$;
 at this point, the Attack Tree starts to repeat itself, since the only possible child node of $A_{11}^-$ is $A_{13}^+$, whose only child node is $A_{11}^-$, and so on.\\
 With respect to the stable extension $\mathcal{E}_2 = \{A_1,A_3,A_4,A_5,A_6,A_8,A_{11}, A_{14}\}$ (see Example~\ref{ex:p1_as_stable}),
 $A_{10}$ has a unique negative Attack Tree $attTree_{\mathcal{E}_2}^-(A_{10})$,
 which is exactly the same as $attTree_{\mathcal{E}_1}^-(A_{10})_1$.
 The reason is that only $A_{14}^+$ can serve as a child node of $A_{10}^-$ since both $A_{12} \notin \mathcal{E}_2$ and $A_{13} \notin \mathcal{E}_2$.
\end{example}

 \begin{figure}[t]
 \centering
 \includegraphics[width=0.9\textwidth]{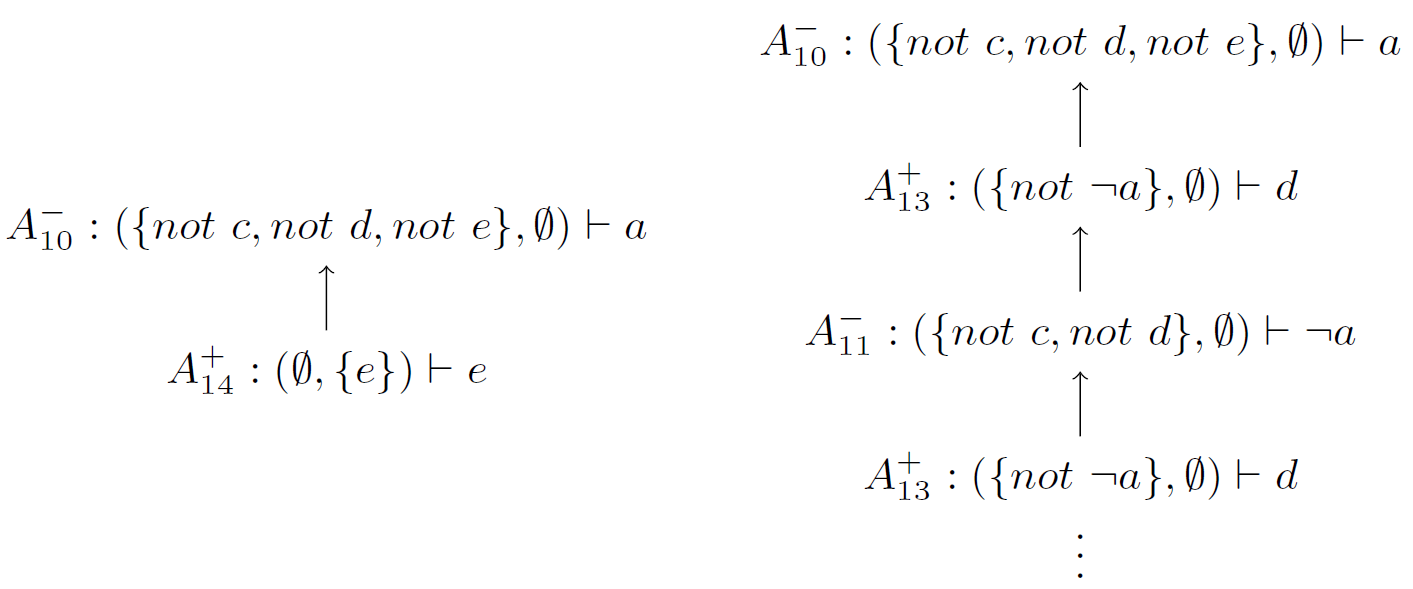}
  \caption{The two negative Attack Trees $attTree_{\mathcal{E}_1}^-(A_{10})_1$ (left) and $attTree_{\mathcal{E}_1}^-(A_{10})_2$ (right) 
  of $A_{10}$ with respect to $\mathcal{E}_1$, as described in Example~\ref{ex:p1_attackTrees}.
  The left Attack Tree is also the unique negative Attack Tree $attTree_{\mathcal{E}_2}^-(A_{10})$ of $A_{10}$ with respect to $\mathcal{E}_2$.}
  \label{fig:p1_attackTrees_a6}
\end{figure}

Figure~\ref{fig:p1_attackTrees_a6} illustrates that an argument might have more than one Attack Tree, as well as that Attack Trees can be infinite.
Figure~\ref{fig:p1_attackTrees_a5} depicts another negative Attack Tree, illustrating the case where a node labelled \textquotesingle$+$\textquotesingle\ has more than one child node.
Note that every argument in an ABA framework has at least one Attack Tree.
 However, an Attack Tree may solely consist of the root, for example the unique positive Attack Tree $attTree_{\mathcal{E}_1}^+(A_{14})$ of $A_{14}$
 with respect to the stable extension $\mathcal{E}_1$ consists of only one node, namely the root node $A_{14}^+$ as this argument has no attackers.
 
\begin{figure}[t]
 \centering
 \includegraphics[width=\textwidth]{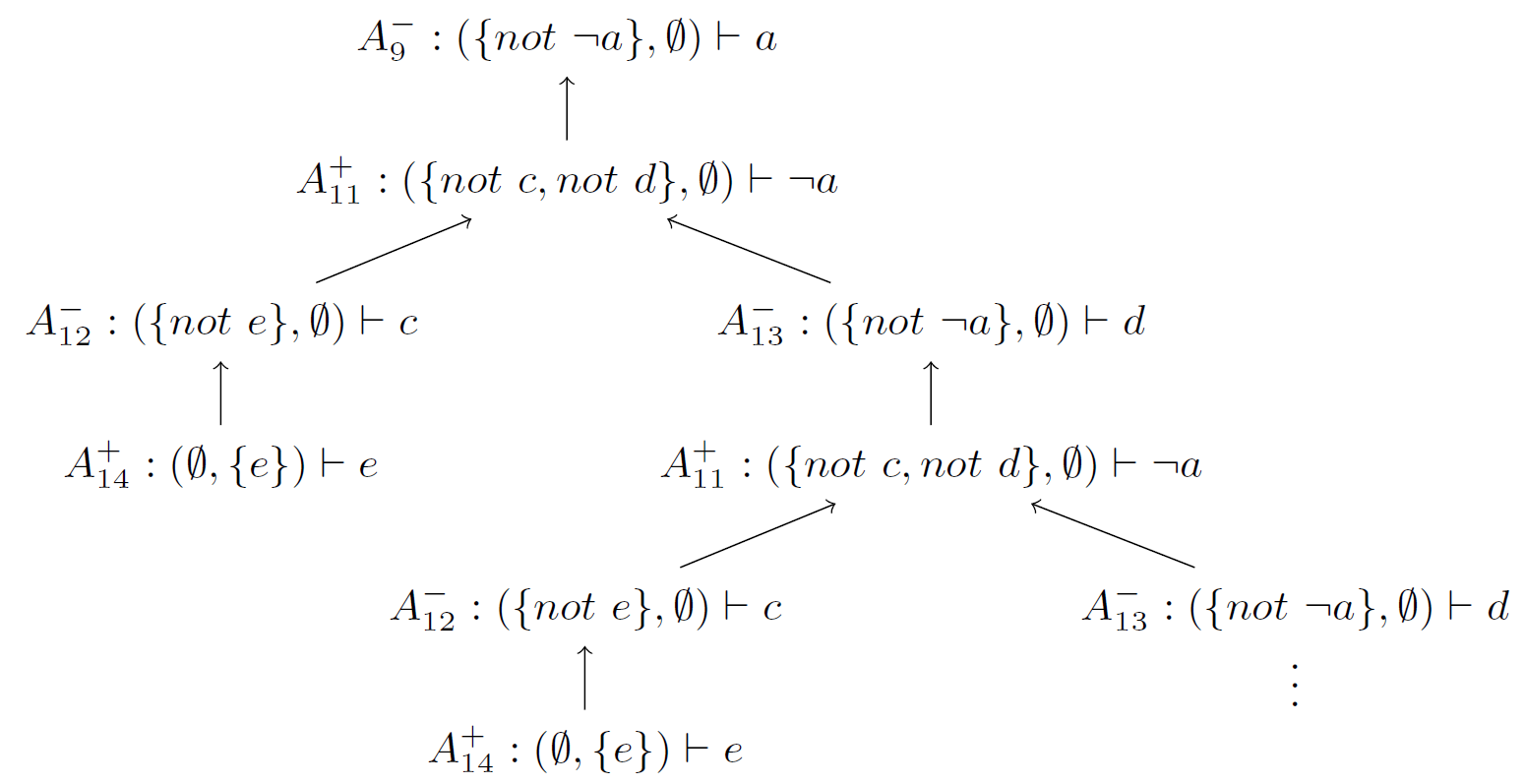}
 \caption{The unique negative Attack Tree $attTree_{\mathcal{E}_2}^-(A_9)$ of $A_9$ with respect to the stable extension
 $\mathcal{E}_2$ in $ABA_{\mathcal{P}_1}$ (see Examples~\ref{ex:p1} and \ref{ex:p1_as_stable}).}
 \label{fig:p1_attackTrees_a5}
\end{figure}

From the definition of Attack Trees it follows that the Attack Trees of an argument are either all positive or all negative.
\begin{lemma}
 \label{lem:pos_vs_neg_attTrees}
 Let $\mathcal{P}$ be a logic program and
 let $X$ be a set of arguments in $ABA_{\mathcal{P}}$.
 \begin{enumerate}
  \item If $A \in X$ then all Attack Trees of $A$ with respect to $X$ are positive Attack Trees $attTree_{X}^+(A)$.
  \item If $A \notin X$ then all Attack Trees of $A$ with respect to $X$ are negative Attack Trees $attTree_{X}^-(A)$.
 \end{enumerate}
\end{lemma}

\begin{proof}
 This follows directly from Definition~\ref{def:attackTree} and Notation~\ref{not:attTree_posneg}. \hfill  
\end{proof}

Intuitively, an Attack Tree of an argument with respect to a set of arguments explains why the argument is or is not in the set by showing
either that the argument is defended by the set, i.e. the set attacks all attackers of the argument, or that the argument is attacked by the sets
and cannot defend itself against it.
The first case explains why the argument is part of the set, whereas the second one justifies that the argument is not part of the set.

%%%%%%%%%%%%%%%%%%%%%%%%%%%%%%%%%%%%%%%%%%%%%%%%%%%%%%%%%%%%%%%%%%%%%%%%%%%%%%%%%%%%%%%%%%%%%%%%%%%%%%%%%%%%%%%%%%%%%%%%%%%%%%%%%%%%%%%%%%%%%%%%%%%%%%%%%%%%%%%%%%
\subsection{Attack Trees with respect to Stable Extensions}
\label{sec:attackTrees_stable}

For justification purposes we construct Attack Trees with respect to stable extensions rather than an arbitrary set of arguments.
This enables us to later extract a justification of a literal with respect to an answer set from an
Attack Tree constructed with respect to the corresponding stable extension.
In this section we show some properties of Attack Trees when constructed with respect to a stable extension, which
hold for both positive and negative Attack Trees.

One of these characteristics is that we can deduce whether or not an argument held by a node in an Attack Tree constructed with respect to a stable extension
is contained in this stable extension:
all arguments labelled \textquotesingle$+$\textquotesingle\ in the Attack Tree
are contained in the stable extension,
whereas all arguments labelled  \textquotesingle$-$\textquotesingle\ are not in the stable extension.

\begin{lemma}
 \label{lem:nodes_in_tree}
 Let $\mathcal{P}$ be a logic program and
 let $\mathcal{E}$ be a stable extension of $ABA_{\mathcal{P}}$.
 Let $\Upsilon = attTree_{\mathcal{E}}(A)$ be an Attack Tree of some argument $A$ in $ABA_{\mathcal{P}}$ with respect to $\mathcal{E}$. Then:
 \begin{enumerate}
  \item For each node $A_i^+$ in $\Upsilon$: $A_i \in \mathcal{E}$.
  \item For each node $A_i^-$ in $\Upsilon$: $A_i \notin \mathcal{E}$.
 \end{enumerate}
\end{lemma}

\begin{proof}
\begin{enumerate}
 \item $A_i^+$ is either the root node, then by definition $A_i \in \mathcal{E}$, or it is the only child node of some $A_N^-$, meaning that by definition $A_i \in \mathcal{E}$.
 \item $A_i^-$ is either the root node, then by definition $A_i \notin \mathcal{E}$, or $A_i^-$ is a child node of some $A_N^+$, and $A_i$ attacks $A_N$.
 From part 1 we know that $A_N \in \mathcal{E}$, hence $A_i \notin \mathcal{E}$ because $\mathcal{E}$ does not attack itself.
\end{enumerate}\hfill
\end{proof}

Another interesting characteristic of an Attack Tree constructed with respect to a stable extension is that
all nodes holding arguments labelled \textquotesingle$-$\textquotesingle\ have exactly one child node, rather than none. 
Furthermore, all leaf nodes hold arguments labelled \textquotesingle$+$\textquotesingle.

 \begin{lemma}
  \label{lem:root_leaves}
  Let $\mathcal{P}$ be a logic program and
  let $\mathcal{E}$ be a stable extension of $ABA_{\mathcal{P}}$.
  Let $\Upsilon = attTree_{\mathcal{E}}(A)$ be an Attack Tree of some argument $A$ in $ABA_{\mathcal{P}}$ with respect to $\mathcal{E}$ . Then:
  \begin{enumerate}
   \item Every node $A_N^-$ in $\Upsilon$ has exactly one child node.
   \item All leaf nodes in $\Upsilon$ hold arguments labelled \textquotesingle$+$\textquotesingle.
  \end{enumerate}
 \end{lemma}

 \begin{proof}
 \begin{enumerate}
  \item By condition 4 in Definition~\ref{def:attackTree}, any node $A_N^-$ in an Attack Tree has either no or exactly one child node.
  By  Lemma~\ref{lem:nodes_in_tree} $A_N \notin \mathcal{E}$.
  Assume that $A_N^-$ has no child node. Then $A_N$ is not attacked in  $ABA_{\mathcal{P}}$.
  But by definition of stable extension
  all arguments not contained in a stable extension are attacked by the stable extension. Contradiction.
  \item This follows directly from part 1 as nodes holding an argument labelled \textquotesingle$-$\textquotesingle\ always have a child node and thus cannot be a leaf node.
 \end{enumerate} \hfill
 \end{proof}
Note that infinite branches of Attack Trees do not have leaf nodes, in which case the second part of Lemma~\ref{lem:root_leaves} is trivially satisfied.

Lemma~\ref{lem:root_leaves} highlights how an Attack Tree justifies an argument $A$ with respect to a stable extension.
If the argument $A$ is part of the stable extension, the Attack Tree shows that the reason is that $A$ is defended by the stable extension.
This means that any attackers of $A$ are counter-attacked by an argument in the stable extension, defending $A$ against the attacker,
and even if the defending argument is further attacked, there will be another argument in the stable extension defending this defender,
until eventually the defending arguments from the stable extension are not further attacked, forming the leaf nodes of the Attack Tree.
Thus $A$ is defended by the stable extension and consequently belongs to it.
If an argument $A$ is not part of the stable extension, the leaf nodes of the Attack Tree again hold arguments from the stable extension,
but this time these leaf nodes defend the argument attacking $A$, 
meaning that this attacker is contained in the stable extension.
Thus, $A$ is attacked by the stable extension and consequently $A$ is not part of the stable extension.

Lemma~\ref{lem:root_leaves} also emphasizes the idea that to justify an argument which is not in the stable extension,
it is enough to show that one of its attackers is contained in the stable extension, even if there might be more than one such attacker.
This follows the general proof concept that something can be disproven by giving one counter-example.
So an Attack Tree disproves that the argument held by the root node is in the stable extension by showing one way
in which the argument is attacked by the stable extension.

From these considerations is follows directly that the subtree of any negative Attack Tree obtained by removing the root node is a positive Attack Tree of the argument
attacking the root node.

\begin{lemma}
 \label{lem:subtrees}
 Let $\mathcal{P}$ be a logic program and let $\mathcal{E}$ be a stable extension of $ABA_{\mathcal{P}}$.
 Let $\Upsilon = attTree^-_{\mathcal{E}}(A)$ be an Attack Tree of some argument $A \notin \mathcal{E}$ and let $A_i^+$ be the (only) child node of the root node $A^-$ in $attTree^-_{\mathcal{E}}(A)$.
 Let $\Upsilon'$ be the subtree of $\Upsilon$ with root node $A_i^+$ obtained from $\Upsilon$ by removing its root node $A^-$.
 Then $\Upsilon'$ is a positive Attack Tree of $A_i$.
\end{lemma}

\begin{proof}
 This follows directly from Definition~\ref{def:attackTree} and Notation~\ref{not:attTree_posneg}.
\end{proof}

This observation will be useful when comparing Attack Trees to abstract dispute trees in the following section.
Example~\ref{ex:doctor_attackTree} demonstrates how an Attack Tree can be used to explain why a literal is or is not contained in an answer set in terms of an argument for this literal.

\begin{example}
 \label{ex:doctor_attackTree}
 Consider Dr. Smith, his patient Peter, and the decision support system introduced in Section~\ref{sec:background_example}.
 In order to explain to Dr. Smith why $laserSurgery$ is not a suggested treatment of the decision support system, an Attack Tree for 
 an argument with conclusion $laserSurgery$ with respect to the corresponding stable extension of the answer set $S_{doctor}$ can be constructed.
 Figure~\ref{fig:doctor_attackTree_surgery} displays such an Attack Tree, which expresses that 
 Peter should not have laser surgery as the decision to use laser surgery is based on the assumption that the patient is not tight on money;
 however there is evidence that Peter is tight on money as he is known to be a student and there is no evidence against the assumption that his parents are not rich.
 Note that this is not the only Attack Tree for $A_1$ and therefore not the only possible explanation why Peter should not have laser surgery:
 a second Attack Tree can be constructed using an argument with conclusion $correctiveLens$ as an attacker of $A_1$.
 
 On the other hand, Dr. Smith might want to know why the treatment recommended by the decision support system is $intraocularLens$.
 The respective Attack Tree is illustrated in Figure~\ref{fig:doctor_attackTree_intraocular}.
 It expresses that Peter should get intraocular lenses because for every possible evidence against intraocular lenses ($A_1$, $A_4$, $A_6$) there is counter-evidence ($A_2$, $A_5$, and $A_7$ respectively):
 for example, receiving intraocular lenses is based on the assumption that it has not been decided that the patient should have glasses.
 Even though there is some evidence that Peter could have glasses, this evidence is based on the assumption that he does not care about the practicality of his treatment.
 However, it is known that Peter cares about practicality since he likes to do sports.
\end{example}

\begin{figure}[t]
\centering
\includegraphics[width=0.9\textwidth]{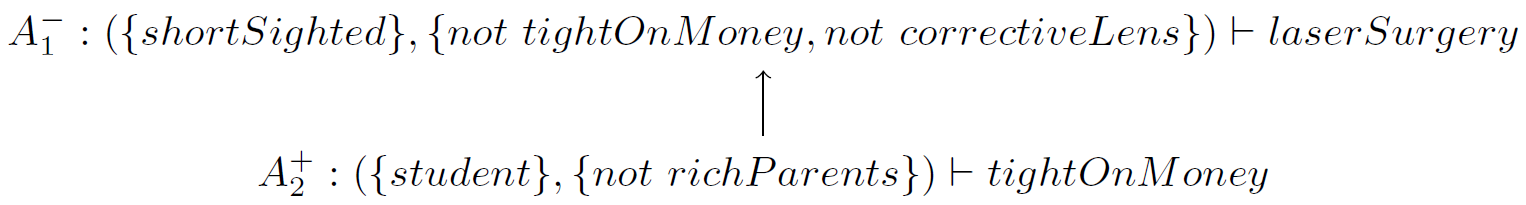}
 \caption{A negative Attack Tree of the argument $A_1$ with conclusion $laserSurgery$ with respect to the corresponding stable extension of the answer set $S_{doctor}$ of the
 logic program $\mathcal{P}_{doctor}$ (see Example~\ref{ex:doctor}), explaining why Peter should not have laser surgery as treatment of his shortsightedness.}
 \label{fig:doctor_attackTree_surgery}
\end{figure}

\begin{figure}[t]
\centering
\includegraphics[width=\textwidth]{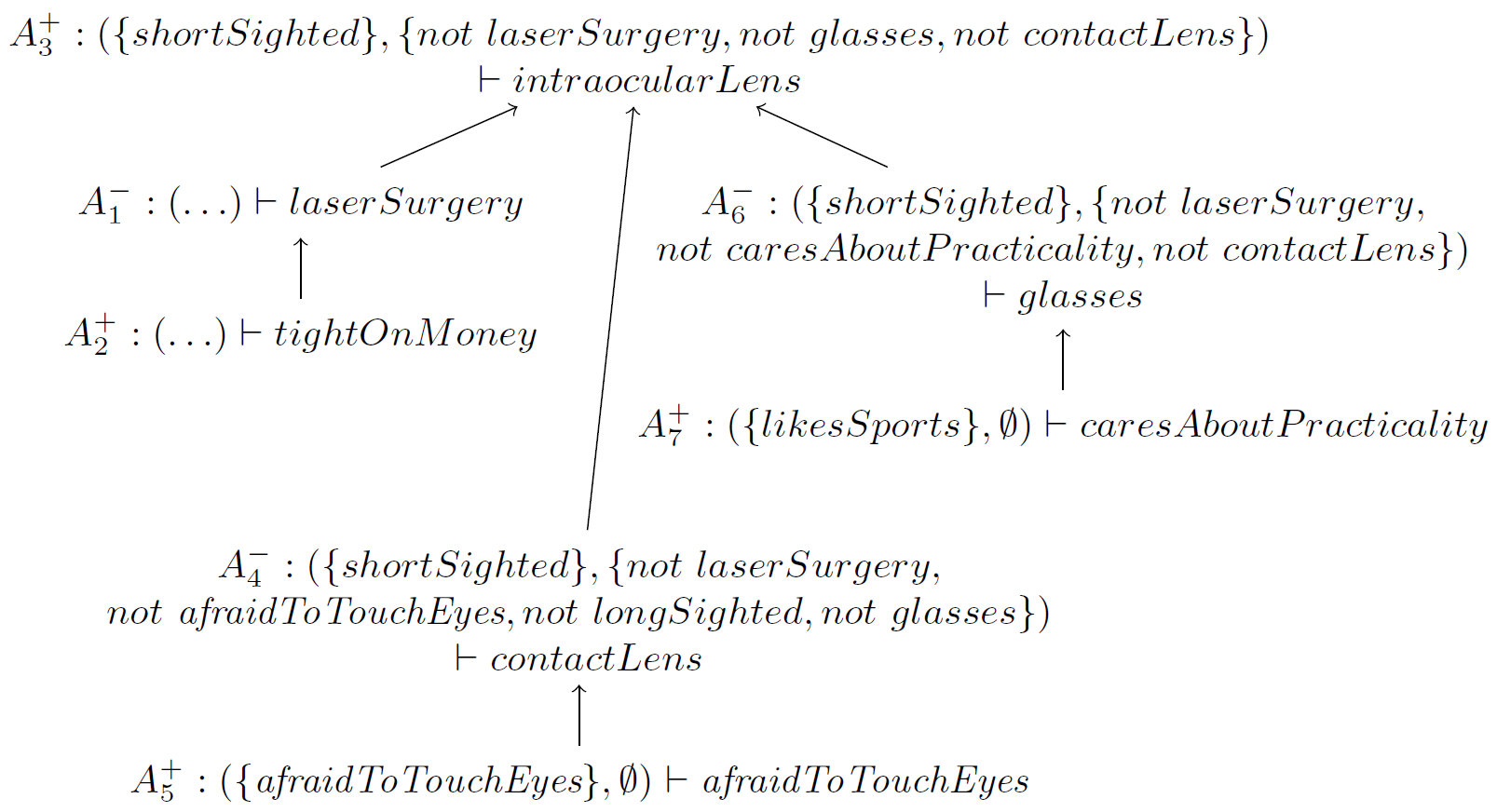}
 \caption{A positive Attack Tree of the argument $A_3$ with conclusion $intraocularLens$ with respect to the corresponding stable extension of the answer set $S_{doctor}$ of the
 logic program $\mathcal{P}_{doctor}$ (see Example~\ref{ex:doctor}), explaining why Peter should get intraocular lenses as treatment of his shortsightedness.
 The nodes holding $A_1^-$ and $A_2^+$ are abbreviated as they are the same as in Figure~\ref{fig:doctor_attackTree_surgery}.}
 \label{fig:doctor_attackTree_intraocular}
\end{figure}

%%%%%%%%%%%%%%%%%%%%%%%%%%%%%%%%%%%%%%%%%%%%%%%%%%%%%%%%%%%%%%%%%%%%%%%%%%%%%%%%%%%%%%%%%%%%%%%%%%%%%%%%%%%%%%%%%%%%%%%%%%%%%%%%%%%%%%%%%%%%%%%%%%%%%%%%%%%%%%%%%%
\subsection{Relationship between Attack Trees and Abstract Dispute Tress}
\label{sec:attackTrees_disputeTrees}

In order to further characterize Attack Trees, we prove that Attack Trees constructed with respect to stable extensions
are special cases of abstract dispute trees \cite{abaDialectical}.
Using this correspondence, we show that Attack Trees provide explanations of an argument in terms of an admissible fragment of the stable extension.
This result is then extended, proving that given a literal $k$ and an answer set,
an Attack Tree of an argument with conclusion $k$ with respect to the corresponding stable extension provides a justification in terms of an admissible
fragment of the answer set.
We first define a translation of the nodes holding arguments labelled \textquotesingle$+$\textquotesingle\ and \textquotesingle$-$\textquotesingle\ in Attack Trees into 
the status of proponent and opponent nodes in abstract dispute trees.

\begin{definition}[Translated Abstract Dispute Tree]
\label{def:translated_dispTree}
 Let $\mathcal{P}$ be a logic program, $X$ a set of arguments in $ABA_{\mathcal{P}}$, and 
 $attTree_{X}(A)$ an Attack Tree of some argument $A$ in $ABA_{\mathcal{P}}$ with respect to $X$.
 The \emph{translated abstract dispute tree} $\mathcal{T}_X(A)$ is obtained form $attTree_{X}(A)$ by assigning the status of proponent to all nodes holding an argument labelled \textquotesingle$+$\textquotesingle,
 the status of opponent to all nodes holding an argument labelled \textquotesingle$-$\textquotesingle, 
 and dropping the labels \textquotesingle$+$\textquotesingle\ and \textquotesingle$-$\textquotesingle\ of all arguments in the tree.
\end{definition}

If Attack Trees are constructed with respect to a stable extension, they correspond to abstract dispute trees in the following way:
\begin{lemma}
\label{lem:attackTree_disputeTree}
 Let $\mathcal{P}$ be a logic program and $\mathcal{E}$ a stable extension of $ABA_{\mathcal{P}}$.
 Let $attTree_{\mathcal{E}}(A)$ be an Attack Tree of some argument $A$ in $ABA_{\mathcal{P}}$ with respect to $\mathcal{E}$ and let $\mathcal{T}_{\mathcal{E}}(A)$ be the translated abstract dispute tree.
 Then:
 \begin{enumerate}
  \item If $A \in \mathcal{E}$, then $\mathcal{T}_{\mathcal{E}}(A)$ is an abstract dispute tree for $A$.
  \item If $A \notin \mathcal{E}$, then the subtree of $\mathcal{T}_{\mathcal{E}}(A)$ with root node $A_i$,
  where $A_i^+$ is the only child of the root $A^-$ in $attTree_{\mathcal{E}}(A)$, is an abstract dispute tree for $A_i$.
 \end{enumerate}
\end{lemma}

\begin{proof}
 This follows directly from the definition of abstract dispute trees and Lemma~\ref{lem:root_leaves}. \hfill
\end{proof}

Note that the converse of Lemma~\ref{lem:attackTree_disputeTree}.1 does not hold, i.e. it is not the case that
every abstract dispute tree for an argument $A$ corresponds to an Attack Tree $attTree_{\mathcal{E}}(A)$.
Example~\ref{ex:attTree_dispTree} illustrates Lemma~\ref{lem:attackTree_disputeTree} as well as that its converse does not hold.

\begin{example}
 \label{ex:attTree_dispTree}
 Let $\mathcal{P}_2$ be the following logic program:
  \begin{align*}
  a &\leftarrow not~a, not~b\\
  b &\leftarrow not~a, not~c\\
  c &\leftarrow not~ b
 \end{align*}
 Six arguments can be constructed in the translated ABA framework $ABA_{\mathcal{P}_2}$:
 \begin{align*}
 &A_1: (\{not~ a\}, \emptyset) \vdash not~ a &&A_4: (\{not~ a, not~ b\},\emptyset) \vdash a\\
 &A_2: (\{not~ b\}, \emptyset) \vdash not~ b &&A_5: (\{not~ a, not~c\}, \emptyset) \vdash b\\
 &A_3: (\{not~ c\}, \emptyset) \vdash not~ c &&A_6: (\{not~ b\}, \emptyset) \vdash c
 \end{align*}
 The only stable extension of $ABA_{\mathcal{P}_2}$ is $\mathcal{E} = \{A_1, A_3, A_5\}$.
 Figure~\ref{fig:attTree_dispTree1} illustrates the unique negative Attack Tree $attTree_{\mathcal{E}}^-(A_4)$ of $A_4$ with respect to $\mathcal{E}$.
 Constructing the translated abstract dispute tree of $attTree_{\mathcal{E}}^-(A_4)$ 
 results in the tree shown in Figure~\ref{fig:attTree_dispTree2}.
 As stated in Lemma~\ref{lem:attackTree_disputeTree}.2 deleting the opponent root node of the translated abstract dispute tree $\mathcal{T}_{\mathcal{E}}(A_4)$
 yields an abstract dispute tree for $A_5$.
 Figure~\ref{fig:dispTree_no_attTree} gives an example of an abstract dispute tree which does not correspond to an Attack Tree,
 showing that the converse of Lemma~\ref{lem:attackTree_disputeTree} does not hold.
 The abstract dispute tree for $A_6$ starts with a proponent node, which corresponds to the label
 \textquotesingle$+$\textquotesingle\ in an Attack Tree.
 However,any Attack Tree of $A_6$ is negative since $A_6 \notin \mathcal{E}$, 
 so the root node is always $A_6^-$.
 Thus, there is no Attack Tree which corresponds to the abstract dispute tree for $A_6$.
\end{example}

\begin{figure}[t]
 \centering
\includegraphics[width=0.9\textwidth]{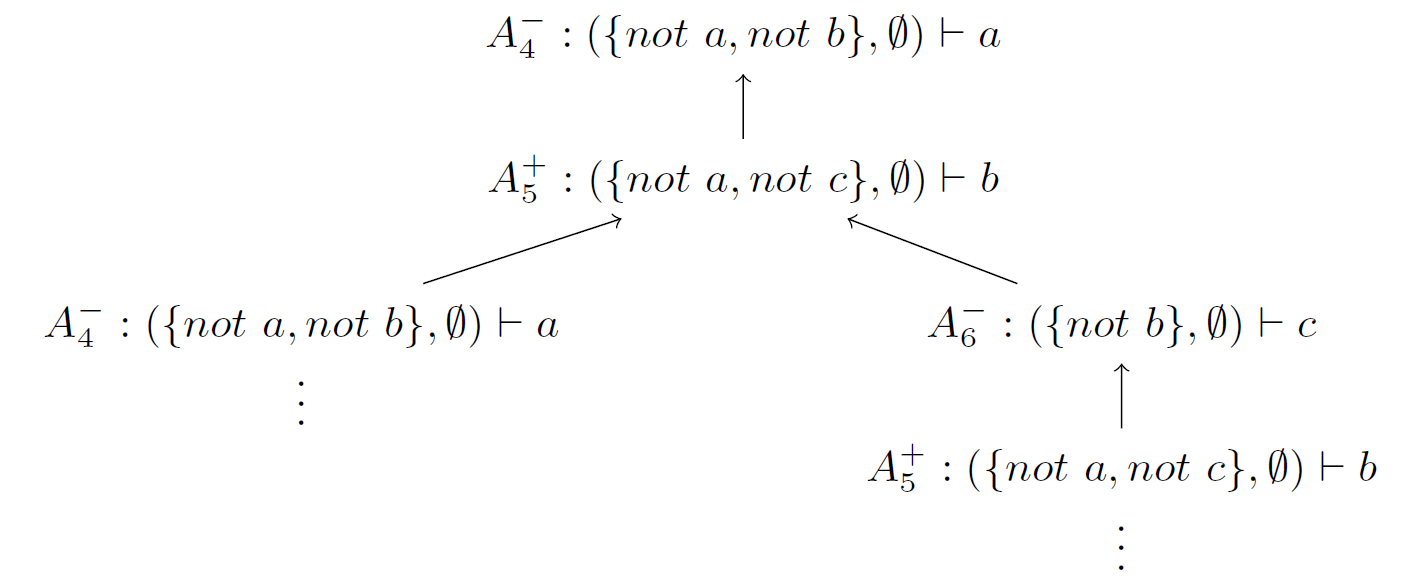}
 \caption{The unique negative Attack Tree $attTree_{\mathcal{E}}^-(A_4)$ of $A_4$ with respect to the stable extension $\mathcal{E}$ of $ABA_{\mathcal{P}_2}$ (see Example~\ref{ex:attTree_dispTree}).}
 \label{fig:attTree_dispTree1}
\end{figure}

\begin{figure}[t]
 \centering
 \includegraphics[width=\textwidth]{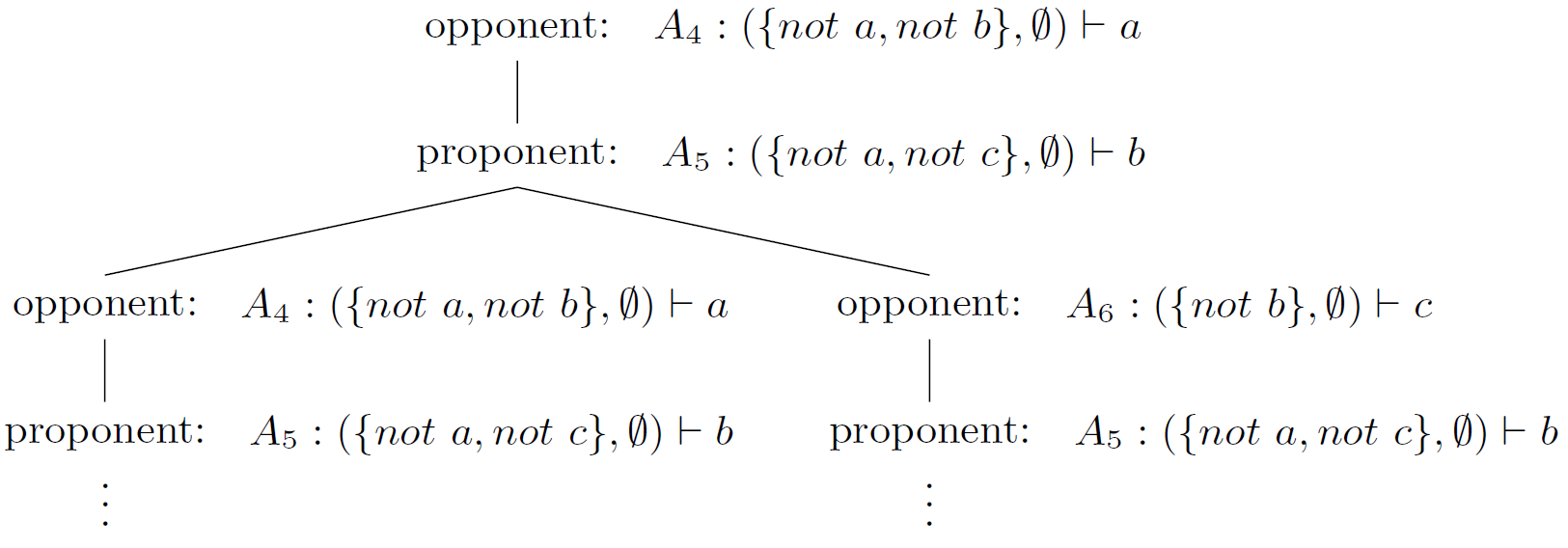}
 \caption{The translated abstract dispute tree $\mathcal{T}_{\mathcal{E}}(A_4)$ of $attTree_{\mathcal{E}}^-(A_4)$ (see Example~\ref{ex:attTree_dispTree} and Figure~\ref{fig:attTree_dispTree1}).
 As the root of $\mathcal{T}_{\mathcal{E}}(A_4)$ is an opponent node,
 it is not an abstract dispute tree. However, the subtree with root node $A_5$ is an abstract dispute tree for the argument $A_5$,
 as stated in Lemma~\ref{lem:attackTree_disputeTree}.}
 \label{fig:attTree_dispTree2}
\end{figure}

\begin{figure}[th]
 \centering
\includegraphics[width=0.5\textwidth]{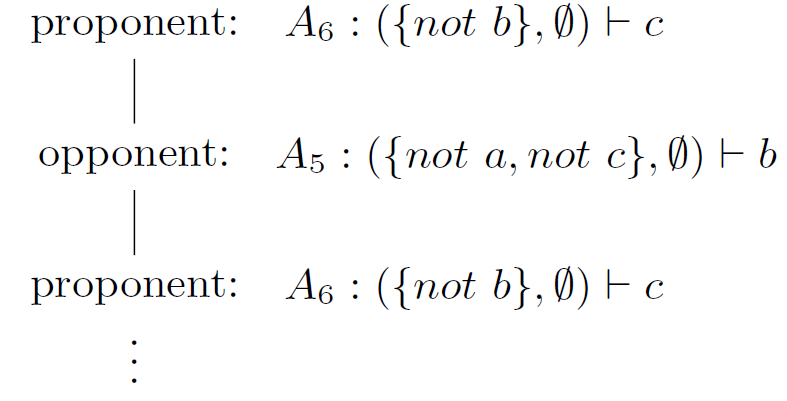}
 \caption{An abstract dispute tree for $A_6$ in $ABA_{\mathcal{P}_2}$ (see Example~\ref{ex:attTree_dispTree}).}
 \label{fig:dispTree_no_attTree}
\end{figure} 

Using the correspondence with abstract dispute trees, we can further characterize Attack Trees constructed with respect to a stable extension
as representing admissible fragments of this stable extension.
Starting with positive Attack Trees, we show that translated abstract dispute trees of positive Attack Trees with respect to a stable extension are admissible.

 \begin{lemma}
\label{lem:admissibleTree}
 Let $\mathcal{P}$ be a logic program, $\mathcal{E}$ a stable extension of $ABA_{\mathcal{P}}$, and $A$ some argument in $\mathcal{E}$.
 For every positive Attack Tree $attTree_{\mathcal{E}}^+(A)$ of $A$ with respect to $\mathcal{E}$, $\mathcal{T}_{\mathcal{E}}(A)$ is an admissible abstract dispute tree.
\end{lemma}

\begin{proof}
 According to Lemma~\ref{lem:nodes_in_tree}, for each $A_i^+$ in $attTree_{\mathcal{E}}^+(A)$, $A_i \in \mathcal{E}$, and
 for each $A_j^-$ in $attTree_{\mathcal{E}}^+(A)$, $A_j \notin \mathcal{E}$.
 By definition of stable extension, for all arguments $B$ in $ABA_{\mathcal{P}}$ either
 $B \in \mathcal{E}$ or $B \notin \mathcal{E}$. Thus, $A_i \neq A_j$ for all $i, j$,
and therefore by Definition~\ref{def:translated_dispTree} no argument labels both a proponent and an opponent node in $\mathcal{T}_{\mathcal{E}}(A)$,
satisfying the condition for admissibility.
By Lemma~\ref{lem:attackTree_disputeTree}, $\mathcal{T}_{\mathcal{E}}(A)$ is an abstract dispute tree.
\hfill
\end{proof}

Since a positive Attack Tree constructed with respect to a stable extension corresponds to an admissible abstract dispute tree,
the set of all arguments labelled \textquotesingle$+$\textquotesingle\ in the Attack Tree forms an admissible extension,
in particular one that is a subset of this stable extension.

\begin{theorem}
 \label{lem:admissibleExt}
 Let $\mathcal{P}$ be a logic program, $\mathcal{E}$ a stable extension of $ABA_{\mathcal{P}}$,
 and $attTree_{\mathcal{E}}^+(A)$ a positive Attack Tree of some argument $A$ in $ABA_{\mathcal{P}}$.
 Then the set $Args$ of all arguments labelled \textquotesingle$+$\textquotesingle\ in $attTree_{\mathcal{E}}^+(A)$ is an admissible extension of $ABA_{\mathcal{P}}$
 and $Args \subseteq \mathcal{E}$.
\end{theorem}

\begin{proof}
Let $Args$ denote the set of all arguments labelled \textquotesingle$+$\textquotesingle\ in $attTree_{\mathcal{E}}^+(A)$.
Then $Args$ is the set of arguments held by proponent nodes in the translated abstract dispute tree $\mathcal{T}_{\mathcal{E}}(A)$ of $attTree_{\mathcal{E}}^+(A)$.
By Lemma~\ref{lem:admissibleTree}, $\mathcal{T}_{\mathcal{E}}(A)$ is an admissible abstract dispute tree.
By Theorem 3.2(i) in \cite{aba_abstract}, $Args$ is an admissible extension, and
by Lemma~\ref{lem:nodes_in_tree}, $Args \subseteq \mathcal{E}$.\hfill
\end{proof}

This result characterizes Attack Trees as a way of justifying an argument by means of an admissible fragment of the stable extension.
In other words, the Attack Tree does not use whole stable extension to explain that an argument is in the stable extension,
but only provides an admissible subset sufficient to show that it defends the argument in question.
Furthermore, we can express this result in logic programming terms:
Given a literal and an answer set, an Attack Tree of an argument for this literal constructed with respect to the corresponding stable extension
justifies the argument using an admissible fragment of the answer set.

\begin{theorem}
 \label{lem:admissibleScenario}
 Let $\mathcal{P}$ be a logic program, $S$ an answer set of $\mathcal{P}$, $k \in S_{NAF}$,
 and $\mathcal{E}$ the corresponding stable extension of $S$ in $ABA_{\mathcal{P}}$.
 Let $A \in \mathcal{E}$ be a corresponding argument of $k$, $attTree_{\mathcal{E}}^+(A)$ an Attack Tree of $A$,
 and $Asms = \{\alpha \; | \; \alpha \in AP, A_1^+: (AP,FP) \vdash k_1 \textit{ in } attTree_{\mathcal{E}}^+(A)\}$.
 Then
 \begin{enumerate}
  \item $\mathcal{P}\; \cup\; Asms$ is an admissible scenario of $\mathcal{P}$ in the sense of \cite{admissible_scenario};
  \item $\{k_1 \; | \; A_1^+: (AP,FP) \vdash k_1 \textit{ in } attTree_{\mathcal{E}}^+(A)\} \subseteq S_{NAF}$.
 \end{enumerate}
\end{theorem}
 
 \begin{proof}
  \begin{enumerate}
   \item By Theorem~\ref{lem:admissibleExt} and Theorem~2.2(ii) in \cite{aba_abstract}, $Asms$ is an admissible set of assumptions.
   Then by Theorem~4.5 in \cite{aba_lp}, $\mathcal{P}\; \cup\; Asms$ is an admissible scenario of $\mathcal{P}$ in the sense of \cite{admissible_scenario}.\footnote{Theorem~4.5 refers to
   \cite{admissible_scenario_journal} where admissible scenarios are defined for logic programs without classical negation.
   This result can be easily extended to the definition of admissible scenarios of logic programs with both classical negation and NAF as we are only concerned with consistent logic programs.\label{fn:repeat}}
   \item By Theorem~\ref{lem:admissibleExt} and Theorem~\ref{lem:corr_ext}.
  \end{enumerate}\hfill
 \end{proof}

This result enables us to construct a justification of a literal in an answer set from an Attack Tree (in Section~\ref{sec:ASjust})
using an admissible fragment of the answer set.
The following example illustrates the characteristics of positive Attack Trees and how they can be used for justifying an argument for a literal in
an answer set.

\begin{example}
 \label{ex:p1_posAttTree_admissible}
 Consider the logic program $\mathcal{P}_1$ and its answer set $S_1 = \{e,d,a\}$ with the corresponding stable extension
 $\mathcal{E}_1 = \{A_2,A_3,A_4,A_6,A_8, A_9,A_{13},A_{14}\}$ (see Examples \ref{ex:p1} and \ref{ex:p1_as_stable}).
 To justify that $not~ c \in S_{1_{NAF}}$, we can construct an Attack Tree of an argument for $not~ c$, i.e. of $A_3$, with respect to $\mathcal{E}_1$.
 The resulting positive Attack Tree $attTree_{\mathcal{E}_1}^+(A_3)$ is depicted on the left of Figure~\ref{fig:p1_attackTree_a3}.
 Translating this Attack Tree into an abstract dispute tree as described in Definition~\ref{def:translated_dispTree}, yields
 the translated abstract dispute tree $\mathcal{T}_{\mathcal{E}_1}(A_3)$ illustrated on the right of Figure~\ref{fig:p1_attackTree_a3}.
 This abstract dispute tree is admissible as stated in Lemma~\ref{lem:admissibleTree}.
 The set arguments labelled \textquotesingle$+$\textquotesingle\ in $attTree_{\mathcal{E}_1}^+(A_3)$ is $\{A_3, A_{14}\} \subseteq \mathcal{E}_1$ 
 which is an admissible extension of $ABA_{\mathcal{P}_1}$ and
 the set of conclusions of these arguments is $\{not~ c, e\} \subseteq S_{1_{NAF}}$ as stated by Theorems \ref{lem:admissibleExt} and \ref{lem:admissibleScenario}.
 The Attack Tree $attTree_{\mathcal{E}_1}^+(A_3)$ explains that the literal $not~c$ is in the answer set $S_1$ because
 it is supported and defended by an admissible subset of $S_1$, namely by $\{not~ c, e\}$.
 In terms of literal the Attack Tree expresses that $not~c$ is ``attacked'' by the literal $c$, which is ``counter-attacked'' by $e$,
 thereby defending $not~c$.
\end{example}

\begin{figure}[th]
 \centering
\includegraphics[width=0.9\textwidth]{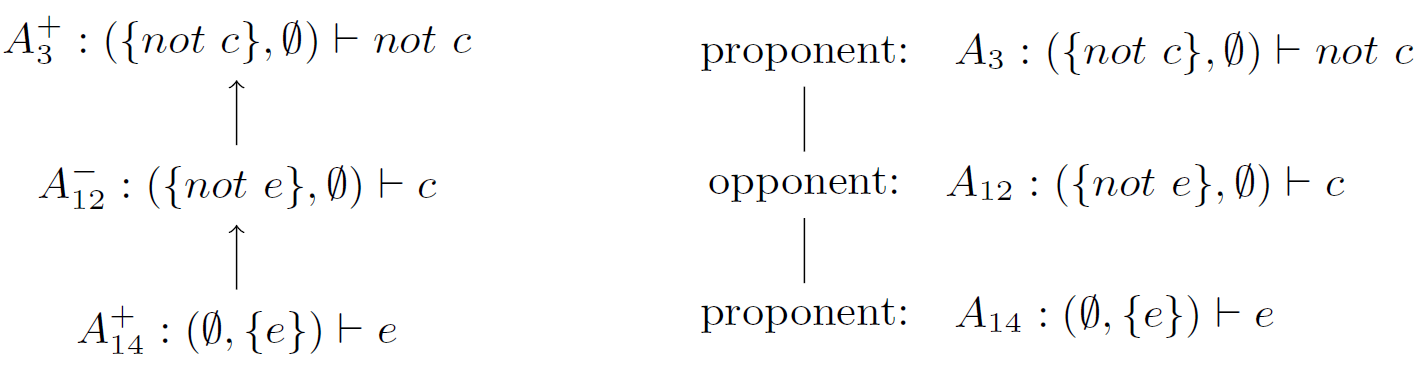}
 \caption{The positive Attack Tree $attTree_{\mathcal{E}_1}^+(A_3)$ of $A_3$ with respect to the corresponding stable extension
 $\mathcal{E}_1$ of $S_1$ (left) and the translated abstract dispute tree $\mathcal{T}_{\mathcal{E}_1}(A_3)$ 
 of $attTree_{\mathcal{E}_1}^+(A_3)$ (right) (see Example~\ref{ex:p1_posAttTree_admissible}).}
 \label{fig:p1_attackTree_a3}
\end{figure}

Similarly to positive Attack Trees, we can characterize the explanations given by negative Attack Trees
using the correspondence between the subtree of a negative Attack Tree and an abstract dispute tree:
Negative Attack Trees justify that an argument is not in a stable extension because it is attacked by an admissible fragment of this stable extension.
We first prove that when deleting the opponent root node of the translated abstract dispute tree of a negative Attack Tree constructed with respect
to a stable extension, the resulting abstract dispute tree is admissible.

\begin{lemma}
\label{lem:admissibleSubTree}
 Let $\mathcal{P}$ be a logic program, $\mathcal{E}$ a stable extension of $ABA_{\mathcal{P}}$, and $A$ some argument not contained in $\mathcal{E}$.
 For every negative Attack Tree $attTree_{\mathcal{E}}^-(A)$ of $A$ with respect to $\mathcal{E}$,
 the subtree of $\mathcal{T}_{\mathcal{E}}(A)$ with root node $A_i$,
 where $A_i^+$ is the only child of the root $A^-$ in $attTree_{\mathcal{E}}^-(A)$, is an admissible abstract dispute tree.
\end{lemma}

\begin{proof}
 By Lemma~\ref{lem:subtrees}, the subtree of $\Upsilon'$ of $attTree_{\mathcal{E}}^-(A)$ with root node $A_i$ is a positive Attack Tree of $A_i$.
 By Lemma~\ref{lem:admissibleTree}, $\Upsilon'$ is an admissible abstract dispute tree.
 Trivially, the subtree of $\mathcal{T}_{\mathcal{E}}(A)$ with root node $A_i$ coincides with the translated abstract dispute tree of $\Upsilon'$.
\hfill
\end{proof}

Knowing that the argument held by the root of a negative Attack Tree constructed with respect to a stable extension is attacked
by an admissible abstract dispute tree, 
we show that this Attack Tree justifies the root by showing that it is
attacked by an admissible extension of $ABA_{\mathcal{P}}$,
and in particular by an admissible extension which is a subset of the stable extension.

\begin{theorem}
 \label{lem:admissibleExtAttacks}
 Let $\mathcal{P}$ be a logic program, $\mathcal{E}$ a stable extension of $ABA_{\mathcal{P}}$, and $attTree_{\mathcal{E}}^-(A)$
 a negative Attack Tree of some argument $A$ in $ABA_{\mathcal{P}}$.
 Then the set $Args$ of all arguments labelled \textquotesingle$+$\textquotesingle\ in $attTree_{\mathcal{E}}^-(A)$ is an admissible extension of $ABA_{\mathcal{P}}$
 and $Args \subseteq \mathcal{E}$.
\end{theorem}

\begin{proof}
Let $Args$ denote the set of all arguments labelled \textquotesingle$+$\textquotesingle\ in $attTree_{\mathcal{E}}^-(A)$.
Then $Args$ is the set of arguments held by proponent nodes in the translated abstract dispute tree $\mathcal{T}_{\mathcal{E}}(A)$ of $attTree_{\mathcal{E}}^-(A)$.
By Lemma~\ref{lem:admissibleSubTree}, the subtree of $\mathcal{T}_{\mathcal{E}}(A)$ with root node $A_i$,
 where $A_i^+$ is the only child of the root $A^-$ in $attTree_{\mathcal{E}}^-(A)$, is an admissible abstract dispute tree.
By Theorem 3.2(i) in \cite{aba_abstract}, $Args$ is an admissible extension.
By Lemma~\ref{lem:nodes_in_tree}, $Args \subseteq \mathcal{E}$. \hfill
\end{proof}

It follows, that a negative Attack Tree justifies an argument for a literal which is not in the answer set in question
in terms of an admissible fragment of the answer set ``attacking'' the literal.

\begin{theorem}
 \label{lem:admissibleScenarioAttacks}
 Let $\mathcal{P}$ be a logic program, $S$ an answer set of $\mathcal{P}$, $k \notin S_{NAF}$,
 and $\mathcal{E}$ the corresponding stable extension of $S$ in $ABA_{\mathcal{P}}$.
 Let $A$ be some argument for $k$, $attTree_{\mathcal{E}}^-(A)$ an Attack Tree of $A$, and $Asms = \{\alpha \; | \; \alpha \in AP, A_1^+: (AP,FP) \vdash k_1 \textit{ in } attTree_{\mathcal{E}}^-(A)\}$.
 Then
 \begin{enumerate}
  \item $\mathcal{P}\; \cup\; Asms$ is an admissible scenario of $\mathcal{P}$ in the sense of \cite{admissible_scenario};
  \item $\{k_1 \; | \; A_1^+: (AP,FP) \vdash k_1 \textit{ in } attTree_{\mathcal{E}}^-(A)\} \subseteq S_{NAF}$.
 \end{enumerate}
\end{theorem}
 
 \begin{proof}
  \begin{enumerate}
   \item By Theorem~\ref{lem:admissibleExtAttacks} and Theorem~2.2(ii) in \cite{aba_abstract}, $Asms$ is an admissible set of assumptions.
   Then by Theorem~4.5 in \cite{aba_lp}, $\mathcal{P}\; \cup\; Asms$ is an admissible scenario of $\mathcal{P}$ in the sense of \cite{admissible_scenario}.\footref{fn:repeat}
   \item By Theorem~\ref{lem:admissibleExtAttacks} and Theorem~\ref{lem:corr_ext}. 
  \end{enumerate}\hfill
 \end{proof}

 This result provides the basis for the construction of a justification of a literal not contained in an answer set from an
 Attack Tree which provides a meaningful explanation in terms of an admissible subset of the answer set.
 
 \begin{example}
  \label{ex:p2_negativeAttTree_admissble}
  Consider the logic program $\mathcal{P}_2$ and its only answer set $S = \{b\}$ with the corresponding stable extension
  $\mathcal{E} = \{A_1, A_3, A_5\}$ (see Example~\ref{ex:attTree_dispTree}).
  To justify why $a \notin S$ we can construct an Attack Tree of an argument with conclusion $a$, i.e. of $A_4$,
  with respect to $\mathcal{E}$.
  The resulting negative Attack Tree $attTree_{\mathcal{E}}^-(A_4)$ is depicted in Figure~\ref{fig:attTree_dispTree1}
  and the translated abstract dispute tree $\mathcal{T}_{\mathcal{E}}(A_4)$ in Figure~\ref{fig:attTree_dispTree2}.
  When deleting the root opponent node $A_4$ of $\mathcal{T}_{\mathcal{E}}(A_4)$,
  the resulting abstract dispute tree is admissible as observed in Lemma~\ref{lem:admissibleSubTree}.
  Furthermore, the set of arguments labelled \textquotesingle$+$\textquotesingle\ in $attTree_{\mathcal{E}}^-(A_4)$
  is $\{A_5\}$, which is a subset of the corresponding stable extension $\mathcal{E}$ and an admissible extension of $ABA_{\mathcal{P}_2}$ (by Theorem~\ref{lem:admissibleExtAttacks}).
  Moreover, the set of conclusions of arguments in this admissible extension is $\{b\} \subseteq S$, which is an admissible scenario of $\mathcal{P}$ as
  stated in Theorem~\ref{lem:admissibleScenarioAttacks}.
  Therefore, the negative Attack Tree $attTree_{\mathcal{E}}^-(A_4)$ explains that the argument $A_4$ is not in the corresponding stable extension because it is 
  attacked by an admissible fragment of this stable extension, namely by $\{A_5\}$.
  Even though $A_4$ together with $A_6$ counter-attacks this attack, $A_5$ defends itself against this counter-attack.
  This explanation can also be interpreted in terms of literals: $a$ is not in the answer set $S$ because
  its derivation is ``attacked'' by a derivation of $b$, which is an admissible fragment of $S$.
  Even though the derivation of $a$ and the derivation of $c$ both ``counter-attack'' the derivation of $b$, attempting to defend $a$,
  the derivation of $b$ can attack both counter-attacks and thus the derivation of $b$ defends itself.
  Consequently, the attack of the derivation of $b$ on the derivation of $a$ ``succeeds'', which is the reason that $a$ is not part of the answer set.
 \end{example}

 In conclusion, Attack Trees provide a justification of an argument with respect to a stable extension
 in terms of an admissible subset of this stable extension.
 Due to the correspondence between answer sets and stable extensions, Attack Trees can also be used to justify a literal
 with respect to an answer set by constructing an Attack Tree of an argument for this literal with respect to the corresponding stable extension.
 The resulting Attack Tree justifies the argument for the literal in question using an admissible fragment of the answer set.
 If the literal in question is contained in the given answer set, the admissible fragment supports and defends a derivation of this literal.
 If the literal in question is not contained in the given answer set, the admissible fragment of the answer set ``attacks'' a derivation of this literal,
 in fact by Theorem~\ref{the:AsStable} the admissible fragment attacks all derivations of this literal.
 
 The only shortcoming of justifying literals with respect to an answer set in terms of Attack Trees is that they use argumentation-theoretic concepts for the explanation.
Thus, we now define a second type of justification which provides explanations in terms of literals and relations between them,
rather than in terms of arguments as used in Attack Trees.
 The new type of justification is constructed from Attack Trees by flattening the structure of arguments occurring in an Attack Tree
 as well as of the attack relation between these arguments.
 In addition to better fitting logic programming concepts,
 another advantage of the new justifications is that they are finite even if constructed from infinite Attack Trees.
 We first introduce a basic version of this new justification to illustrate the idea of flattening Attack Trees.
 Then, we define a more elaborate version using the same flattening approach but simultaneously labelling
 literals and their relations, yielding a more informative explanation.

%%%%%%%%%%%%%%%%%%%%%%%%%%%%%%%%%%%%%%%%%%%%%%%%%%%%%%%%%%%%%%%%%%%%%%%%%%%%%%%%%%%%%%%%%%%%%%%%%%%%%%%%%%%%%%%%%%%%%%%%%%%%%%%%%%%%%%%%%%%%%%%%%%%%%%%%%%%%%%%%%%
%%%%%%%%%%%%%%%%%%%%%%%%%%%%%%%%%%%%%%%%%%%%%%%%%%%%%%%%%%%%%%%%%%%%%%%%%%%%%%%%%%%%%%%%%%%%%%%%%%%%%%%%%%%%%%%%%%%%%%%%%%%%%%%%%%%%%%%%%%%%%%%%%%%%%%%%%%%%%%%%%%
\section{Basic ABA-Based Answer Set Justifications}
\label{sec:ASjust}

In this section we define the basic concepts for constructing justifications of a literal $k$ in terms of literals and their relations, based on
Attack Trees of arguments with conclusion $k$.
The idea is to extract the assumptions- and fact-premises of each argument in the Attack Tree to express a support-relation between each
of the premise-literals and the literal forming the conclusion of an argument.
Furthermore, the attacks between arguments in an Attack Tree are translated into attack-relations between the literals forming the conclusions of these arguments.
We first introduce some terminology to refer to the structure of an Attack Tree.

\begin{notation}
\label{not:attTree}
Let $\Upsilon$ be an Attack Tree and let $N$ be a node in $\Upsilon$.
$arg(N)$ denotes the argument held by node $N$.
If $arg(N)$ is $A: (AP,FP) \vdash k$, then $name(N) = A$, $conc(N) = k$, $AP(N) = AP$, $FP(N) = FP$,
and $label(N)$ is either \textquotesingle$+$\textquotesingle\ or \textquotesingle$-$\textquotesingle,
depending on the label of $A$ in $\Upsilon$. The set of all child nodes of $N$ in $\Upsilon$ is denoted $children(N)$.
\end{notation}

%%%%%%%%%%%%%%%%%%%%%%%%%%%%%%%%%%%%%%%%%%%%%%%%%%%%%%%%%%%%%%%%%%%%%%%%%%%%%%%%%%%%%%%%%%%%%%%%%%%%%%%%%%%%%%%%%%%%%%%%%%%%%%%%%%%%%%%%%%%%%%%%%%%%%%%%%%%%%%%%%%
\subsection{Basic Justifications}
\label{sec:ASjust_basicJust}

We now define how to express the structure of an Attack Tree as a set of relations between literals.
\begin{definition}[Basic Justification]
 Let $\mathcal{P}$ be a logic program and
 let $X$ be a set of arguments in $ABA_{\mathcal{P}}$.
 Let $A$ be an argument in $ABA_{\mathcal{P}}$ and $\Upsilon = attTree_{X}(A)$ an Attack Tree of $A$ with respect to $X$.
 The \emph{Basic Justification} of $A$ with respect to $\Upsilon$, denoted $justB_{\Upsilon}(A)$, is obtained as follows:\\
 $justB_{\Upsilon}(A) =
 \bigcup_{N \textit{ in } \Upsilon}\\
 \{ supp\mathunderscore rel(k, conc(N)) \; | \; k \in AP(N) \; \cup \; FP(N) \backslash \{conc(N)\}\} \; \cup \;\\
 \{ att\mathunderscore rel(conc(M), k) \; | \; M \in children(N), conc(M) = \overline{k}\}$
 \end{definition}
 
 \begin{example}
  \label{ex:p1_basic_justification}
  Consider the logic program $\mathcal{P}_1$ from Example~\ref{ex:p1} and the Attack Trees discussed in Example~\ref{ex:p1_attackTrees}.
   Since $\Upsilon_1 = attTree^+_{\mathcal{E}_1}(A_{14})$ comprises only the node $A_{14}^+$,
  the Basic Justification of $A_{14}$ with respect to $\Upsilon_1$ is $justB_{\Upsilon_1}(A_{14}) = \emptyset$.
  
  Now consider the negative Attack Tree $\Upsilon_2= attTree_{\mathcal{E}_2}^-(A_{10})$ of $A_{10}$ with respect to $\mathcal{E}_2$ depicted on the left of Figure~\ref{fig:p1_attackTrees_a6}.
  The Basic Justification of $A_{10}$ with respect to $\Upsilon_2$ is:
 \begin{align*}
 justB_{\Upsilon_2}(A_{10}) &= \{supp\mathunderscore rel(not~ c, a),\: supp\mathunderscore rel(not~ d, a),\: supp\mathunderscore rel(not~ e, a)\}\: \cup \\
 &\qquad \{att\mathunderscore rel(e, not~ e)\}\\
 &= \{supp\mathunderscore rel(not~ c, a),\: supp\mathunderscore rel(not~ d, a),\:  supp\mathunderscore rel(not~ e, a),\\ &\qquad att\mathunderscore rel(e, not~ e)\}
 \end{align*}
 
 The following Basic Justification is obtained from the negative Attack Tree $\Upsilon_3 = attTree_{\mathcal{E}_2}^-(A_9)$ of $A_9$
 with respect to the stable extension $\mathcal{E}_2$ (see Figure~\ref{fig:p1_attackTrees_a5}):
 \begin{align*}
  justB_{\Upsilon_3}(A_9) &= \{supp\mathunderscore rel(not~ \neg a, a),\:  att \mathunderscore rel(\neg a, not~ \neg a), \: 
  supp\mathunderscore rel(not~ c, \neg a),\\
  &\qquad supp\mathunderscore rel(not~d, \neg a),\: 
  att \mathunderscore rel(c, not~ c), \: att \mathunderscore rel(d, not~ d),\\
  &\qquad supp \mathunderscore rel(not~ e, c), \: att\mathunderscore rel(e, not~ e),\: 
  supp \mathunderscore rel(not~ \neg a, d)\}
 \end{align*}
 Note that even though $\Upsilon_3$ is an infinite Attack Tree, the Basic Justification of $A_9$ with respect to $\Upsilon_3$ is finite.
 In particular, when $A_{11}$ reoccurs in the Attack Tree as an attacker of $A_{13}$, 
 no new $att \mathunderscore rel$ or $supp \mathunderscore rel$ pairs are added to the Basic Justification:
 even though $A_{11}$ attacks $A_9$ with conclusion $a$ at its first occurrence and $A_{13}$ with conclusion $d$ at its second occurrence,
 no new $att \mathunderscore rel$ pair is added since the attacked assumption is in both cases $not~ \neg a$.
 \end{example}
 
 In Basic Justifications attacks between arguments are translated into ``attacks'' between literals, and supports of arguments into ``supports'' of literals.
 In other words, a Basic Justification is the flattened version of an Attack Tree.
 Even though it provides an explanation in terms of literals rather than arguments, it is not sufficient to justify
 a literal with respect to an answer set for two reasons, as explained below.

 Firstly, a Basic Justification does not contain the literal being justified, which is for example a problem when justifying a fact.
 When justifying a fact $k$, we construct an Attack Tree of the fact-argument for $k$, which
 consists of only the root node $A^+: (\emptyset, \{k\})\vdash k$, leading to an empty Basic Justification.
 An empty set is not meaningful, so it would be useful if the literal in question was contained in the justification. 
 Furthermore, a problem arises when trying to justify a literal for which no argument exists in the translated ABA framework,
 i.e. a literal which cannot be derived in any way from the logic program.
 For such a literal, which is trivially not part of any answer set, it is not possible to construct an Attack Tree
 as no argument for this literal exists in the translated ABA framework.
 Since a Basic Justification is constructed from an Attack Tree, there is no Basic Justification for such a literal.
 This is unsatisfying, so we would like to have some kind of justification, rather than failing.

 The second problem or shortcoming of a Basic Justification is that it only provides one reason why a literal is not in an answer set
 as it is constructed from a single negative Attack Tree, which provides one explanation how the root argument is attacked by the set of arguments in question.
 However, it is more meaningful to capture all different explanations of how a literal ``failed'' to be in the answer set in question.
 Thus, we want the justification of a literal not in the answer set to consist of all possible Basic Justifications of this literal.
 
 In order to overcome these two deficiencies, we introduce BABAS Justifications, which add the literal being justified to the Basic Justification set
 and provide a collection of all Basic Justifications for a literal which is not contained in an answer set.

 %%%%%%%%%%%%%%%%%%%%%%%%%%%%%%%%%%%%%%%%%%%%%%%%%%%%%%%%%%%%%%%%%%%%%%%%%%%%%%%%%%%%%%%%%%%%%%%%%%%%%%%%%%%%%%%%%%%%%%%%%%%%%%%%%%%%%%%%%%%%%%%%%%%%%%%%%%%%%%%%%%
\subsection{BABAS Justifications}
\label{sec:ASjust_babas}

 We now define the \emph{Basic ABA-Based Answer Set (BABAS) Justification} of a literal with respect to an answer set,
 which is based on the Basic Justifications of an argument with respect to an Attack Tree.
 If a literal $k$ is contained in an answer set, its BABAS Justification is constructed from one Basic Justification of one of the corresponding arguments of $k$.
 This is inspired by the result in Theorem~\ref{the:AsStable} that a literal $k$ is part of an answer set if and only if there exists some argument
 with conclusion $k$ in the corresponding stable extension.
 Conversely, if $k$ is not contained in an answer set, its BABAS Justification is constructed from all Basic Justifications of all arguments with conclusion $k$,
 expressing all reasons why $k$ is not part of this answer set.
 Again, the choice to consider all arguments with conclusion $k$ is based on Theorem~\ref{the:AsStable}, stating that a literal $k$ is not part of an answer set
 if and only if all arguments with conclusion $k$ are not contained in the corresponding stable extension.
 \begin{definition}[Basic ABA-Based Answer Set Justification]
 \label{def:babas}
 Let $\mathcal{P}$ be a logic program and
 $S$ an answer set of $\mathcal{P}$.
 Let $\mathcal{E}$ be the corresponding stable extension of $S$ in $ABA_{\mathcal{P}}$.
 \begin{enumerate}
  \item Let $k \in S_{\textit{NAF}}$, $A \in \mathcal{E}$ a corresponding argument of $k$, and $\Upsilon = attTree_{\mathcal{E}}^+(A)$ some positive Attack Tree of $A$ with respect to $\mathcal{E}$.
 A \emph{Positive BABAS Justification} of $k$ with respect to $S$ is:\\ $justB_{S}^+(k) = \{k\} \; \cup \; justB_{\Upsilon}(A)$.
 \item Let $k \notin S_{\textit{NAF}}$, $A_1,\ldots, A_n$ ($n \geq 0$) all arguments with conclusion $k$ in $ABA_{\mathcal{P}}$,
 and $\Upsilon_{11},\ldots, \Upsilon_{1m_1}, \ldots, \Upsilon_{n1}, \ldots, \Upsilon_{nm_n}$ ($m_1,\ldots,m_n \geq 0$) all negative Attack Trees of $A_1,\ldots, A_n$ with respect to $\mathcal{E}$.
 \begin{enumerate}[(a)]
    \item If $n = 0$, then the \emph{Negative BABAS Justification} of $k$ with respect to $S$ is:\\
    $justB_{S}^-(k) = \emptyset$
    \item If $n > 0$, then the \emph{Negative BABAS Justification} of $k$ with respect to $S$ is:\\
 $justB_{S}^-(k) = \{\{k\} \; \cup \; justB_{\Upsilon_{11}}(A_1),\ldots,\;\;\{k\} \; \cup \; justB_{\Upsilon_{1m_1}}(A_1), \ldots,\; \\ \{k\} \; \cup \; justB_{\Upsilon_{nm_n}}(A_n)\}$.
 \end{enumerate}
 \end{enumerate}
\end{definition}

Note that there can be more than one Positive BABAS Justification of a literal contained in an answer set,
but only one Negative BABAS Justification of a literal not contained in an answer set.
Note also that the Positive BABAS Justification is a set of $supp\mathunderscore rel$ and $att\mathunderscore rel$ pairs (plus the literal which is justified),
whereas the Negative BABAS Justification is a set of sets containing these pairs (where each set also contains the literal which is justified).

A BABAS Justification can be represented as a graph, where all literals occurring in a $supp\mathunderscore rel$ or $att \mathunderscore rel$ pair form nodes,
and the $supp\mathunderscore rel$ and $att \mathunderscore rel$ relations are edges between these nodes.
For Negative BABAS Justifications, a separate graph for each set in the justification is given.
In contrast, Positive BABAS Justifications are illustrated as a single graph.

\begin{example}
\label{ex:p1_just_e}
 Based on the Basic Justifications in Example~\ref{ex:p1_basic_justification}, we illustrate the construction of BABAS Justifications.
 Consider $e \in S_1$, where the corresponding stable extension of $S_1$ is $\mathcal{E}_1$ (see Example~\ref{ex:p1_as_stable}).
 There is only one corresponding argument of $e$ in $\mathcal{E}_1$, namely $A_{14}: (\emptyset, \{e\}) \vdash e$,
 which has a unique positive Attack Tree with respect to $\mathcal{E}_1$, $\Upsilon_1 = attTree_{\mathcal{E}_1}^+(A_{14})$.
 As shown in Example~\ref{ex:p1_basic_justification}, the Basic Justification of $A_{14}$ with respect to $\Upsilon_1$ is $justB_{\Upsilon_1}(A_{14}) = \emptyset$.
 Therefore, the unique Positive BABAS Justification of $e$ with respect to $S_1$ is $justB_{S_1}^+(e) = \{e\}$.
 This justification expresses that $e$ is in the answer set $S_1$ because it is supported only by itself, in other words, it is a fact.

 We now consider the BABAS Justification of $a \notin S_2$, where the corresponding stable extension of $S_2$ in $ABA_{\mathcal{P}_1}$ is $\mathcal{E}_2$.
 Since $a \notin S_2$, we examine all arguments with conclusion $a$ in $ABA_{\mathcal{P}_1}$, that is $A_9$ and $A_{10}$.
 Both $A_9$ and $A_{10}$ have a unique negative Attack Tree with respect to $\mathcal{E}_2$,
 $\Upsilon_3 = attTree_{\mathcal{E}_2}^-(A_9)$ (see Figure~\ref{fig:p1_attackTrees_a5}) and $\Upsilon_2 = attTree_{\mathcal{E}_2}^-(A_{10})$ (see left of Figure~\ref{fig:p1_attackTrees_a6}).
 From the Basic Justifications $justB_{\Upsilon_3}(A_9)$ and $justB_{\Upsilon_2}(A_{10})$ explained in Example~\ref{ex:p1_basic_justification},
 the BABAS Justification of $a$ with respect to $S_2$ is obtained as follows:
\begin{align*}
 justB_{S_2}^-(a) &= \{\{a, supp\mathunderscore rel(not~ \neg a, a), att \mathunderscore rel(\neg a, not~ \neg a), supp\mathunderscore rel(not~ c, \neg a),\\
 &\qquad supp\mathunderscore rel(not~d, \neg a), att \mathunderscore rel(c, not~ c), att \mathunderscore rel(d, not~ d), \\
 &\qquad supp \mathunderscore rel(not~ e, c), att\mathunderscore rel(e, not~ e), supp \mathunderscore rel(not~ \neg a, d)\},\\
 &\qquad \{a, supp\mathunderscore rel(not~ c, a), supp\mathunderscore rel(not~ d, a), supp\mathunderscore rel(not~ e, a),\\ 
 &\qquad att\mathunderscore rel(e, not~ e)\}\}
\end{align*}
Figure~\ref{fig:p1_just_a} depicts the graphical representation of the Negative BABAS Justification $justB_{S_2}^-(a)$,
where the left of the figure represents the first set in $justB_{S_2}^-(a)$,
and the right of the figure the second set.
\end{example}

\begin{figure}[t]
 \centering
 \includegraphics[width=0.9\textwidth]{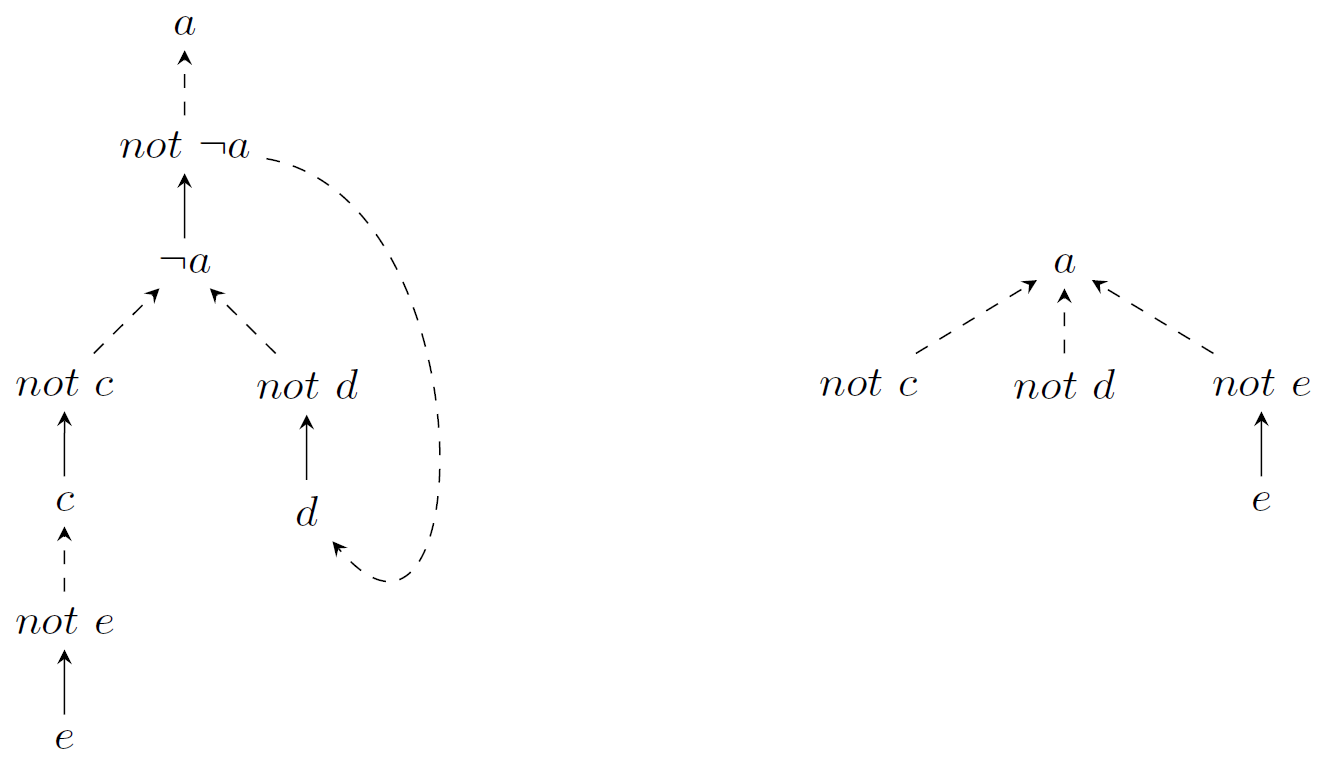}
  \caption{Graphical representation of the Negative BABAS Justification $justB_{S_2}^-(a)$ in Example~\ref{ex:p1_just_e},
  where the left graph represents the first set and the right graph the second set in $justB_{S_2}^-(a)$.
 Dotted lines stand for $supp\mathunderscore rel$ pairs in the BABAS Justification, whereas solid lines represent $att \mathunderscore rel$ pairs.}
 \label{fig:p1_just_a}
\end{figure}

So far, we only illustrated BABAS Justifications of literals $k$ for which at least one argument with conclusion $k$ exists in the translated ABA framework.
The next example demonstrates the BABAS Justification of a literal which does not have such an argument.
In general, the BABAS Justification of such literals is the empty set.

\begin{example}
\label{ex:babas_noArg}
Consider the literal $\neg c \notin S_1$ in the logic program $\mathcal{P}_1$ (see Examples \ref{ex:p1} and \ref{ex:p1_as_stable}).
There is no rule with head $\neg c$ in $\mathcal{P}_1$, and consequently $ABA_{\mathcal{P}_1}$ does not comprise an argument with conclusion $\neg c$.
Thus, there is no Attack Tree of an argument for $\neg c$ and no Basic Justification of an argument for $\neg c$.
As a consequent the Negative BABAS Justification of $\neg c$ with respect to $S_1$ is $justB_{S_1}^-(\neg c) = \emptyset$ (by Definition~\ref{def:babas}.2(b)).
\end{example}

%%%%%%%%%%%%%%%%%%%%%%%%%%%%%%%%%%%%%%%%%%%%%%%%%%%%%%%%%%%%%%%%%%%%%%%%%%%%%%%%%%%%%%%%%%%%%%%%%%%%%%%%%%%%%%%%%%%%%%%%%%%%%%%%%%%%%%%%%%%%%%%%%%%%%%%%%%%%%%%%%%
\subsection{Shortcomings of BABAS Justifications}
\label{sec:ASjust_simplicity}

A BABAS Justification is a flat structure which loses some information as compared to the underlying Attack Trees.
Attack Trees label arguments with respect to a stable extension,
expressing whether or not an argument is part of the stable extension.
However, a BABAS Justification does not provide any information about whether or not a literal is contained in the answer set in question.
Whether or not a literal is part of an answer set is important to know,
since attacks  and supports by literals contained in the answer set ``succeed'', whereas attacks and supports by literals not in the answer set do not ``succeed''.
This additional information is not captured by BABAS Justifications, even though it is provided by the underlying Attack Trees.

\begin{example}
 \label{ex:babas_shortcomings}
 Consider the Negative BABAS Justification $justB_{S_2}^-(a)$ from Example~\ref{ex:p1_just_e}, depicted as a graph in Figure~\ref{fig:p1_just_a}.
$justB_{S_2}^-(a)$ does not express whether or not the ``attacking'' literal $\neg a$ is part of $S_2$, neither in set notation nor in the graphical representation.
In contrast, the underlying Attack Tree $attTree_{\mathcal{E}_2}^-(A_9)$ in Figure~\ref{fig:p1_attackTrees_a5} specifies that
the argument $A_{11}$ for $\neg a$ is in the corresponding stable extension $\mathcal{E}_2$ by labelling $A_{11}$ as \textquotesingle$+$\textquotesingle.
It would be useful to capture this kind of information not only in the Attack Tree but also in the justification in terms of literals, so
$justB_{S_2}^-(a)$ should express that $a$ is not in $S_2$ because the support by $not~ e$ does not ``succeed'' as $not~ e \notin S_{2_{NAF}}$
because the attack by $e$ on $not~ e$ ``succeeds'' as $e \in S_2$.
\end{example}

The next example illustrates another shortcoming of BABAS Justifications,
which arises if the underlying Attack Tree contains different arguments which have the same conclusion and
occur as child nodes of the same parent node.

\begin{example}
\label{ex:sameConc_differentAsm}
 Consider the two logic programs $\mathcal{P}_3$ (left) and $\mathcal{P}_4$ (right):\\
 \begin{minipage}{0.49\textwidth}
  \begin{align*}
  p &\leftarrow not~ a\\
  p &\leftarrow not~ b\\
  q &\leftarrow not~ p\\
  a &\leftarrow\\
  b &\leftarrow
 \end{align*}
 \end{minipage}
 \begin{minipage}{0.49\textwidth}
  \begin{align*}
  p &\leftarrow not~ a, not~ b\\
  q &\leftarrow not~ p\\
  a &\leftarrow\\
  b &\leftarrow
 \end{align*}
 \end{minipage}\\
Both logic programs have only one answer set, $S_{\mathcal{P}_3} = S_{\mathcal{P}_4} = \{a,b,q\}$.
The translated ABA frameworks $ABA_{\mathcal{P}_3}$ (left) and $ABA_{\mathcal{P}_4}$ (right) have the following arguments:\vspace{0.2cm}
\\ 
\begin{minipage}{0.49\textwidth}

$A_1: (\{not~ a\}, \emptyset) \vdash not~ a$\\
$A_2: (\{not~ \neg a\}, \emptyset) \vdash not~ \neg a$\\
$A_3: (\{not~ b\}, \emptyset) \vdash not~ b$\\
$A_4: (\{not~ \neg b\}, \emptyset) \vdash not~ \neg b$\\
$A_5: (\{not~ p\}, \emptyset) \vdash not~ p$\\
$A_6: (\{not~ \neg p\}, \emptyset) \vdash not~ \neg p$\\
$A_7: (\{not~ q\}, \emptyset) \vdash not~ q$\\
$A_8: (\{not~ \neg q\}, \emptyset) \vdash not~ \neg q$\\
$A_9: (\{not~ p\}, \emptyset) \vdash q$\\
$A_{10}: (\emptyset,\{a\}) \vdash a$\\
$A_{11}: (\emptyset,\{b\}) \vdash b$\\
$A_{12}: (\{not~ a\}, \emptyset) \vdash p$\\
$A_{13}: (\{not~ b\}, \emptyset) \vdash p$\\
\end{minipage}
\begin{minipage}{0.49\textwidth}
$A_1: (\{not~ a\}, \emptyset) \vdash not~ a$\\
$A_2: (\{not~ \neg a\}, \emptyset) \vdash not~ \neg a$\\
$A_3: (\{not~ b\}, \emptyset) \vdash not~ b$\\
$A_4: (\{not~ \neg b\}, \emptyset) \vdash not~ \neg b$\\
$A_5: (\{not~ p\}, \emptyset) \vdash not~ p$\\
$A_6: (\{not~ \neg p\}, \emptyset) \vdash not~ \neg p$\\
$A_7: (\{not~ q\}, \emptyset) \vdash not~ q$\\
$A_8: (\{not~ \neg q\}, \emptyset) \vdash not~ \neg q$\\
$A_9: (\{not~ p\}, \emptyset) \vdash q$\\
$A_{10}: (\emptyset,\{a\}) \vdash a$\\
$A_{11}: (\emptyset,\{b\}) \vdash b$\\
$A_{14}: (\{not~ a, not~ b\}, \emptyset) \vdash p$\\
\end{minipage}\\
$ABA_{\mathcal{P}_3}$ and $ABA_{\mathcal{P}_4}$ share arguments $A_1$ to $A_{11}$.
In addition, $ABA_{\mathcal{P}_3}$ has arguments $A_{12}$ and $A_{13}$, whereas $ABA_{\mathcal{P}_4}$ has only one additional argument $A_{14}$.
Both ABA frameworks have a unique stable extension, $\mathcal{E}_{\mathcal{P}_3} = \mathcal{E}_{\mathcal{P}_4} = \{A_2,A_4,A_5,A_6, A_8,A_9,\\A_{10},A_{11}\}$.
$\mathcal{E}_{\mathcal{P}_3}$ is the corresponding stable extension of  $S_{\mathcal{P}_3}$ and $\mathcal{E}_{\mathcal{P}_4}$ the corresponding stable extension of $S_{\mathcal{P}_4}$.
We now examine the BABAS Justifications of $q$ with respect to $S_{\mathcal{P}_3}$ and $S_{\mathcal{P}_4}$ by constructing
Attack Trees of the corresponding arguments of $q$ with respect to  $\mathcal{E}_{\mathcal{P}_3}$ and  $\mathcal{E}_{\mathcal{P}_4}$, respectively.
In both $ABA_{\mathcal{P}_3}$ and $ABA_{\mathcal{P}_4}$, the only corresponding argument of $q$ is $A_9$ which
has a unique positive Attack Tree with respect to $\mathcal{E}_{\mathcal{P}_3}$ ($attTree_{\mathcal{E}_{\mathcal{P}_3}}^+(A_9)$),
depicted in Figure~\ref{fig:sameConc_differentAsm_p2},
and two positive Attack Trees with respect to $\mathcal{E}_{\mathcal{P}_4}$ ($attTree_{\mathcal{E}_{\mathcal{P}_4}}^+(A_9)_1$
and $attTree_{\mathcal{E}_{\mathcal{P}_4}}^+(A_9)_2$), depicted in Figure~\ref{fig:sameConc_differentAsm_p3}.
The unique Positive BABAS Justification of $q$ with respect to $S_{\mathcal{P}_3}$
constructed from $attTree_{\mathcal{E}_{\mathcal{P}_3}}^+(A_9)$ and the two possible Positive BABAS Justifications of $q$ with respect to $S_{\mathcal{P}_4}$
constructed from $attTree_{\mathcal{E}_{\mathcal{P}_4}}^+(A_9)_1$
and $attTree_{\mathcal{E}_{\mathcal{P}_4}}^+(A_9)_2$, respectively, are:
\begin{align*}
 justB_{S_{\mathcal{P}_3}}^+(q)
 & = \{q,\: supp \mathunderscore rel(not~ p, q),\: att \mathunderscore rel(p, not~ p),\:
 supp \mathunderscore rel(not~a, p),\\
 &\qquad att \mathunderscore rel(a, not~a),\:supp \mathunderscore rel(not~b, p),\: att \mathunderscore rel(b, not~b)\}
\end{align*}
\begin{align*}
justB_{S_{\mathcal{P}_4}}^+(q) 
 & = \{q,\: supp \mathunderscore rel(not~ p, q),\: att \mathunderscore rel(p, not~ p),\:
 supp \mathunderscore rel(not~a, p),\\
 &\qquad supp \mathunderscore rel(not~b, p),\: att \mathunderscore rel(a, not~a)\}
\end{align*}
\begin{align*}
justB_{S_{\mathcal{P}_4}}^+(q) 
 & = \{q,\: supp \mathunderscore rel(not~ p, q),\: att \mathunderscore rel(p, not~ p),\:
 supp \mathunderscore rel(not~a, p),\\
 &\qquad supp \mathunderscore rel(not~b, p),\: att \mathunderscore rel(b, not~b)\}
\end{align*}
The graphical representations of these BABAS Justifications are depicted in Figure~\ref{fig:sameConc_differentAsm_just}.
All of them give the impression that $p$ is supported by $not~ a$ and $not~b$ together,
which is only correct in the case of $\mathcal{P}_4$.
In $\mathcal{P}_3$, there are two different ways of concluding $p$, one supported by the NAF literal $not~a$, and the other one by $not~b$,
which is not clear from $justB_{S_{\mathcal{P}_3}}^+(q)$.
\end{example}

\begin{figure}
 \centering
  \includegraphics[width=0.8\textwidth]{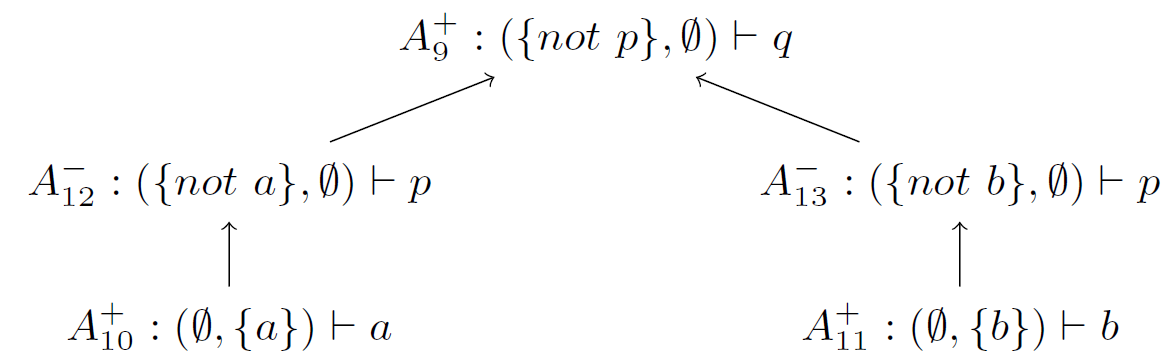}
 \caption{The unique positive Attack Tree $attTree_{\mathcal{E}_{\mathcal{P}_3}}^+(A_9)$ of $A_9$ with respect to $\mathcal{E}_{\mathcal{P}_3}$ (see Example~\ref{ex:sameConc_differentAsm}).}
 \label{fig:sameConc_differentAsm_p2}
\end{figure}

\begin{figure}
\centering
\includegraphics[width=0.9\textwidth]{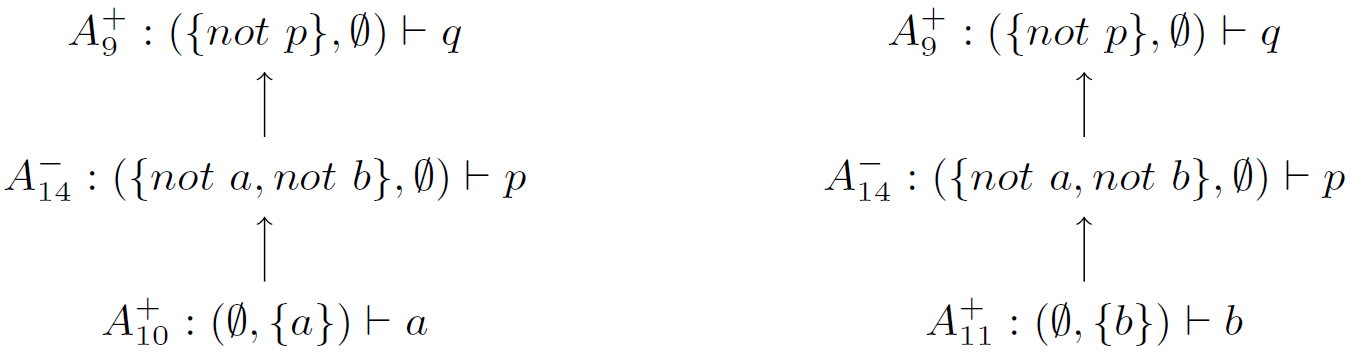}
 \caption{The two positive Attack Trees $attTree_{\mathcal{E}_{\mathcal{P}_4}}^+(A_9)_1$ (left) and $attTree_{\mathcal{E}_{\mathcal{P}_4}}^+(A_9)_2$ (right)
 of $A_9$ with respect to $\mathcal{E}_{\mathcal{P}_4}$ (see Example~\ref{ex:sameConc_differentAsm}).}
 \label{fig:sameConc_differentAsm_p3}
\end{figure}

\begin{figure}

 \centering
  \includegraphics[width=\textwidth]{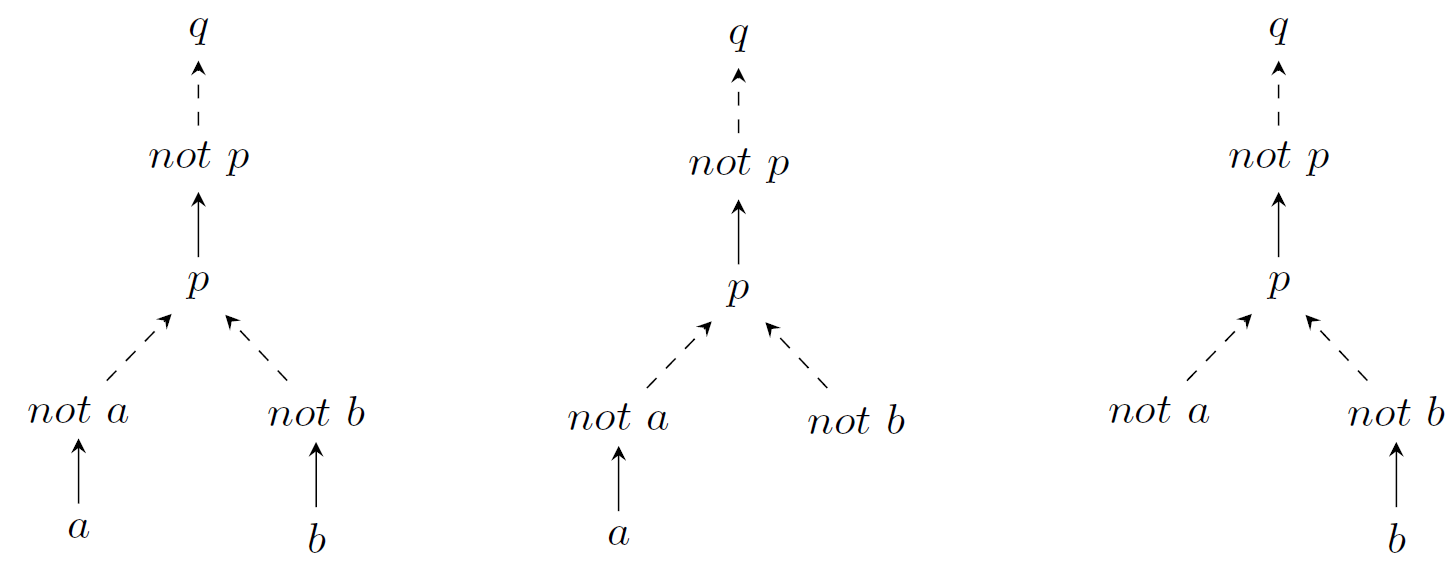}
 \caption{The unique Positive BABAS Justification $justB_{S_{\mathcal{P}_3}}^+(q)$ (left) and the two possible Positive BABAS Justifications $justB_{S_{\mathcal{P}_4}}^+(q)$ (middle and right)
  from Example~\ref{ex:sameConc_differentAsm}.}
 \label{fig:sameConc_differentAsm_just}
\end{figure}

Example~\ref{ex:sameConc_differentAsm} suggests that if a node in an Attack Tree has various children holding arguments with the same conclusion,
these child nodes should be distinguished in a justification.
We address this problem in the next section by defining a more elaborate version of ABAS Justifications.

%%%%%%%%%%%%%%%%%%%%%%%%%%%%%%%%%%%%%%%%%%%%%%%%%%%%%%%%%%%%%%%%%%%%%%%%%%%%%%%%%%%%%%%%%%%%%%%%%%%%%%%%%%%%%%%%%%%%%%%%%%%%%%%%%%%%%%%%%%%%%%%%%%%%%%%%%%%%%%%%%%
%%%%%%%%%%%%%%%%%%%%%%%%%%%%%%%%%%%%%%%%%%%%%%%%%%%%%%%%%%%%%%%%%%%%%%%%%%%%%%%%%%%%%%%%%%%%%%%%%%%%%%%%%%%%%%%%%%%%%%%%%%%%%%%%%%%%%%%%%%%%%%%%%%%%%%%%%%%%%%%%%%
\section{Labelled ABA-Based Answer Set Justifications}
\label{sec:ASjust_labelled}

We now introduce \emph{Labelled ABA-Based Answer Set (LABAS) Justifications}, which address the shortcomings of BABAS Justifications
by labelling the relations and literals in the justification as either \textquotesingle$+$\textquotesingle\ or \textquotesingle$-$\textquotesingle,
depending on the labels of arguments in the underlying Attack Trees.
In addition, literals can have an \textit{asm} or \textit{fact} tag, indicating that they are used as assumptions or facts, respectively.
Non-assumption and non-fact literals are tagged with their argument's name in order to distinguish between different arguments with the same conclusion occurring in an Attack Tree.
We refer to the structure of nodes in an Attack Tree as introduced in Notation~\ref{not:attTree}.
Similarly to BABAS Justifications, LABAS Justifications are defined in terms of \emph{Labelled Justifications},
which are a flattened version of Attack Trees.
In contrast to Basic Justifications, Labelled Justifications label the literals and relations extracted from an Attack Tree,
and extract only relevant support relations.

%%%%%%%%%%%%%%%%%%%%%%%%%%%%%%%%%%%%%%%%%%%%%%%%%%%%%%%%%%%%%%%%%%%%%%%%%%%%%%%%%%%%%%%%%%%%%%%%%%%%%%%%%%%%%%%%%%%%%%%%%%%%%%%%%%%%%%%%%%%%%%%%%%%%%%%%%%%%%%%%%%
\subsection{Labelled Justifications}
\label{sec:ASjust_labelled_justification}

A Labelled Justification assigns the label \textquotesingle$+$\textquotesingle\ to all facts and NAF literals occurring as premises of an argument labelled \textquotesingle$+$\textquotesingle\
in the Attack Tree, as well as to this argument's conclusion.
A Labelled Justification assigns the label \textquotesingle$-$\textquotesingle\ to the conclusion of an argument labelled \textquotesingle$-$\textquotesingle\ in the Attack Tree 
as well as to some NAF literals supporting this argument,
namely to those NAF literals whose contrary is the conclusion of a child node of this argument in the Attack Tree.
Attack and support relations are labelled \textquotesingle$+$\textquotesingle\ if the first literal in the relation is labelled \textquotesingle$+$\textquotesingle,
and labelled \textquotesingle$-$\textquotesingle\ if the first literal in the relation is labelled \textquotesingle$-$\textquotesingle.
Since the labels in a Labelled Justification depend on the labels of arguments in an Attack Tree, the definition is split into two cases:
one for nodes holding arguments labelled \textquotesingle$+$\textquotesingle\ in the Attack Tree,
and the other for nodes holding arguments labelled \textquotesingle$-$\textquotesingle\ in the Attack Tree.

 \begin{definition}[Labelled Justification]
 \label{def:labas}
 Let $\mathcal{P}$ be a logic program and
 let $X$ be a set of arguments in $ABA_{\mathcal{P}}$.
 Let $A$ be an argument in $ABA_{\mathcal{P}}$ and $\Upsilon = attTree_{X}(A)$ an Attack Tree of $A$ with respect to $X$.
 The \emph{Labelled Justification} of $A$ with respect to $\Upsilon$,
 denoted $justL_{\Upsilon}(A)$, is obtained as follows:\\
$justL_{\Upsilon}(A) = \\
\bigcup_{N \textit{ in } \Upsilon,\: label(N) = +}$
\begin{align*}
 &\{ supp\mathunderscore rel^+(k^+_{asm}, conc(N)_{A_N}^+) \; & &| \; k \in AP(N) \backslash conc(N), name(N) = A_N \} \; \cup\\
 &\{ supp\mathunderscore rel^+(k^+_{fact}, conc(N)_{A_N}^+) \; & &| \; k \in FP(N) \backslash conc(N), name(N) = A_N \} \; \cup\\
 &\{ att\mathunderscore rel^-(conc(M)_{A_M}^-, k^+_{asm}) \; & &| \; M \in children(N),conc(M) = \overline{k},\\
 & & &\;\;name(M)=A_M \} \; \cup
  \end{align*}
$\bigcup_{N \textit{ in } \Upsilon,\: label(N) = -}$
\begin{align*}
 &\{ supp\mathunderscore rel^-(k^-_{asm}, conc(N)_{A_N}^-) \; & &| \; k \in AP(N) \backslash conc(N),children(N) = \{M\},\\
 & & &\;\;  conc(M) = \overline{k}, name(N) = A_N \} \; \cup\\
 &\{ att\mathunderscore rel^+(conc(M)^+_{fact}, k^-_{asm}) \; & &| \; children(N) = \{M\}, conc(M) = \overline{k},\\
 & & &\;\; FP(M) = \{conc(M)\}, AP(M) = \emptyset\}\; \cup\\
 &\{ att\mathunderscore rel^+(conc(M)^+_{A_M}, k^-_{asm}) \; & &| \; children(N) = \{M\},   conc(M) = \overline{k}, AP(M) \neq \emptyset\\
 & & &\;\; \textit{ or } FP(M) \neq \{conc(M)\} , name(M)=A_M\}
 \end{align*}
 \end{definition}
 
 To illustrate Labelled Justifications and the differences with Basic Justification,
 we construct the Labelled Justifications for some of the arguments we used for Basic Justifications in Example~\ref{ex:p1_basic_justification}.

 \begin{example}
 \label{ex:p1_labelledJust}
 The Labelled Justification of $A_{14}: (\emptyset, \{e\}) \vdash e$ with respect to the positive Attack Tree
 $\Upsilon_1 =  attTree_{\mathcal{E}_1}^+(A_{14})$ is the empty set, exactly as for
 the Basic Justification: $justL_{\Upsilon_1}(A_{14}) = justB_{\Upsilon_1}(A_{14}) = \emptyset$.
 The reason is that $A_{14}$ is labelled \textquotesingle$+$\textquotesingle\ in $\Upsilon_1$, but none of the three conditions
 for nodes with label \textquotesingle$+$\textquotesingle\ in Definition~\ref{def:labas} is satisfied.
 
Now consider the Labelled Justification of $A_{10}$ with respect to the negative Attack Tree $\Upsilon_2 = attTree_{\mathcal{E}_2}^-(A_{10})$:
 \begin{align*}
 justL_{\Upsilon_2}(A_{10}) &= \{supp\mathunderscore rel^-(not~ e^-_{asm}, a_{A_{10}}^-)\}\; \cup \; \{att\mathunderscore rel^+(e^+_{fact}, not~ e^-_{asm})\}\\
 &= \{supp\mathunderscore rel^-(not~ e^-_{asm}, a_{A_{10}}^-), \: att\mathunderscore rel^+(e^+_{fact}, not~ e^-_{asm})\}
 \end{align*}
 This Labelled Justification contains fewer literal-pairs than the Basic Justification of $A_{10}$
 with respect to $\Upsilon_2$ (see Example~\ref{ex:p1_basic_justification}),
 which additionally comprises supports of $not~c$ and $not~d$ for $a$.
 Since these two supports are not necessary to explain why $a$ is not in $S_2$ (the explanation is that the supporting literal $not~ e$ is attacked by the fact $e$),
 they are omitted in the Labelled Justification.
\end{example}

The procedure of extracting attack and support relations from an Attack Tree in the construction of a Labelled Justification
is similar to the method of Basic Justifications,
where the relations are extracted step by step for every node in the Attack Tree.
The main difference of Labelled Justifications is that nodes holding arguments labelled \textquotesingle$+$\textquotesingle\ and nodes holding arguments labelled \textquotesingle$-$\textquotesingle\
in an Attack Tree are handled separately in order to obtain the correct labelling of literals and relations in the justification.
Furthermore, the extraction of the support relation is divided into two cases: one for assumption-premises, and one for fact-premises.
Similarly, there are two cases for the extraction of the attack relation: the attacker can be a fact or another (non-fact and non-assumption) literal.
Note that not all supporting literals of an argument with label \textquotesingle$-$\textquotesingle\ are
extracted for a Labelled Justification, but only ``attacked'' ones.

%%%%%%%%%%%%%%%%%%%%%%%%%%%%%%%%%%%%%%%%%%%%%%%%%%%%%%%%%%%%%%%%%%%%%%%%%%%%%%%%%%%%%%%%%%%%%%%%%%%%%%%%%%%%%%%%%%%%%%%%%%%%%%%%%%%%%%%%%%%%%%%%%%%%%%%%%%%%%%%%%%
\subsection{LABAS Justifications}
\label{sec:ASjust_labelled_labas}
 
 In this section, we define the \emph{Labelled ABA-Based Answer Set (LABAS) Justification} of a literal with respect to an answer set,
 which is based on the Labelled Justifications of an argument for this literal with respect to an Attack Tree.
 We also prove that a LABAS Justification provides an explanation for a literal using an admissible fragment of the answer set in question.
 
 Just as for BABAS Justifications, if a literal $k$ is contained in an answer set, its LABAS Justification is constructed
 from one Labelled Justification of one of the corresponding arguments of $k$.
 Conversely, if $k$ is not in an answer set, its LABAS Justification is constructed from all Labelled Justifications of all arguments with conclusion $k$.
 The only difference in the construction is that the literal being justified is labelled before it is added to the justification.
 \begin{definition}[Labelled ABA-Based Answer Set Justification]
 \label{def:labas_pos_neg}
 Let $\mathcal{P}$ be a logic program and
 $S$ an answer set of $\mathcal{P}$.
 Let $\mathcal{E}$ be the corresponding stable extension of $S$ in
 $ABA_{\mathcal{P}} = \langle \mathcal{L}_{\mathcal{P}}, \mathcal{R}_{\mathcal{P}}, \mathcal{A}_{\mathcal{P}},\: \bar{\; }\: \rangle$.
 \begin{enumerate}
  \item Let $k \in S_{\textit{NAF}}$, $A \in \mathcal{E}$ a corresponding argument of $k$, and $\Upsilon = attTree_{\mathcal{E}}^+(A)$ some positive Attack Tree of $A$ with respect to $\mathcal{E}$.
  Let $lab(k) = k^+_{asm}$ if $k \in \mathcal{A}_{\mathcal{P}}$, $lab(k) = k^+_{fact}$ if $k \leftarrow \;\in \mathcal{R}_{\mathcal{P}}$, and $lab(k) = k^+_A$ else.
 A \emph{Positive LABAS Justification} of $k$ with respect to $S$ is:\\ $justL_{S}^+(k) = \{lab(k)\} \; \cup \; justL_{\Upsilon}(A)$.
 \item Let $k \notin S_{\textit{NAF}}$, $A_1,\ldots, A_n$ ($n \geq 0$) all arguments with conclusion $k$ in $ABA_{\mathcal{P}}$,
 and $\Upsilon_{11},\ldots, \Upsilon_{1m_1}, \ldots, \Upsilon_{n1}, \ldots, \Upsilon_{nm_n}$ ($m_1,\ldots, m_n \geq 0$) all negative Attack Trees of $A_1,\ldots, A_n$ with respect to $\mathcal{E}$.
  \begin{enumerate}[(a)]
    \item If $n = 0$, then the \emph{Negative LABAS Justification} of $k$ with respect to $S$ is:\\
    $justL_{S}^-(k) = \emptyset$
    \item If $n > 0$, then let $lab(k_1) = \ldots = lab(k_n) = k^-_{asm}$ if $k \in \mathcal{A}_{\mathcal{P}}$ and $lab(k_1) = k^-_{A_1}, \ldots, lab(k_n) = k^-_{A_n}$ else.
 The \emph{Negative LABAS Justification} of $k$ with respect to $S$ is:\\ $justL_{S}^-(k) = \{\{lab(k_1)\} \; \cup \; justL_{\Upsilon_{11}}(A_1),\ldots, \{lab(k_n)\} \; \cup \; justL_{\Upsilon_{nm_n}}(A_n)\}$.
   \end{enumerate}
 \end{enumerate}
 \end{definition}

\begin{example}
 \label{ex:p2_labas}
 We illustrate the advantages of LABAS Justifications as compared to BABAS Justifications
 by justifying the same literal as in Example~\ref{ex:sameConc_differentAsm}, i.e. $q \in S_{\mathcal{P}_3}$ and $q \in S_{\mathcal{P}_4}$ of the logic programs $\mathcal{P}_3$ and $\mathcal{P}_4$.
 The LABAS Justifications are constructed from the same Attack Trees as the BABAS Justifications
 (see Figures~\ref{fig:sameConc_differentAsm_p2} and \ref{fig:sameConc_differentAsm_p3}).
 The unique Positive LABAS Justification of $q$ with respect to $S_{\mathcal{P}_3}$ and the two possible Positive LABAS Justifications of $q$ with respect to $S_{\mathcal{P}_4}$ are:
 \begin{align*}
  justL_{S_{\mathcal{P}_3}}^+(q) 
  &= \{q^+_{A_9},\: supp \mathunderscore rel^+(not~p^+_{asm}, q^+_{A_9}),\: att\mathunderscore rel^-(p_{A_{12}}^-,not~p^+_{asm}),\\
  &\qquad att\mathunderscore rel^-(p_{A_{13}}^-,not~p^+_{asm}),\: supp\mathunderscore rel^-(not~a^-_{asm},p_{A_{12}}^-),\\
  &\qquad att\mathunderscore rel^+(a^+_{fact}, not~a^-_{asm}),\: supp\mathunderscore rel^-(not~b^-_{asm},p_{A_{13}}^-),\\
  &\qquad att\mathunderscore rel^+(b^+_{fact}, not~b^-_{asm})\}
 \end{align*}
  \begin{align*}
  justL_{S_{\mathcal{P}_4}}^+(q) 
  &= \{q^+_{A_9},\: supp \mathunderscore rel^+(not~p^+_{asm}, q^+_{A_9}),\: att\mathunderscore rel^-(p_{A_{14}}^-,not~p^+_{asm}),\\
  &\qquad \: supp\mathunderscore rel^-(not~a^-_{asm},p_{A_{14}}^-), att\mathunderscore rel^+(a^+_{fact}, not~a^-_{asm})\}
 \end{align*}
   \begin{align*}
  justL_{S_{\mathcal{P}_4}}^+(q) 
  &= \{q^+_{A_9},\: supp \mathunderscore rel^+(not~p^+_{asm}, q^+_{A_9}),\: att\mathunderscore rel^-(p_{A_{14}}^-,not~p^+_{asm}),\\
  &\qquad \: supp\mathunderscore rel^-(not~b^-_{asm},p_{A_{14}}^-), att\mathunderscore rel^+(b^+_{fact}, not~b^-_{asm})\}
 \end{align*}
The graphical representations of these LABAS Justifications are depicted in Figure~\ref{fig:sameConc_differentAsm_justlabas}.
The differences between BABAS and LABAS Justifications can be easily spotted when comparing the BABAS Justification graphs in Figure~\ref{fig:sameConc_differentAsm_just} 
with the LABAS Justification graphs in Figure~\ref{fig:sameConc_differentAsm_justlabas},
both of which explain why $q$ is part of $S_{\mathcal{P}_3}$ and $S_{\mathcal{P}_4}$.
In contrast to the BABAS Justifications, the LABAS Justifications express that in $\mathcal{P}_3$ there are two different ways of deriving $p$,
one supported by $not~ a$ (yielding $A_{12}$) and the other one by $not~ b$ (yielding $A_{13}$), but in $\mathcal{P}_4$ there is only one way of
deriving $p$, supported by both $not~a$ and $not~b$ (yielding $A_{14}$).
The reason that neither of the two LABAS Justifications of $q$ with respect to $S_{\mathcal{P}_4}$ comprises both of these
supporting NAF literals is that LABAS Justifications only contain the supporting NAF literals which are ``attacked'';
in the first case $not~ a$ is attacked by $a$, in the second case $not~ b$ is attacked by $b$.
\end{example}

\begin{figure}
 \centering
  \includegraphics[width=\textwidth]{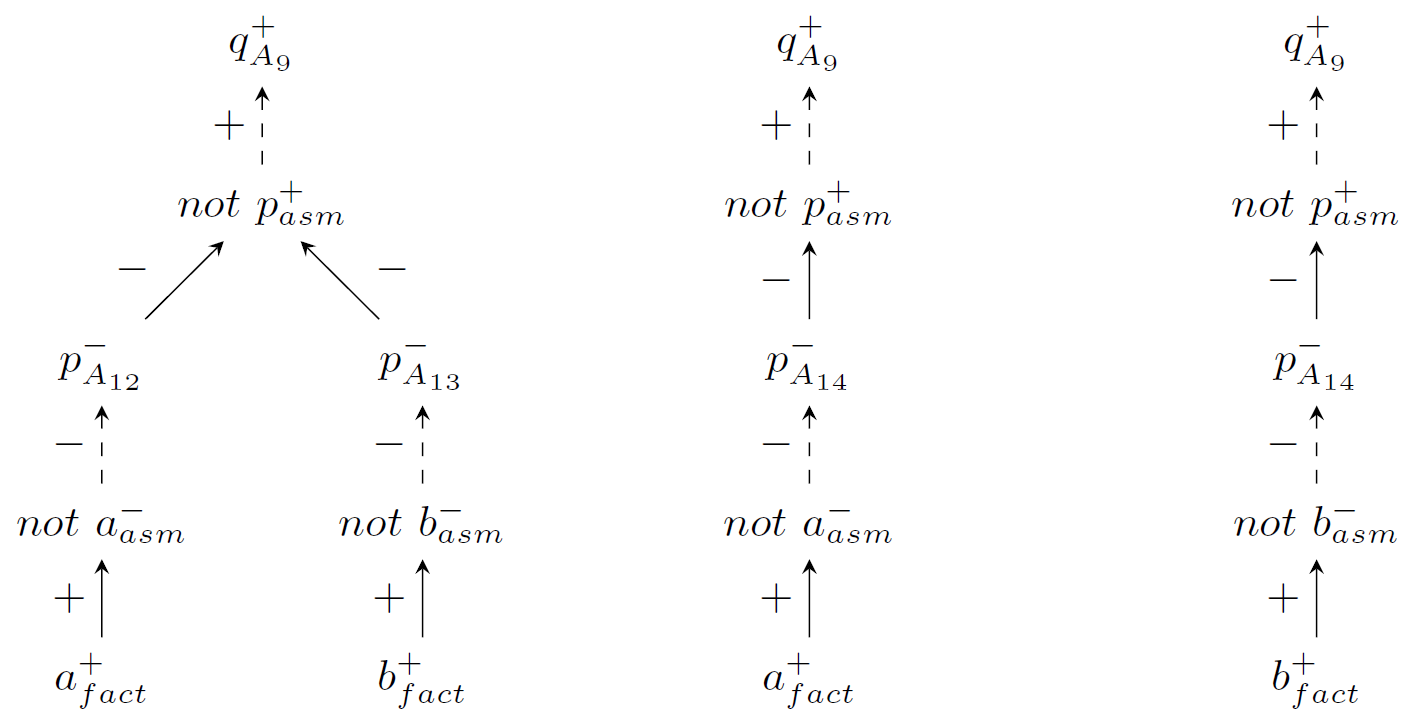}
 \caption{The unique Positive LABAS Justification $justL_{S_{\mathcal{P}_3}}^+(q)$ (left) and the two Positive LABAS Justifications $justL_{S_{\mathcal{P}_4}}^+(q)$ (middle and right)
  from Example~\ref{ex:p2_labas}. See Figure~\ref{fig:sameConc_differentAsm_just} for the respective BABAS Justifications of $q$.}
 \label{fig:sameConc_differentAsm_justlabas}
\end{figure}

As illustrated by Example~\ref{ex:p2_labas}, LABAS Justifications solve the shortcomings of BABAS Justifications:
They indicate whether or not support and attack relations ``succeed'', as well as which literals are facts or assumptions.
Furthermore, tagging literals with argument-names makes it possible to distinguish between different ways of deriving the same literal.
In addition, a LABAS Justification is sometimes shorter than the respective BABAS Justification,
only comprising relevant supporting literals of a literal not in the answer set in question.

\begin{example}
 \label{ex:doctor_labas}
 Recall Dr. Smith who has to determine whether
 to follow his own decision to treat the shortsightedness of his patient Peter with laser surgery
 or whether to act according to the suggestion of his decision support system and treat Peter with intraocular lenses (see Example \ref{ex:doctor}).
 In Example~\ref{ex:doctor_attackTree}, we illustrated how Attack Trees can be used to explain the suggestion of the decision support system
 as well as why Dr. Smith's treatment decision is wrong.
 Here, we demonstrate the LABAS Justifications explaining this.
 
 Figure~\ref{fig:doctor_labas_surgery} displays the Negative LABAS Justification of the literal $laserSurgery$ which is not contained in the answer set
 $S_{doctor}$ of the logic program $\mathcal{P}_{doctor}$ (see Example~\ref{ex:doctor}).
 This LABAS Justification is constructed from all Labelled Justifications of all arguments with conclusion $laserSurgery$,
 i.e. from all Attack Trees for arguments with conclusion $laserSurgery$.
 There is only one argument with conclusion $laserSurgery$, but there are two different negative Attack Trees for this argument (see Example~\ref{ex:doctor_attackTree}).
 The negative Attack Tree underlying the left part of the LABAS Justification in Figure~\ref{fig:doctor_labas_surgery} was illustrated in Figure~\ref{fig:doctor_attackTree_surgery}.
 The Negative LABAS Justification of $laserSurgery$ expresses that Peter should not have laser surgery for two reasons:
 first (left part), because laser surgery should only be used if the patient is not tight on money, but Peter is tight on money as he is a student and as there is no evidence that his parents are rich;
 second (right part), because laser surgery should only be used if it has not been decided that the patient should have corrective lenses,
 but there is evidence that Peter should have corrective lenses since he is shortsighted and since there is evidence against having laser surgery (and assuming that the patient does not have laser surgery
 is a prerequisite for having corrective lenses).
 
 A Positive LABAS Justification explaining why Peter should get intraocular lenses is displayed in Figure~\ref{fig:doctor_labas_intraocular}.
 This LABAS Justification expresses that all supporting assumptions needed to draw the conclusion that Peter should have intraocular lenses are satisfied,
 namely Peter is shortsighted, he should not have laser surgery, he should not have glasses, and he should not have contact lenses.
 The explanation also illustrates why these other treatments are not applicable.
 
 Using the LABAS Justifications, Dr. Smith can now understand why the decision support system suggested intraocular lenses as the best treatment for Peter
 and why Peter should not have laser surgery.
 Dr. Smith can therefore easily revise his original decision that Peter should have laser surgery, realizing that he forgot
 to consider that Peter is a student and that consequently Peter has not enough money to pay for laser surgery.
\end{example}

\begin{figure}
\centering
 \includegraphics[width=\textwidth]{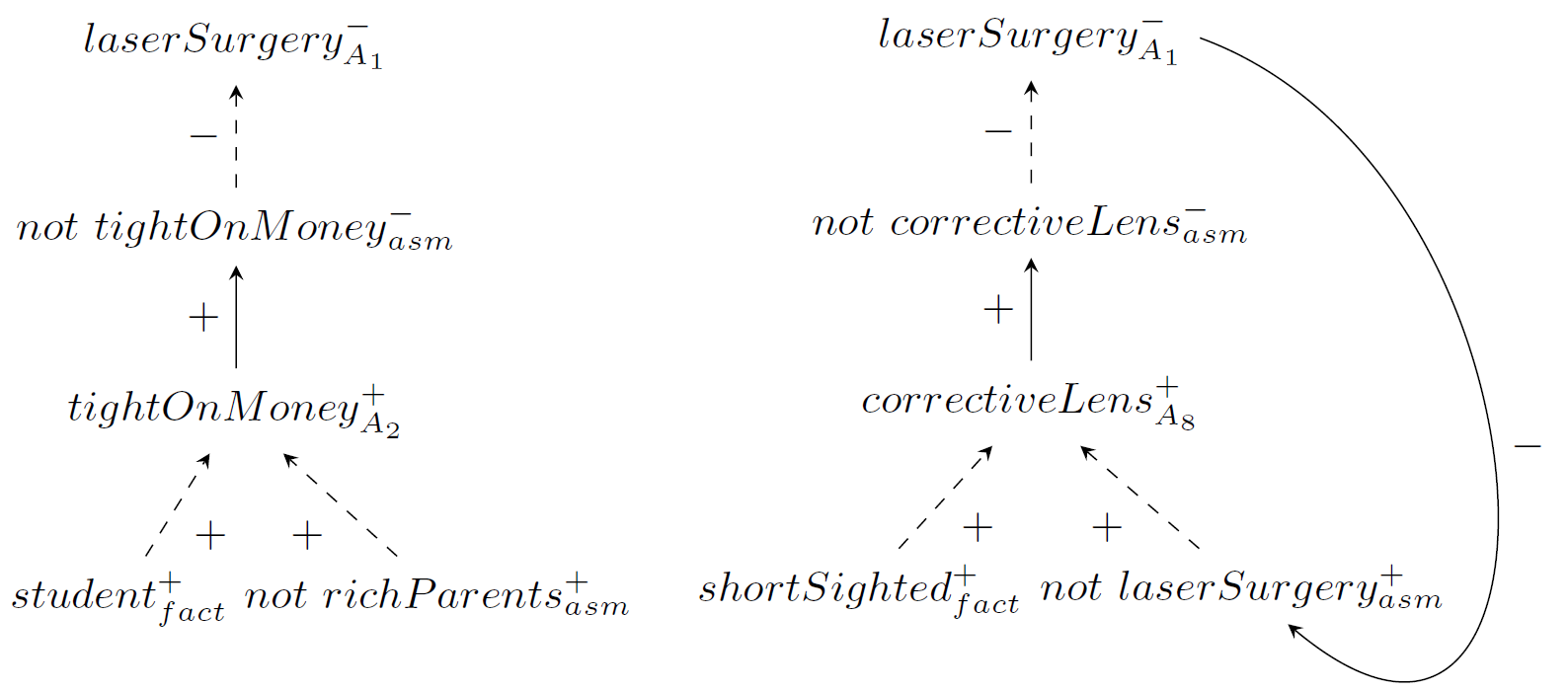}
 \caption{The Negative LABAS Justification of $laserSurgery$ with respect to $S_{doctor}$ of the logic program $\mathcal{P}_{doctor}$ (see Example~\ref{ex:doctor})
 as explained in Example~\ref{ex:doctor_labas}.}
 \label{fig:doctor_labas_surgery}
\end{figure}

\begin{figure}
 \centering
 \includegraphics[width=\textwidth]{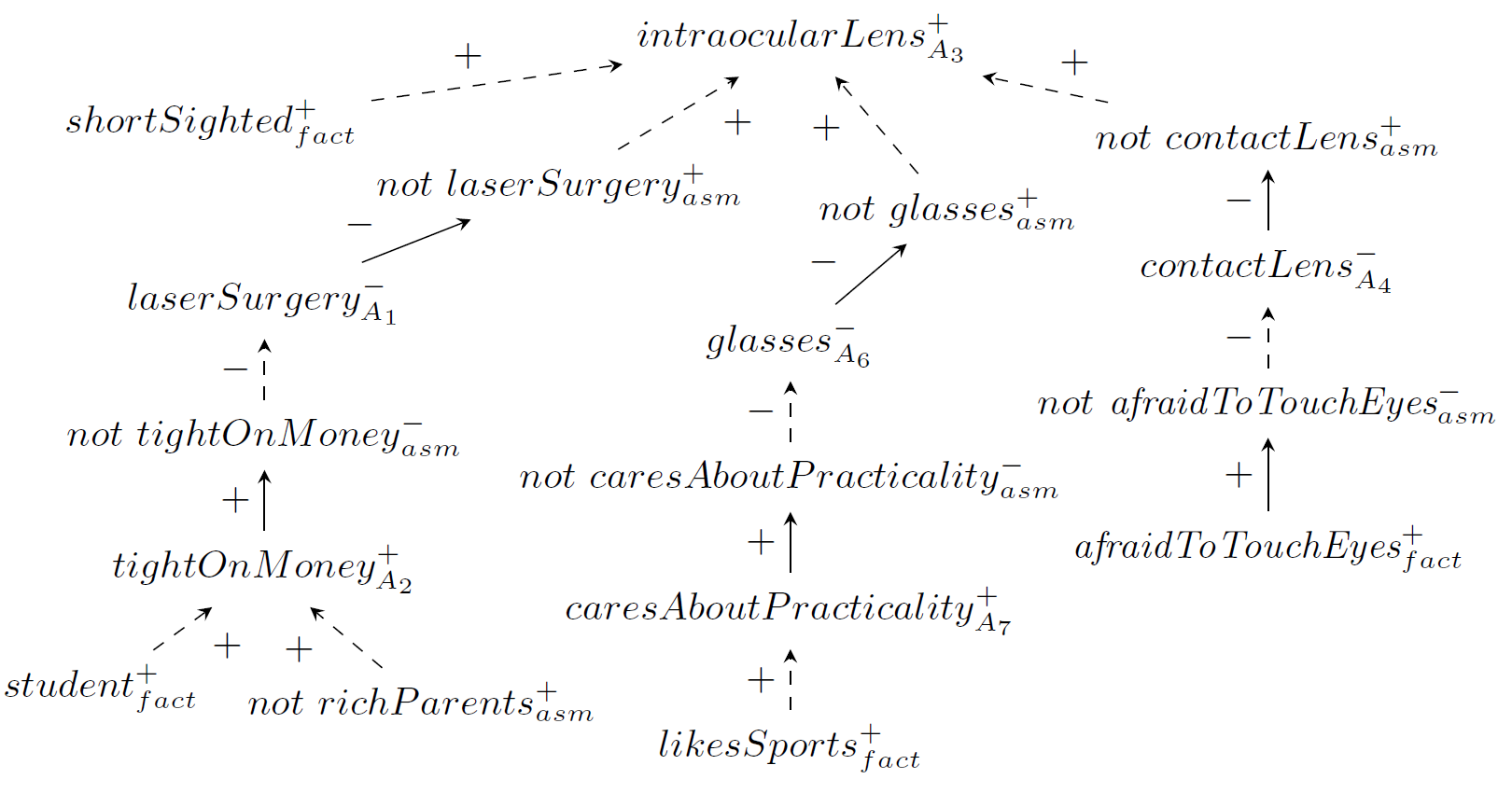}
 \caption{A Positive LABAS Explanation of $intraocularLens$ with respect to $S_{doctor}$ of the logic program $\mathcal{P}_{doctor}$ (see Example~\ref{ex:doctor})
 as explained in Example~\ref{ex:doctor_labas}.}
 \label{fig:doctor_labas_intraocular}
\end{figure}

In the following we show that LABAS Justifications
explain a literal with respect to an answer set in terms of an admissible fragment of this answer set.
We first introduce some terminology to refer to the literals in a LABAS Justification.

\begin{notation}
\label{not:labas_occuring_lit}
 Let $justL_{S}^{+}(k)$ be a Positive LABAS Justification. We say that a literal $k_1$ \emph{occurs in} $justL_{S}^{+}(k)$
 if and only if $k_1 = k$ or $k_1$ is one of the literals in a support- or attack-pair in $justL_{S}^{+}(k)$.
  We say that $k_1$ \emph{occurs positively} in $justL_{S}^{+}(k)$ if and only if it occurs as $k^+_{1_{asm}}$, $k^+_{1_{fact}}$,
  or $k^+_{1_{A}}$ (where $A$ is some argument with conclusion $k_1$).
 
  We use the same terminology for Negative LABAS Justifications.
\end{notation}

The following theorem characterizes the explanations given by Positive LABAS Justifications.
\begin{theorem}
 \label{lem:labas_admissble_pos}
 Let $\mathcal{P}$ be a logic program and let $justL_{S}^+(k_1)$ be a Positive LABAS Justification of some literal $k_1$ with respect to an answer set $S$ of $\mathcal{P}$.
  Let $NAF^+ = \{k\; |\;  k_{asm}^+ \textit{ occurs in } justL_{S}^+(k_1)\}$ be
 the set of all NAF literals occurring positively in $justL_{S}^+(k_1)$.
 Then
 \begin{itemize}
  \item $\mathcal{P}\; \cup\; NAF^+$ is an admissible scenario of $\mathcal{P}$ in the sense of \cite{admissible_scenario};
  \item $NAF^+ \subseteq S_{NAF}$.
 \end{itemize}
\end{theorem}

\begin{proof}
 By Definitions~\ref{def:labas} and \ref{def:labas_pos_neg} and Notation~\ref{not:labas_occuring_lit},
$NAF^+$ is the union of all assumptions supporting arguments labelled \textquotesingle$+$\textquotesingle\ in the Attack Tree $attTree_{\mathcal{E}}^+(A)$
used for the construction of $justL_{S}^+(k_1)$,
where $\mathcal{E}$ is the corresponding stable extension of $S$ and $A \in \mathcal{E}$ is a corresponding argument of $k_1$.
So $NAF^+ = Asms$ as defined in Theorem~\ref{lem:admissibleScenarioAttacks}. \hfill
\end{proof}

This result expresses that LABAS Justifications explain that a literal is contained in an answer set because this literal is
supported and defended by the answer set.
However, LABAS Justifications do not simply provide the whole answer set as an explanation, but instead use an admissible fragment of it.
A similar result can be formulated for Negative LABAS Justifications:
\begin{theorem}
 \label{lem:labas_admissble_neg}
 Let $\mathcal{P}$ be a logic program and let $justL_{S}^-(k_1)$ be a Negative LABAS Justification of a literal $k_1$ with respect to an answer set $S$ of $\mathcal{P}$.
Let $NAF^+_{11}, \ldots, NAF^+_{1m_1}, \ldots, NAF^+_{n1},\\ \ldots, NAF^+_{nm_n}$ be the sets of all NAF literals occurring positively in the
subsets of $justL_{S}^-(k_1)$, i.e. $NAF^+_{ij} = \{ k \;|\; k_{asm}^+ \textit{ occurs in } lab(k_{1_i}) \cup justL_{\Upsilon_{ij}}(A_i)\}$ where
$0 \leq i \leq n$ and $0 \leq j \leq m_n$.
 Then for each $NAF^+_{ij}$
 \begin{itemize}
  \item $\mathcal{P}\; \cup\; NAF^+_{ij}$ is an admissible scenario of $\mathcal{P}$ in the sense of \cite{admissible_scenario};
  \item $NAF^+_{ij} \subseteq S_{NAF}$.
 \end{itemize}
\end{theorem}

\begin{proof}
Analogous to the proof of Theorem~\ref{lem:labas_admissble_pos}.
\hfill 
\end{proof}

This means that the LABAS Justification of a literal which is not part of an answer set explains all different ways in which this
literal is ``attacked'' by an admissible fragment of the answer set.

In summary, LABAS Justifications use the same information for an explanation as Attack Trees,
namely an admissible fragment of an answer set, but expressing these information in terms of literals and the support and ``attack'' relations between them
rather than in terms of arguments and attacks.
Thus, LABAS Justifications are more suitable explanations if logic programming concepts are desired.

In the following, we will use the term \emph{ABAS Justification} as shorthand for both BABAS and LABAS Justifications.

%%%%%%%%%%%%%%%%%%%%%%%%%%%%%%%%%%%%%%%%%%%%%%%%%%%%%%%%%%%%%%%%%%%%%%%%%%%%%%%%%%%%%%%%%%%%%%%%%%%%%%%%%%%%%%%%%%%%%%%%%%%%%%%%%%%%%%%%%%%%%%%%%%%%%%%%%%%%%%%%%%
%%%%%%%%%%%%%%%%%%%%%%%%%%%%%%%%%%%%%%%%%%%%%%%%%%%%%%%%%%%%%%%%%%%%%%%%%%%%%%%%%%%%%%%%%%%%%%%%%%%%%%%%%%%%%%%%%%%%%%%%%%%%%%%%%%%%%%%%%%%%%%%%%%%%%%%%%%%%%%%%%%
\section{Discussion and Related Work}
\label{sec:relWork}
So far, the problem of justifying answer sets has not received much attention, even though the need for justifications has been expressed \cite{asp_expl_position}.
According to \cite{ASP_justification}, a justification should ``provide only the information that are relevant to the item being explained'',
making it easier understandable.
We incorporate this in ABAS Justifications by not using the whole derivation of a literal, but only 
the underlying facts and NAF literals necessary to derive the literal in question.

The two approaches for justifying why a literal is or is not part of an answer set which are
most related to ABAS Justifications are Argumentation-Based Answer Set Justifications and off-line justifications.
\emph{Argumentation-Based Answer Set Justifications} \cite{arg_based_just} are a ``predecessor''
of ABAS Justifications using the ASPIC+ argumentation framework \cite{Prakken} instead of ABA.
\emph{Off-line justifications} \cite{ASP_justification} explain why a literal is or is not part of an answer set by
making use of the well-founded model semantics for logic programs.
In the following Sections \ref{sec:relWork_offline} and \ref{sec:relWork_argBased}, we look at these two related approaches in more detail and compare them to ABAS Justifications.
In Section~\ref{sec:relWork_other}, we look at a number of other, less closely related explanation approaches.

%%%%%%%%%%%%%%%%%%%%%%%%%%%%%%%%%%%%%%%%%%%%%%%%%%%%%%%%%%%%%%%%%%%%%%%%%%%%%%%%%%%%%%%%%%%%%%%%%%%%%%%%%%%%%%%%%%%%%%%%%%%%%%%%%%%%%%%%%%%%%%%%%%%%%%%%%%%%%%%%%%
\subsection{Off-line Justifications}
\label{sec:relWork_offline}
The off-line justification for a classical literal $l$ is a graph of classical literals with root node $l$.
The child nodes of $l$ are the relevant literals which $l$ directly depends on.
In other words, the justified literal $l$ has the relevant body literals of an applicable clause in the logic program as its child nodes,
and the justifications of these body literals as subgraphs.

\begin{example}
\label{ex:abc}
Consider the following logic program $\mathcal{P}_{abc}$ (taken from \cite{ASP_justification}), which has two answer sets $S_1 = \{ b,e,f\}$ and $S_2 = \{a,e,f\}$:
\begin{align*}
a &\leftarrow f, not~ b\\
 b &\leftarrow e, not~ a\\
 f &\leftarrow e\\
 d &\leftarrow c, e\\
 c &\leftarrow d, f\\
 e &\leftarrow
 \end{align*}
 The off-line justification for $b \in S_1$ is depicted on the top right of Figure~\ref{fig:b_argjust}.
 It is constructed using the second clause in $\mathcal{P}_{abc}$, yielding a positive dependency of $b$ on $e$, and a negative dependency of $b$ on $a$.
 This expresses that $b$ is in the answer set because it depends on $e$ being part of the answer set and on $a$ not being part of it.
 Whether or not a classical literal $l$ occurring in the off-line justification is part of the answer set in question is indicated by the labels
 \textquotesingle$+$\textquotesingle\ (if $l$ is in the answer set)
 or \textquotesingle$-$\textquotesingle\ (if $l$ is not in the answer set).
 The dependency conditions of $b$ on $e$ and $a$ are satisfied, since $e$ is labelled \textquotesingle$+$\textquotesingle\ and
 $a$ is labelled \textquotesingle$-$\textquotesingle.
 The off-line justification graph also expresses that $e$ is known to be true since it is a fact
 (indicated by $\top$ in the graph) and that $a$ is assumed to be false (indicated by $assume$ in the graph).
 It is important to note that NAF literals are represented indirectly in an off-line justification by means of their corresponding classical literal.
 For example in the off-line justification of $b$ (top right of Figure~\ref{fig:b_argjust}),
 the classical literal $a$ is used to represent the dependency of $b$ on the NAF literal $not~ a$.
  \end{example}
  
\begin{figure}[t]
 \centering
 \includegraphics[width=0.9\textwidth]{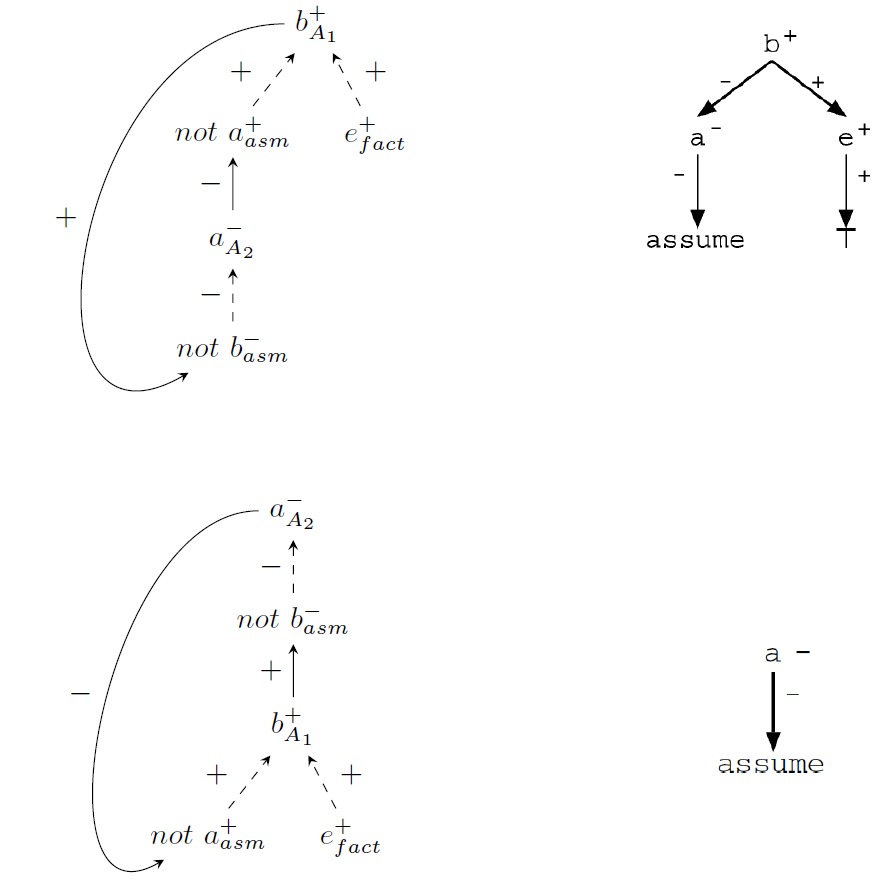}
 \caption{The two graphs at the top illustrate the LABAS Justification (left) and the Off-line Justification (right) of $b \in S_1$ of $\mathcal{P}_{abc}$, 
 whereas the graphs at the bottom represent the justifications of $a \notin S_1$ of $\mathcal{P}_{abc}$.}
 \label{fig:b_argjust}
\end{figure}
 
 Off-line justifications treat the relationship between literals in a proof-oriented way, that is as top-down dependencies, 
 whereas ABAS Justifications (and Attack Trees) provide explanations in a bottom-up manner in terms of
 assumptions and underlying knowledge supporting the conclusion.
 We argue that our bottom-up approach might be clearer for non-experts, as human decision making seems to involve
 starting from what is known along with some kind of assumptions, and then drawing conclusions from that.
 Instead of saying that $b$ is dependent on $e$ in $\mathcal{P}_{abc}$ as done by an off-line justification,
 an ABAS Justification expresses that $e$ supports $b$, as shown on the top left of Figure~\ref{fig:b_argjust}.
Especially with respect to NAF literals, we believe that a bottom-up support relation is more intuitive than a top-down dependency relation:
instead of saying that $b$ negatively depends on $a$ not being in the answer set as done by an off-line justification,
the ABAS Justification states that $not~ a$ supports $b$ (compare the two graphs at the top of Figure~\ref{fig:b_argjust}).

The well-founded model semantics is used in the construction of off-line justifications to determine literals which are ``assumed'' to be false with respect to an answer set, as opposed to literals which are always false.
These assumed literals are not further justified, i.e. they are leaf nodes in an off-line justification graph.
In contrast, ABAS Justifications further justify these ``assumed'' literals.
They are usually true NAF literals which are part of a dependency cycle.
An example is the literal $a$ in the logic program $\mathcal{P}_{abc}$, which is assumed to be false in the off-line justification
of $b$ with respect to $S_1$ (bottom right of Figure~\ref{fig:b_argjust}).
In contrast, the ABAS Justification further explains that $a$ is not in the answer set because the support by $not~ b$ does not ``succeed''
since the attack by $b$ on $not~ b$ ``succeeds'' (bottom left of Figure~\ref{fig:b_argjust}).

An off-line justification graph includes all intermediate literals in the derivation of the literal in question.
However, following \cite{asp_expl_position} we argue that for a justification it is sufficient to include the most basic relevant literals,
without considering intermediate steps.
Especially in the case of large logic programs, where derivations include many steps, an off-line justification will be a large graph with many positive and negative dependency relations,
which is hard to understand for humans.
In contrast, an ABAS Justification only contains the basic underlying literals,
i.e. facts and NAF literals necessary to derive the literal in question, making the justification clearer.
However, if the intermediate steps were required, they could be easily extracted from the arguments in the Attack Trees underlying an ABAS Justification.

In contrast to off-line justifications, where in addition to answer sets the well-founded model has to be computed,
for the construction of ABAS Justifications the computation of answer sets is sufficient.
Even though the definitions of ABAS Justifications refer to the corresponding stable extensions of the translated ABA framework,
it is not necessary to compute these stable extensions.
Whether or not the arguments needed for an ABAS Justification are contained in the respective corresponding stable extension can be directly deduced from the
answer set due to the correspondence between answer sets and stable extensions as stated in Theorems~\ref{lem:corr_ext}, \ref{lem:corr_stable}, and \ref{the:AsStable}.

%%%%%%%%%%%%%%%%%%%%%%%%%%%%%%%%%%%%%%%%%%%%%%%%%%%%%%%%%%%%%%%%%%%%%%%%%%%%%%%%%%%%%%%%%%%%%%%%%%%%%%%%%%%%%%%%%%%%%%%%%%%%%%%%%%%%%%%%%%%%%%%%%%%%%%%%%%%%%%%%%%
\subsection{Argumentation-Based Answer Set Justification}
\label{sec:relWork_argBased}

Argumentation-Based Answer Set Justification \cite{arg_based_just}
is the first work that applies argumentation theory to answer set programming, and in particular for the justification of answer sets.
There, the ASPIC+ argumentation framework \cite{Prakken} is used instead of ABA.

\begin{figure}[t]
\includegraphics[width=0.99\textwidth]{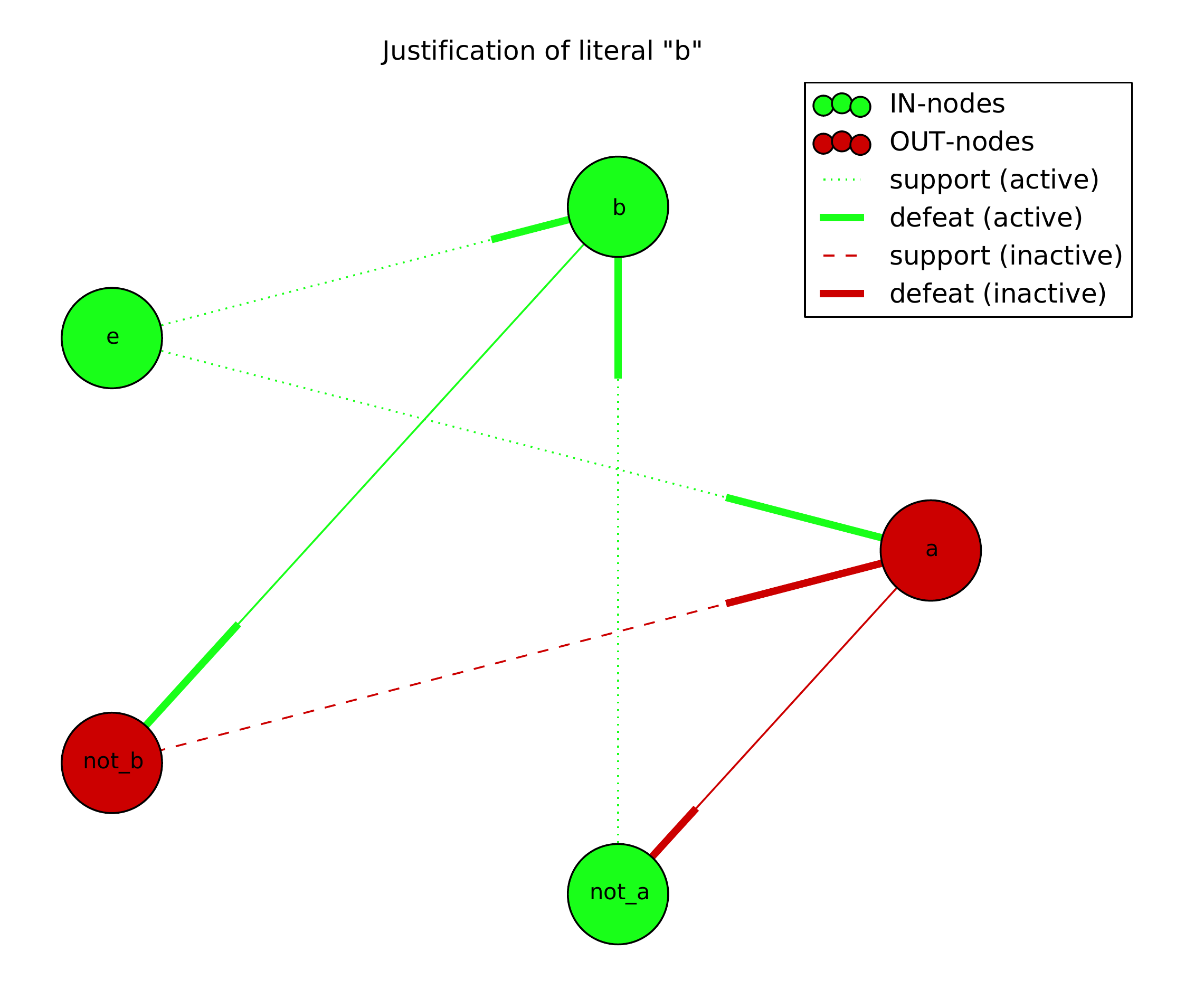}
\caption{Argumentation-Based Answer Set Justification of $b \in S_1$ of $\mathcal{P}_{abc}$ from Example~\ref{ex:abc}.}
 \label{fig:arg_based}
\end{figure}

Similarly to ABAS Justifications, in Argumentation-Based Answer Set Justifications literals are justified with respect to an answer set by means of ASPIC+ arguments
with respect to the stable extension corresponding to the answer set in question.
For the translation of a logic program into an ASPIC+ framework
 only a fraction of ASPIC+ features are needed; defeasible rules, issues, and preference orders are redundant.
This is to say that the ASPIC+ framework is too complex for the purpose of a justification and a more lightweight framework like ABA is more suitable.

The method for constructing a justification in Argumentation-Based Answer Set Justification is slightly different from the ABAS Justification approach.
Instead of extracting support- and attack-pairs from Attack Trees, Argumentation-Based Answer Set Justifications are defined recursively:
for an assumption-argument its attackers are investigated, whereas for non-assumption- and non-fact-arguments supports by assumption- and fact-arguments are examined.
The recursion terminates when fact-arguments or non-attacked assumption-arguments are encountered.

Argumentation-Based Answer Set Justifications have the same deficiencies as BABAS Justifications;
it is not clear which literals are facts or assumptions, and whether or not support and attack relations ``succeed''.
The implementation of Argumentation-Based Answer Set Justification colours the relations and literals similarly to the labels \textquotesingle$+$\textquotesingle\ and  \textquotesingle$-$\textquotesingle\
on relations and literals in LABAS Justifications, where green corresponds to \textquotesingle$+$\textquotesingle\
and red to \textquotesingle$-$\textquotesingle.
However, facts and assumptions cannot be distinguished from other literals, as depicted in Figure~\ref{fig:arg_based}.

In summary, ABAS Justifications are an improvement of Argumentation-Based Answer Set Justifications, both with respect to the elegance of the justification definition
and the appropriateness of the argumentation framework used.
LABAS Justifications also solve the deficiencies of Argumentation-Based Answer Set Justifications
by providing more information about the literals in the explanation as well as about their relationship.
Furthermore, Argumentation-Based Answer Set Justifications were introduced without any characterization.
In contrast, here we prove that ABAS Justifications provide an explanation in terms of an admissible fragment of the answer set in question,
and show their relationship with abstract dispute trees in ABA. 

%%%%%%%%%%%%%%%%%%%%%%%%%%%%%%%%%%%%%%%%%%%%%%%%%%%%%%%%%%%%%%%%%%%%%%%%%%%%%%%%%%%%%%%%%%%%%%%%%%%%%%%%%%%%%%%%%%%%%%%%%%%%%%%%%%%%%%%%%%%%%%%%%%%%%%%%%%%%%%%%%%
\subsection{Other related explanation approaches}
\label{sec:relWork_other}

In addition to the two explanations approaches for answer sets discussed in the previous sections, \cite{erdem_biologicalExplanation} introduce a formalism for explaining biomedical queries expressed in ASP.
Similar to ABAS Justifications, they construct trees for the explanation, but in contrast to our justifications these trees carry rules in the nodes rather than literals.
Another difference is that their explanation trees comprise every step in the derivation of a literal 
(similar to the approach of \cite{ASP_justification} explained in Section~\ref{sec:relWork_offline}) rather than
abstracting away from intermediate derivation steps between the literal in question and the underlying facts and NAF literals.

\cite{brain_debuggingASP} try to answer a similar question as the one we address with ABAS Justifications,
i.e. why a set of literals is or is not a subset of an answer.
Their explanations are presented in text form, but they point out that it might be possible to use a tree representation instead.
Just like \cite{erdem_biologicalExplanation}, all intermediate steps in a derivation are considered in the explanation,
thus differing from ABAS Justifications.

Related to the explanation of ASP is the visualization of the structure of logic programs in general.
ASPIDE \cite{aspide} is an Integrated Development Environment for ASP which, among other features,
displays the dependency graph of a logic program, i.e. it visualizes the positive (negative) dependencies between the rule heads and the atoms (NAF literals, respectively).
It is thus similar to the previously mentioned approaches in that it illustrates every step in a derivation.

The problem of constructing explanations has been addressed for logic programs without NAF in \cite{explanation_derivation} and \cite{explanation_prooftree}.
In the early work by Arora et al. explanations of atoms in a logic program are constructed as simple derivations of these atoms.
Thus, this approach is closer to \cite{erdem_biologicalExplanation} and \cite{brain_debuggingASP} than to ABAS Justifications as it provides all intermediate derivation steps.
Similar to this,
\cite{explanation_prooftree} show how to use proof trees as explanations for least fixpoint operators,
such as the semantics of constraint logic programs, where proof trees are derivations.

The comparison with these existing approaches demonstrates the novelty of ABAS Justifications
as they only provide the facts and NAF literals necessary for the derivation of a literal in question rather than
the whole derivation with all its intermediate steps.

Explanations have also received attention in other areas in the field of knowledge representation and reasoning,
and it has been emphasized that any expert system should provide explanations
for its solutions (see \cite{explanation_review} for an overview of explanations in heuristic expert systems).
Furthermore, it has been pointed out that even though argumentation and other knowledge-based systems have been studied mostly separately in the past,
argumentation could serve as a useful tool for the explanation of other knowledge-based systems \cite{explanation_argumentation}.
In fact, \cite{explanation_argumentsLP} provide an early account of explanations for logic programs in terms of arguments,
where Toulmin's argument scheme is applied.
However, a meta-program encoding the argument scheme has to be created by hand for any logic program that needs explanation,
making it infeasible for automatic computation.
Related to argumentation as an explanation method, \cite{garcia_argumentExplanation} introduce explanations in argumentative terms for argumentation-based reasoning methods,
such as Defeasible Logic Programming \cite{Garcia}, explaining why an argument with a certain conclusion is or is not deemed to be ``winning''.
Similar to ABAS Justifications and Attack Trees, the motivation behind their approach is to explain the solution of 
applying argumentation semantics to an argumentation framework using the context of the semantic analysis,
i.e. the attacking and defending relations between arguments.
Explanations are given in terms of argument trees similar to Attack Trees,
where arguments held by child nodes in the tree attack the argument held by the parent node.
In contrast to Attack Trees, however, every node in the tree is extended with all its attackers and the tree is labelled with respect to the grounded extension,
a different argumentation semantics, instead of stable extensions.
Another difference to our justifications is that Garc\'{\i}a et al. explain why a literal $l$ is not a winning conclusion
in terms of an explanation why the contrary literal $\neg l$ is a winning conclusion.
In contrast, ABAS Justifications explain why a literal $l$ is not a winning conclusion by pointing out why it cannot possibly be winning.

%%%%%%%%%%%%%%%%%%%%%%%%%%%%%%%%%%%%%%%%%%%%%%%%%%%%%%%%%%%%%%%%%%%%%%%%%%%%%%%%%%%%%%%%%%%%%%%%%%%%%%%%%%%%%%%%%%%%%%%%%%%%%%%%%%%%%%%%%%%%%%%%%%%%%%%%%%%%%%%%%%
%%%%%%%%%%%%%%%%%%%%%%%%%%%%%%%%%%%%%%%%%%%%%%%%%%%%%%%%%%%%%%%%%%%%%%%%%%%%%%%%%%%%%%%%%%%%%%%%%%%%%%%%%%%%%%%%%%%%%%%%%%%%%%%%%%%%%%%%%%%%%%%%%%%%%%%%%%%%%%%%%%
\section{Conclusion and Future Work}
\label{sec:conclusion}

We present two approaches for justifying why a literal is or is not contained in an answer set of a consistent logic program by translating the logic program into an Assumption-Based Argumentation (ABA) framework and
using the structure of arguments and attacks in this \emph{translated ABA framework} for the explanation.
Attack Trees, our first justification approach, provide an explanation for a literal in argumentation-theoretic terms, i.e. 
in terms of arguments and attacks between them.
ABA-Based Answer Set Justifications, our second justification approach, flatten the structure of Attack Trees,
yielding a set of literal-pairs in a support relation and literal-pairs in an attack relation.
This justification approach is more aligned with logic programming concepts as it uses literals rather than arguments as an explanation.
Both justification approaches are based on the correspondence between answer sets of a logic program and stable extensions of the translated ABA framework,
namely for every answer set of a consistent logic program there is a \emph{corresponding stable extension} of the translated ABA framework and vice versa.

Nodes in an \emph{Attack Tree} hold arguments, where the argument held by a parent node is attacked by the arguments held by the parent's child nodes.
The root node of an Attack Tree always holds an argument for the literal being justified.
Importantly, an Attack Tree is constructed with respect to the stable extension corresponding to the answer set in question.
If an argument in the Attack Tree is contained in the corresponding stable extension,
all arguments attacking it occur as its child nodes in the Attack Tree.
The intuition behind this is that an argument is contained in the stable extension if all attacking arguments are not contained in this stable extension.
Thus, all attacking arguments are added as children in the Attack Tree and further justified as to why they are not contained in the stable extension.
In contrast, if an argument in the Attack Tree is not contained in the corresponding stable extension,
only one attacking argument is picked as a child node, in particular one which is part of the corresponding stable extension.
The intuition behind picking only one attacking argument is inspired by the idea of proof by counterexample,
i.e. that one counterexample is enough to disprove a claim.
Thus, it is enough to show one way in which an argument can be disproven by an attacking argument,
even if there are other ways.
Importantly, the attacking argument has to be in the stable extension to prove that the attacked argument is not in the stable extension.
The resulting structure of an Attack Tree is an alternation of arguments in the corresponding stable extension attacked by 
arguments not in the corresponding stable extension and so on.

An \emph{ABA-Based Answer Set (ABAS) Justification} is obtained from Attack Trees by extracting a support-relation between literals
from the structure of arguments occurring in the Attack Trees,
and an attack-relation between literals from the attacks between these arguments.
Thus, ABAS Justifications are the flattened version of Attack Trees, expressing the same explanation,
but in terms of literals and their relations rather than in terms of arguments and attacks between them.
We present two versions of ABAS Justifications:
The simpler \emph{BABAS Justifications} are used to introduce the flattening method; 
the more elaborate \emph{LABAS Justifications} apply the same flattening method but additionally use labels
on literals and their relations in order to overcome some deficiencies of BABAS Justifications.
An ABAS Justification can also be interpreted as a graph of literal-nodes connected via support and attack edges.

Importantly, both Attack Trees and ABAS Justifications explain why a literal is or is not in an answer set in terms of an admissible fragment of
this answer set.
The justification that a literal is in an answer set is that a derivation of this literal is supported by an admissible fragment of this answer set.
In contrast, the justification that a literal is not contained in an answer set is that all derivations of this literal are ``attacked'' by an admissible fragment of this answer set.
In comparison to the few existing explanation methods for logic programming,
ABAS Justifications take an argumentative premise-conclusion approach, i.e. a literal is explained in terms of the facts and NAF literals necessary for its derivations,
rather than in terms of the whole derivation.

Future work includes to apply ABAS Justifications to real-world examples, with focus on medical decision making and legal reasoning.
Applying ABAS Justifications to these domains will not only yield a plausible medical or legal decision but also provide
an easily accessible explanation for elements of the solution.
A potential legal rule base for the application of ABAS Justifications is the encoding of the Japanese Civil Code as used in \cite{proleg}.
With respect to applicability of ABAS Justifications, we are planning to develop a user-friendly implementation of ABAS Justification
and conduct a survey both among experts in ASP and among non-experts using ASP as a decision-making tool.
Furthermore, we are working on an extension of ABAS Justifications to explain inconsistencies in logic programs and to help debugging these logic programs.

%%%%%%%%%%%%%%%%%%%%%%%%%%%%%%%%%%%%%%%%%%%%%%%%%%%%%%%%%%%%%%%%%%%%%%%%%%%%%%%%%%%%%%%%%%%%%%%%%%%%%%%%%%%%%%%%%%%%%%%%%%%%%%%%%%%%%%%%%%%%%%%%%%%%%%%%%%%%%%%%%%
%%%%%%%%%%%%%%%%%%%%%%%%%%%%%%%%%%%%%%%%%%%%%%%%%%%%%%%%%%%%%%%%%%%%%%%%%%%%%%%%%%%%%%%%%%%%%%%%%%%%%%%%%%%%%%%%%%%%%%%%%%%%%%%%%%%%%%%%%%%%%%%%%%%%%%%%%%%%%%%%%%
\section*{Acknowledgements}
\label{sec:acknowledgements}
We would like to thank the anonymous reviewers and David Pearce for their constructive feedback, as well as Abdallah Arioua for pointing out some related work.

%%%%%%%%%%%%%%%%%%%%%%%%%%%%%%%%%%%%%%%%%%%%%%%%%%%%%%%%%%%%%%%%%%%%%%%%%%%%%%%%%%%%%%%%%%%%%%%%%%%%%%%%%%%%%%%%%%%%%%%%%%%%%%%%%%%%%%%%%%%%%%%%%%%%%%%%%%%%%%%%%%
%%%%%%%%%%%%%%%%%%%%%%%%%%%%%%%%%%%%%%%%%%%%%%%%%%%%%%%%%%%%%%%%%%%%%%%%%%%%%%%%%%%%%%%%%%%%%%%%%%%%%%%%%%%%%%%%%%%%%%%%%%%%%%%%%%%%%%%%%%%%%%%%%%%%%%%%%%%%%%%%%%

\bibliography{literature}
\bibliographystyle{acmtrans}
\end{document}